\documentclass[twoside]{article}

\usepackage[preprint]{aistats2027}
\usepackage[round]{natbib}

\usepackage{mymacros}
\usepackage{graphicx}
\usepackage{enumitem}
\usepackage{multicol}
\usepackage[table]{xcolor}
\usepackage{makecell}
\usepackage{hhline}
\usepackage{pgfplots}
\usepackage{xurl}
\usepackage{subcaption}

\usepackage{etoolbox}
\makeatletter
\newskip\cbrl@origtextfloatsep
\cbrl@origtextfloatsep=\textfloatsep  
\patchcmd{\@copyrightspace}{\begin{center}}
  {\vspace*{\dimexpr\cbrl@origtextfloatsep-\textfloatsep\relax}\begin{center}}
  {}{\PackageWarning{main}{[TEXTFLOATSEP] Could not patch AISTATS copyright box spacing}}
\makeatother

\begin{document}

%

%

\twocolumn[

\aistatstitle{On Generalizing Preference-based Reinforcement Learning to Multi-Objective Problems}

\aistatsauthor{ Simone Drago \And Marco Mussi \And Leonardo Bianconi \And Alberto Maria Metelli }

\aistatsaddress{ Politecnico di Milano \And Politecnico di Milano \And Politecnico di Milano \And Politecnico di Milano } ]

\begin{abstract}
We study \emph{reinforcement learning} (RL) from pairwise human-generated trajectory comparisons.
We generalize preference-based RL by formalizing a novel setting in which the expert can also label trajectory pairs as \emph{incomparable} when neither trajectory \quotes{dominates} the other.
We introduce the learning problem and the desiderata that a model should satisfy, showing that, when multiple objectives are present, no compliant model admits a negative log-likelihood that is convex in the utility difference.
Then, we propose a novel Bradley-Terry-inspired rationality model that enables incomparability and infers a multi-dimensional reward function, and discuss its compliance with the desiderata.
We upper bound the divergence between the true and learned comparison distributions when a dataset is available, separating the rate of approach to the optimum from a non-vanishing term due to non-convexity.
Then, we study how learned parameters affect downstream guarantees, highlighting the impact of invariance on Pareto optimality.
\end{abstract}

\section{Introduction}
\label{sec:intro}
\setlength{\abovedisplayskip}{4pt}
\setlength{\belowdisplayskip}{4pt}
\setlength{\abovedisplayshortskip}{3pt}
\setlength{\belowdisplayshortskip}{3pt}
\setlength{\floatsep}{6pt}

In \emph{reinforcement learning} \citep[RL,][]{sutton2018reinforcement}, a learning agent interacts with an environment by iteratively observing its current \emph{state} and, based on it, selecting an \emph{action}, aiming to optimize an objective function over a given time horizon. The learning process in RL is guided by a \emph{reward} function, i.e., an engineered numerical feedback provided to the agent after each action \citep{sutton2004reward}. The goal of the agent then becomes maximizing the expected cumulative reward.
However, defining a reward function is a complex endeavor \citep{booth2023perils}. The system engineer must select numerical values to induce behavior that \quotes{solves} the task while avoiding unwanted, potentially dangerous phenomena \citep[e.g., reward hacking,][]{amodei2016concrete}. This process, namely \emph{reward engineering} \citep{dewey2014reinforcement}, is an iterative trial-and-error process that relies on domain knowledge and requires tweaking the reward, as the induced behavior is highly sensitive to misspecification \citep{pan2022effects}.

\emph{Preference-based RL} \citep[PbRL,][]{wirth2017survey} has emerged as a powerful paradigm to overcome the inherent challenges of reward design in RL. Crucially, it avoids reward engineering entirely, relying instead on pairwise comparisons provided by a \emph{human expert} and learning the policy that best aligns with them. PbRL achieved remarkable results in domains ranging from robotics \citep{lee2021pebble} to large language models \citep{ouyang2022training}.
Basing on the \emph{utility representation theorem} \citep{vonneumann1947theory}, which postulates that any rational agent generates feedback in terms of an underlying \emph{scalar} utility function defined over available options. In RL terms, this corresponds to assuming the existence of an underlying, unknown reward \citep{friedman1952expected} which allows defining a utility via the trajectory return.
One class of PbRL approaches, namely \emph{reward learning} \citep[see, e.g.,][]{christiano2017deep}, comprises two steps: ($i$) \emph{reward estimation} from the expert's preferences and ($ii$) \emph{policy optimization} via RL methods.
The link between the observed preferences and the reward that generates them is formalized in terms of a \emph{rationality models}. The \emph{Bradley-Terry} model \citep[BT,][]{bradley1952rank} is most commonly employed in the literature \citep[see, e.g.,][]{christiano2017deep,rafailov2024direct}, and it is fed by a utility function that quantifies the \quotes{goodness} of each alternative. Under the BT model, only \quotes{clear} preferences are possible, and the probability of observing the preference of one alternative over the other is computed as the sigmoid function of their difference in utility.
In PbRL, the expert is assumed to be always able to provide a (stochastic) preference \citep[or, at most, indifference; see ][]{rao1967ties} label.
In decision theory, this is known as the \emph{completeness axiom} \citep{vonneumann1947theory} that, together with \emph{transitivity}, induces a \emph{total (pre)order} over the available options, determined by their scalar utility.

However, in real-world scenarios, an expert may choose not to state a clear preference among alternatives \citep{aumann1962utility}. A common practical workaround attributes this to the expert's irrationality, discarding the sample \citep[see, e.g.,][]{christiano2017deep}.
However, \citet{aumann1962utility} argues that forcing a decision when the expert cannot choose violates their rationality. Indeed, the validity of the completeness axiom has been disputed by von~Neumann and Morgenstern themselves \citep{vonneumann1947theory} and removing it, while retaining the other axioms, has been shown to induce a \emph{partial (pre)order} among the alternatives \citep{ok2002utility}.
We expand upon this reasoning, overcoming the completeness axiom from PbRL and studying the effects on learning guarantees.

An intuitive scenario in which the expert may choose not to decide between two alternatives is when \emph{multiple objectives conflict}, e.g., speed vs.\ safety in autonomous driving. In such cases, it can happen that neither alternative \emph{dominates} the other, i.e., is better w.r.t.\ all objectives. Hence, this scenario is not a failure of the expert, but rather the expression of a different rationality model, in which \quotes{not deciding} is a signal in itself.
Following the literature, we refer to this feedback as \emph{incomparability}.\footnote{Incomparability
first appeared together with utility theory axioms \citep{vonneumann1947theory}. \citet{roy1984relational} formalized the properties that cause it to arise and \citet{tsoukias1995new} proposed an axiomatic first-order logic foundation.
}
Indeed, the inability to decide between alternatives is an informative signal, localizing the conflict between objectives.
Yet, its explicit treatment has only recently been identified as an open problem in PbRL \citep{abdelkareem2022advances}.
The partial order induced by incomparabilities can be represented via a multi-dimensional utility function, with clear preference arising only under Pareto dominance \citep{ok2002utility}, while preserving the remainder of the utility theory axioms \citep{dubra2004expected}.
Casting this in terms of the reward learning approach, the framework that enables policy optimization with multi-dimensional rewards is \emph{multi-objective RL} \citep[MORL,][]{hayes2022practical}.
In MORL, a single, universally optimal behavior is typically unattainable and, thus, the goal is to recover a set of non-dominated behaviors, i.e., the \emph{Pareto front} \citep{roijers2013survey}.

Achieving our goal requires connecting PbRL and MORL. The latter handles multiple, possibly conflicting, objectives via a numerical reward signal, whereas the former handles human feedback only through clear preferences, and thus generally fails to represent incomparability.
We propose \emph{comparison-based reinforcement learning} (CbRL), connecting the two in a novel paradigm, as shown in Table~\ref{tab:cbrl_positioning}.

To the best of our knowledge, there exist approaches that address multi-objective problems with feedback provided by a human expert, but they do so by effectively reducing to single-objective problems, either eliciting feedback conditioned on a weighting of the objectives provided by the agent \citep{mu2025preference} or considering per-objective feedback \citep{dai2024safe,astudillo2025preferential, adikari2026lemur}. An alternative assumption is that the expert internally scalarizes the objectives with an unknown weighting before providing feedback \citep{cherukuri2025learning}.
Unlike prior work, we study CbRL in the most general comparison-elicitation setting, namely, showing two trajectories and receiving only a label among four options, with incomparability as a distinct outcome, without imposing any additional structure on the expert.
Notably, the main complexity lies in moving from a scalar to a multi-dimensional utility, as we lose the total order over trajectories, convexity of the loss function, the reward translation as the only source of invariance, and the existence of a single optimal behavior of PbRL.
We focus on \emph{what comparisons allow us to identify}, providing general results and specializing to a novel extension of the BT model, giving theoretical guarantees of what can be learned using such a model.
We refer the reader to Appendix~\ref{apx:related} for an in-depth discussion of related works.

\begin{table}[t]
\centering
\small
\renewcommand{\arraystretch}{1.5}
\begin{tabular}{l|c|c|}
\multicolumn{1}{c}{} & \multicolumn{1}{c}{Single-Objective} & \multicolumn{1}{c}{Multi-Objective} \\
\hhline{~--}
Rewards & RL & MORL \\
\hhline{~--}
\makecell[l]{Human\\Feedback} & PbRL & \cellcolor{black!20}\textbf{CbRL} \\
\hhline{~--}
\end{tabular}
\caption{Positioning of CbRL in the literature.}
\label{tab:cbrl_positioning}
\end{table}

\textbf{Original Contributions.}~~Our contributions of are summarized as follows:
\begin{itemize}[noitemsep, topsep=-1pt, leftmargin=*]
    \item In Section~\ref{sec:framework}, we formalize the novel \emph{Markov decision process with comparisons} (MDPC) framework, with incomparability as a rational feedback, define the desiderata of an MDPC-compliant rationality model, and show that no compliant model can yield a convex negative log-likelihood (Proposition~\ref{prop:noncvx}).
    \item In Section~\ref{sec:model_based}, we define the novel \emph{multi-objective Bradley-Terry} (MOBT) model, study its compliance and its behavior in the $d=1$ case. Then, we provide upper bounds on the error induced by learning its parameters via maximum likelihood estimation (Theorems~\ref{thr:sample_complexity} and~\ref{thr:sample_complexity_local}). We characterize how reward identifiability guarantees change with $d$ (Theorem~\ref{thr:identifiability_main}), and provide downstream guarantees in terms of scalarized Pareto front (Theorem~\ref{thr:pareto}).
\end{itemize}

\textbf{Notation.}~~Given $a \!\in\! \mathbb{N}$, we define $\dsb{a} \coloneqq \{ 1, 2, \ldots , a\}$.
We denote as $\bm{1}_n, \bm{0}_n \!\in\! \mathbb{R}^n$ the column vectors of $n$ ones and zeros, respectively, as $\bm{e}_i \!\in\! \mathbb{R}^n$ the $i$-th vector of the canonical basis, as $\bm{I}_n$ the $n \!\times\! n$ identity matrix, and as $x_i$ the $i$-th element of a vector $\bm{x} \!\in\! \mathbb{R}^n$.
Inequalities between vectors hold element-wise, and we denote as $\mathbb{R}^n_{\ge 0}$ and $\mathbb{R}^n_{\le 0}$ the non-negative and non-positive orthants, respectively.
Given a measurable set $\mathcal{Y}$, we denote as $\Delta(\mathcal{Y})$ the set of probability distributions over $\mathcal{Y}$. We abbreviate $\Delta_{n-1} \coloneqq \Delta(\dsb{n})$.
We denote as $\bm{A}_n \!\coloneqq\! \bm{1}_n \bm{1}_n^\top / n$ the \emph{averaging} matrix and as $\bm{\Pi}_n \!\coloneqq\! \bm{I}_n - \bm{A}_n$ the \emph{orthogonal projection} onto the complement of the standard diagonal, omitting the subscript when clear from the context.
We denote the \emph{population standard deviation} of a vector $\bm{x}$ as $\std (\bm{x}) \!\coloneqq\! \| \bm{\Pi}_n \bm{x}\|_2/\sqrt{n}$.
Given a matrix $\bm{X} \in \mathbb{R}^{m \times n}$, we denote the column vector obtained by stacking its columns as $\mathrm{vec}(\bm{X}) \in \mathbb{R}^{mn}$.
Let $P, Q \!\in\! \Delta(\mathcal{Y})$ be s.t.\ 
$P \!\ll\! Q$, we define their Kullback-Leibler (KL) divergence as $\dkl{P}{Q} \!\coloneqq\! \mathbb{E}_{y \sim P} [ \log (\mathrm{d}P / \mathrm{d}Q)(y) ]$.
We use the $\widetilde{\BigO} (\cdot)$  to omit constant and polylogarithmic factors.

\section{Framework for Incomparability}
\label{sec:framework}
In this section, we introduce the feedback classes of CbRL, propose the novel \emph{Markov decision process with comparisons} setting, and discuss the desiderata for a rationality model. Finally, we provide an important theoretical result that characterizes the setting.

\textbf{Comparison Feedback.}~~Consider the case in which a human expert is shown two alternatives and is asked to provide a comparison between the two.
Following \citet{eliaz2006indifference}, we characterize the four possible outcomes that the expert can provide:

\begin{enumerate}[noitemsep, topsep=-1pt, leftmargin=20pt]
    \item[($\prelsinv$)] \emph{Direct preference}, when the first alternative is clearly preferred;
    \item[($\prels$)] \emph{Inverse preference}, when the second alternative is clearly preferred;
    \item[($\equivrel$)] \emph{Indifference}, when the alternatives are equivalent;
    \item[($\incomprel$)] \emph{Incomparability}, when none of the above holds.
\end{enumerate}

Notably, incomparability can arise for two fundamental reasons \citep{tsoukias1995new}. The first is when the data is incomplete, the request is ambiguous, or the expert presents some degree of irrationality. The second, instead, occurs when the expert perceives an inherent conflict among the alternatives.\footnote{Practical PbRL approaches do not uncouple the two scenarios, circumventing the problem by discarding incomparable samples.}
In this work, we focus on the latter.
This inherent conflict can be characterized by a multi-dimensional utility function that assigns a utility value to each objective to every alternative \citep{ok2002utility}. CbRL casts this to an RL-like problem, combining the multi-objective aspect of MORL with the human-generated feedback of PbRL.

\textbf{Framework.}~~We now formalize the setting that allows defining CbRL problems.

\begin{defi}[Markov Decision Process with Comparisons]
\label{defi:mdpc}
    A \emph{Markov decision process with comparisons} (MDPC) is a tuple $(\Ss, \As, H, p, \mu, \rho)$, where $\Ss$ and $\As$ are the state and action spaces, respectively, $H \in \mathbb{N}$ is the horizon,
    $p \!\coloneqq\! (p_h)_{h \in \dsb{H}}$ is the stage-dependent transition model defined as $p_h: \Ss \times \As \to \Delta(\Ss)$ for every $h \!\in\! \dsb{H}$, $\mu \!\in\! \Delta(\Ss)$ is the initial state distribution, and $\rho: \Ts^2 \to \Delta(\Comps)$ is the \emph{comparison probability distribution}, where $\Ts \!\subseteq\! (\Ss \times \As)^H$ is the space of possible trajectories, i.e., sequences of state-action pairs $\tau \!\coloneqq\! (s_h, a_h)_{h \in \dsb{H}}$, and $\Comps \!\coloneqq\! \{\prelsinv, \prels, \equivrel, \incomprel\}$ is the \emph{outcome} set of comparisons.
\end{defi}
The MDPC can be interpreted as the \quotes{multi-objective counterpart} of the \emph{Markov decision process with preferences} \citep{wirth2017survey}.
Intuitively, upon selecting two trajectories $\tau, \tau' \in \Ts$ and showing them to the expert, the agent receives a \emph{comparison} from the outcome set $\Comps$ generated stochastically according to distribution $\rho$.
Following \citet{aumann1962utility}, we assume the existence of an underlying multi-dimensional utility function $\bm{u}: \Ts \to \mathbb{R}^d$, mapping each trajectory to a $d$-dimensional utility vector, where $d \in \mathbb{N}$ is the number of objectives. Thus, the expert internally evaluates alternatives in terms of such a utility function and generates a stochastic feedback. Throughout the paper, as is common in the MORL literature \citep{roijers2013survey, hayes2022practical}, we assume $d$ to be known. Indeed, this only requires the system engineer to know which \emph{criteria or objectives} should be evaluated, e.g., speed vs.\ safety in autonomous driving, and communicate them to the expert when eliciting their feedback. Relaxing this assumption requires inferring $d$ from data, e.g., via held-out likelihood. However, as observed in \citep{drago2025towards}, computing a minimum-dimensional utility is NP-hard when a finite partial order is concerned.\footnote{This obstruction derives from the fact that such a minimum-dimensional utility has a number of components equal to the order dimension \citep{dushnik1941partially}. Deciding whether the order dimension is at most $3$ is known to be NP-complete \citep{yannakakis1982complexity, felsner2017complexity}.}

We model the agent via a stage-dependent Markovian policy $\pi \coloneqq (\pi_h)_{h \in \dsb{H}}$, with $\pi_h : \Ss \to \Delta(\As)$ for each $h \in \dsb{H}$, and denote as $\Pi$ the set of such policies.

Taking after the \emph{reward learning} formulation of PbRL, the reward estimation phase of a CbRL problem requires learning a function $f: \mathbb{R}^d \to \Delta(\Comps)$ that allows approximating $\rho$ given only the utility difference $\bm{\delta} (\tau, \tau') \coloneqq \bm{u}(\tau) - \bm{u}(\tau')$. Then, recovering a $d$-dimensional reward function from $f$ allows addressing the policy optimization phase using standard MORL methods.
We call $f$ a \emph{rationality model}.

We now define and discuss the desiderata that a rationality model should satisfy.
\begin{defi}[Compliance in the Limits]
\label{defi:desiderata}
    A function $f:\mathbb{R}^d \to \Delta(\Comps)$ is a \emph{desirable rationality model} for an MDPC if the following holds:
    \begin{subequations}
        \begin{align}
            \argmax_{\comp \in \Comps} \lim_{\bm{\delta} \to +\infty \bm{1}_d}  &f(\comp | \bm{\delta}) = \{ \prelsinv \}, \label{eq:desiderata:direct_pref} \\
            \argmax_{\comp \in \Comps} \lim_{\bm{\delta} \to -\infty \bm{1}_d}  &f(\comp | \bm{\delta}) = \{ \prels \}, \label{eq:desiderata:inverse_pref} \\
            \argmax_{\comp \in \Comps} \lim_{\bm{\delta} \to \bm{0}_d} &f(\comp | \bm{\delta}) = \{\equivrel\}, \label{eq:desiderata:indifference} \\
            \argmax_{\comp \in \Comps} \lim_{\bm{\delta} \to +\infty \bm{t}} &f(\comp | \bm{\delta}) = \{\incomprel\}, \label{eq:desiderata:incomparability}
        \end{align}
    \end{subequations}
    for every $\bm{t} \in \{-1, 1\}^d \setminus \{-\bm{1}_d, \bm{1}_d\}$, where, for $\bm{v} \in \mathbb{R}^d$, $\lim_{\bm{\delta} \to \pm\infty \bm{v}} f(\comp | \bm{\delta}) \coloneqq \lim_{x \to +\infty} f(\comp | \pm x \bm{v})$ denotes the limit along the ray of direction $\bm{v}$, assuming they exist.
\end{defi}
Definition~\ref{defi:desiderata} formalizes the limit behavior a rationality model must present, requiring the most probable outcome to be ($i$) direct preference $\prelsinv$ if the first option is \emph{infinitely better} than the second in all dimensions (and symmetrically for inverse preference $\prels$), ($ii$) indifference $\equivrel$ if the two options are \emph{equal} with respect to every objective, and ($iii$) incomparability $\incomprel$ if \emph{neither option dominates} the other, i.e., either option is infinitely better than the other with respect to at least one objective.
We graphically represent the problem desiderata for $d=2$ in Figure~\ref{fig:desiderata}.

\begin{figure}[t]
    \centering
    \resizebox{0.9\linewidth}{!}{\usetikzlibrary {arrows.meta}

\begin{tikzpicture}

\definecolor{lightgray204}{RGB}{204,204,204}
\definecolor{color1}{RGB}{216,27,96}      
\definecolor{color2}{RGB}{30,136,229}      
\definecolor{color3}{RGB}{255,193,7}      
\definecolor{color4}{RGB}{0,77,64}    

\begin{axis}[
    xtick=\empty,
    ytick=\empty,
    xmin=-2, xmax=2,
    ymin=-2, ymax=2
]
    \addplot graphics [includegraphics cmd=\pgfimage, xmin=-2, xmax=2, ymin=-2, ymax=2] {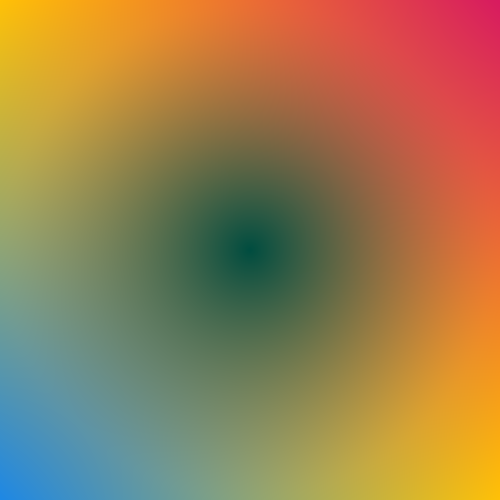};
    
    \draw[-Triangle, thick, draw=black] (axis cs:-2,0) -- (axis cs:2,0);
    
    \draw[-Triangle, thick, draw=black] (axis cs:0,-2) -- (axis cs:0,2);
    
    \draw (axis cs:0.9, 0.3) node[
      scale=1.3,
      anchor=west,
      text=black,
      rotate=0.0
    ]{$\delta_1(\tau,\tau')$};
    
    \draw (axis cs:0.6, 1.5) node[
      scale=1.3,
      anchor=south,
      text=black,
      rotate=0.0,
    ]{$\delta_2(\tau,\tau')$};
\end{axis}
\hspace{0.3cm}

\node[anchor=west, inner sep=0pt] at (current bounding box.east) {
    \begin{tikzpicture}
        \node[draw=none, fill=none, inner sep=1pt] {
            \begin{tabular}{clcl}
                \tikz\fill[color1] (0,0) rectangle (0.6,0.3); 
                & \LARGE $\prelsinv$ (Direct Preference)\\[2pt]
                \tikz\fill[color4] (0,0) rectangle (0.6,0.3);
                & \LARGE $\equivrel$ (Indifference) \\[2pt]
                \tikz\fill[color2] (0,0) rectangle (0.6,0.3);
                & \LARGE $\prels$ (Inverse Preference)\\[2pt]
                \tikz\fill[color3] (0,0) rectangle (0.6,0.3);
                & \LARGE $\incomprel$ (Incomparability)
            \end{tabular}
        };
    \end{tikzpicture}
};

\end{tikzpicture}}
    \caption{Desiderata for a $d=2$ CbRL problem.}
    \label{fig:desiderata}
\end{figure}

\textbf{Learning Goal.}~~Consider a class of models $\mathcal{F}$ and a hypothesis space of utility functions $\mathcal{U} \subseteq \{\bm{u}: \Ts \to \mathbb{R}^d\}$.
Given a distribution $\mathcal{Q} \in \Delta(\Ts^2)$, denote as $\mathcal{P}_{f, \bm{u}} \in \Delta(\Ts^2 \times \Comps)$ the distribution of triples $(\tau, \tau', \comp)$ with $(\tau, \tau') \sim \mathcal{Q}$ and $\comp \sim f(\cdot | \bm{u}(\tau) - \bm{u}(\tau'))$ for some $f \in \mathcal{F}$ and $\bm{u} \in \mathcal{U}$.
Let $f^* \in \mathcal{F}$, $\bm{u}^{*} \in \mathcal{U}$ characterize the expert model.
We say that two pairs $(f, \bm{u}), (f', \bm{u}')$ are \emph{observationally equivalent} if $\mathcal{P}_{f, \bm{u}} = \mathcal{P}_{f', \bm{u}'}$.
From comparison triples alone, the goal of CbRL is to identify a model and utility function that are observationally equivalent to the expert. This can be achieved by minimizing the \emph{expected negative log-likelihood} (NLL):
\begin{equation}
\nonumber
    \argmin_{f \in \mathcal{F}, \bm{u} \in \mathcal{U}} \loss (f, \bm{u}) \coloneqq \mathbb{E}_{(\tau, \tau', \comp) \sim \mathcal{P}_{f^{*}, \bm{u}^{*}}} [\ell(\comp | \tau, \tau'; f, \bm{u})],
\end{equation}
where $\ell(\comp | \tau, \tau'; f, \bm{u}) \coloneqq - \log (f(\comp | \bm{u}(\tau) - \bm{u}(\tau')))$. Notably, this objective is similar to that of PbRL, as once a parametrized model class and hypothesis space are defined, the problem becomes finding the optimal parameters, and all downstream results follow from this.\footnote{Clearly, the pragmatic approach is to define $\mathcal{F}$ and $\mathcal{U}$ in terms of the same parameter space.}

Let $\mathcal{D} \coloneqq \{(\tau_i, \tau_i', \comp_i)\}_{i=1}^N \sim \mathcal{P}_{f^*, \bm{u}^*}^N$ be a dataset of i.i.d.\ comparison triples.
When a dataset is available, we can find the \emph{maximum likelihood} pair $(\widehat{f}, \widehat{\bm{u}})$, obtained by minimizing the \emph{empirical} NLL over the dataset, i.e.:
\begin{equation}
\label{eq:nll}
\textstyle
    ( \widehat{f}, \widehat{\bm{u}} ) \in \argmin_{f \in \mathcal{F}, \bm{u} \in \mathcal{U}} \losshat (f, \bm{u} | \mathcal{D}),
\end{equation}
where $\losshat (f, \bm{u} | \mathcal{D}) \coloneqq \frac{1}{N} \sum_{i \in \dsb{N}} \ell (\comp_i | \tau_i, \tau_i'; f, \bm{u})$.
In principle, we aim to select a class of models $\mathcal{F}$ such that the NLL of Equation~\eqref{eq:nll} is convex \citep{boyd2004convex}.
However, the following negative result shows that this is incompatible with the request that models in $\mathcal{F}$ be compliant with Definition~\ref{defi:desiderata}.

\begin{restatable}[Non-convex NLL]{prop}
{noncvx}\label{prop:noncvx}
    For every $d \ge 2$,  there exists no function $f: \mathbb{R}^d \!\to\! \Delta (\Comps)$ that satifies Definition~\ref{defi:desiderata} and such that \ $\!- \log f(\comp | \bm{\delta})$ is convex in $\bm{\delta}$ for every $\comp \!\in\! \Comps$.
\end{restatable}

This result provides a strong insight into the complexity of CbRL. Intuitively, non-convexity arises because the problem is natively symmetric w.r.t.\ every pair of contrasting objectives. Showing the expert the pair $(\tau, \tau')$ is equivalent to showing the pair $(\tau', \tau)$ if the two trajectories are incomparable. This symmetry forces the model to be non-convex along the non-standard diagonals, which represent conflicts between objectives.
Thus, Equation~\eqref{eq:nll} is, in general, non-convex.

\section{Novel Rationality Model}
\label{sec:model_based}

In this section, we propose the novel \emph{multi-objective Bradley Terry} (MOBT) model, discuss under which conditions it satisfies the desiderata, and explain how it relates to PbRL models. Then, we provide upper bounds on the comparison distribution estimation error, characterize reward identifiability, and provide downstream guarantees on the scalarized Pareto front.

To generalize the BT model to CbRL, we retain the original softmax structure \citep{bradley1952rank, luce1959individual} and model a score function for each outcome. In particular, we consider the standard deviation of the utility difference vector as the basis for the incomparability score. Intuitively, the idea is to quantify the dispersion of the utility differences, which is highest when the trajectories are in conflict.\footnote{We defer the formalization of the intuition behind the incomparability score to Appendix~\ref{apx:derivation}.}

\begin{defi}[MOBT Model]
\label{defi:mobt}
    Let $\bm{u}: \Ts \to \mathbb{R}^d$ be a $d$-dimensional utility function, $\gamma > 0$ be a \emph{conflict sensitivity} parameter, and let $\alpha, \beta \in \mathbb{R}$ encode the expert's \emph{propensity} towards indifference and incomparability, respectively. The \emph{multi-objective Bradley-Terry} (MOBT) model induced by $(\bm{u}, \alpha, \beta, \gamma)$ is the softmax:
    \begin{equation}
    \label{eq:mobt_model}
        f(\comp | \bm{\delta}) \coloneqq \frac{\exp \left( h_{\comp}(\bm{\delta}) \right)}{\sum_{\comp' \in \Comps} \exp \left( h_{\comp'} (\bm{\delta}) \right)},
    \end{equation}
    where $\bm{\delta} \coloneqq \bm{u}(\tau) - \bm{u}(\tau')$ for every $\tau, \tau' \in \Ts$, and the \emph{scores} $h_{\comp}: \mathbb{R}^d \to \mathbb{R}$ are defined as:\\
    \begin{subequations}
        \begin{minipage}{0.4\linewidth}
            \begin{align}
                \hspace{-0.1cm}h_{\prelsinv}(\bm{\delta}) \!\coloneqq\! \bm{1}_d^{\top} \bm{\delta}/2,
            \end{align}
        \end{minipage}
        \hfill
        \begin{minipage}{0.58\linewidth}
            \begin{align}
                h_{\prels}(\bm{\delta}) \!\coloneqq\! -\bm{1}_d^{\top} \bm{\delta}/2,
            \end{align}
        \end{minipage}\\[-0.1cm]
        \begin{minipage}{0.4\linewidth}
            \begin{align}
                \hspace{-0.1cm}h_{\equivrel}(\bm{\delta}) \!\coloneqq\! \alpha,
            \end{align}
        \end{minipage}
        \hfill
        \begin{minipage}{0.58\linewidth}
            \begin{align}
                h_{\incomprel}(\bm{\delta}) \!\coloneqq\! \gamma \sqrt{d} \std (\bm{\delta}) + \beta. \label{eq:mobt:incomparability}
            \end{align}
        \end{minipage}
    \end{subequations}
\end{defi}

The next result shows the $d$-dependent ranges of $\gamma$ for which the model is compliant.
\begin{restatable}{prop}{gammarangesimp}
\label{prop:gammarangesimp}
    Let $d \ge 2$. Define:
    \begin{align}
    \label{eq:gammarangesimp}
        \underline{\gamma}_d \coloneqq  \frac{(d-2)\sqrt{d}}{4\sqrt{d-1}}.
    \end{align}
    The MOBT model satisfies Definition~\ref{defi:desiderata} for every $\alpha > 0 \ge \beta$ and $\gamma > \underline{\gamma}_d$.
\end{restatable}
Note that $\underline{\gamma}_d$ is monothonically increasing in $d$ and has order $\underline{\gamma}_d = O(d)$ for $d \rightarrow +\infty$. For instance,  $\underline{\gamma}_2 = 0$ and $\underline{\gamma}_3 = \sqrt{6}/8$.
Notably, Proposition~\ref{prop:gammarangesimp} does not prescribe a range for $d=1$. Indeed, in such a case, we have that $h_{\incomprel}(\cdot) = \beta$ for every $\gamma > 0$ and Equation~\eqref{eq:desiderata:incomparability} is vacuous. Nonetheless, the model is well-defined, and it is straightforward to verify compliance for every $\alpha > 0 \ge \beta$. We now show that MOBT falls back to two well-known single-objective models.

\begin{restatable}[Equivalences of MOBT in PbRL]{prop}{mobtoned}
\label{prop:mobt1d}
    Consider an MDPC with $d=1$, a scalar utility function $u: \Ts \to \mathbb{R}$, and fix $\beta = -\infty$. Denote, for every $\tau, \tau' \in \Ts$ their utility difference as $\delta \coloneqq u(\tau) - u(\tau')$.

    If $\alpha \in \mathbb{R}$ is finite, then MOBT falls back to the BT extension by \citet{davidson1970extending}, with probabilities:
    \begin{align*}
        f(\prelsinv | \delta) &= \exp(\delta/2)/D(\delta),  \!&&f(\prels | \delta) = \exp(-\delta/2)/D(\delta), \\
        f(\equivrel | \delta) &= \nu / D(\delta), &&f(\incomprel | \delta) \:= 0.          
    \end{align*}
    where $D(\delta) \!\coloneqq\! \exp(\delta/2) + \exp(-\delta/2) + \nu$ and $\nu \!\coloneqq\! \exp(\alpha)$ is the \emph{index of discrimination} of \citet{davidson1970extending}.

    If $\alpha = -\infty$, then MOBT falls back to the original BT model with probabilities:
    \begin{align*}
        f(\prelsinv | \delta) &= \sigma(\delta), \quad &&f(\prels | \delta) = 1 - \sigma(\delta), \\
        f(\equivrel | \delta) &= 0, \quad &&f(\incomprel | \delta) \:= 0.
    \end{align*}
    where $\sigma(\cdot)$ is the sigmoid function.
\end{restatable}

\subsection{Expert Model Estimation}
We now study the sample complexity of recovering a multi-dimensional reward under the MOBT model when the utility function is parametric.
Let $\Theta$ be the parameter space, to be fully specified later. Then, with a slight abuse of notation, we rewrite Equation~\eqref{eq:nll} as:
\begin{equation}
\label{eq:mle_theta}
    \textstyle
    \widehat{\bm{\theta}} \in \argmin_{\bm{\theta} \in \Theta} \losshat (\bm{\theta} | \mathcal{D})
\end{equation}
where:
\begin{equation}
\label{eq:nll_theta}
    \textstyle
    \widehat{\mathcal{L}}(\bm{\theta} | \mathcal{D}) \coloneqq \frac{1}{N} \sum_{(\tau_i, \tau_i', \comp_i) \in \mathcal{D}} \ell(\comp_i | \tau_i, \tau_i'; \bm{\theta}),
\end{equation}
and $\ell(\comp | \tau, \tau'; \bm{\theta}) \!\coloneqq\! - \log f(\comp | \tau, \tau'; \bm{\theta})$, with $f(\comp | \tau, \tau'; \bm{\theta})$ being the parametric version of Equation~\eqref{eq:mobt_model}, formalized after Assumption~\ref{asm:features}.
For convenience of analysis, we consider the case of feature-based linear utilities.
\begin{asm}[Feature-based Linear Utility]
\label{asm:features}
    There exists a known feature map $\bm{\varphi} : \Ss \times \As \to \mathbb{R}^k$ and an unknown matrix $\bm{M} \in \mathbb{R}^{d \times k}$ such that the expert's utility function is $\bm{u}(\tau) \coloneqq \bm{M} \bm{\phi}(\tau)$, where $\bm{\phi}(\tau) \coloneqq \sum_{h \in \dsb{H}} \bm{\varphi}(s_h, a_h)$ for every $\tau \in \Ts$.
\end{asm}
Assumption~\ref{asm:features} is widely adopted in the literature \citep{myers2021learning,buening2025strategyproof}, as $\bm{\phi}$ can come from feature engineering \citep{abbeel2004apprenticeship} or unsupervised pre-training \citep{lee2021pebble}, whereas linearity is standard in both theoretical \citep{jin2023provably} and practical \citep{silver2014deterministic} RL.
Under Assumption~\ref{asm:features}, every $\bm{\theta} = (\mathrm{vec}(\bm{M}), \alpha, \beta, \gamma) \in \Theta$ identifies the pair $(f, \bm{u})$, where $\bm{u} = \bm{M} \bm{\phi}$ and $f$ is the MOBT model induced by $(\bm{u}, \alpha, \beta, \gamma)$. We denote as $\mathcal{P}_{\bm{\theta}}, \ell(\comp | \tau, \tau'; \bm{\theta}), \loss(\bm{\theta})$, and $\losshat(\bm{\theta} | \mathcal{D})$ the parametric versions of $\mathcal{P}_{f, \bm{u}}, \ell(\comp | \tau, \tau'; f, \bm{u}), \loss(f, \bm{u})$, and $\losshat(f, \bm{u}| \mathcal{D})$, respectively.
As customary in the literature, we consider a boundedness assumption of the parameter $\bm{\theta}$ \citep{sadigh2017active,zhao2024rapbrl} and of feature differences \citep{wang2020reward,zhang2021reward}.

\begin{asm}[Boundedness]
\label{asm:bounded}
    There exist $R, \Lambda \ge 1$ such that:
    \begin{enumerate}[noitemsep, topsep=-1pt, label=(\roman*)]
    \item The expert's model is parametrized by $\bm{\theta}^{*} \coloneqq (\mathrm{vec}(\bm{M}^{*}), \alpha^{*}, \beta^{*}, \gamma^{*}) \in \Theta \subseteq [-R, R]^{dk} \times (0, R] \times [-R, 0] \times (\underline{\gamma}_d,R]$;
    \item The known feature map $\bm{\phi} : \Ts \to \mathbb{R}^k$ satisfies $\| \bm{\phi}(\tau) - \bm{\phi}(\tau') \|_2 \le \Lambda$ for every $\tau, \tau' \in \Ts$.
    \end{enumerate}
\end{asm}
Assumptions~\ref{asm:features} and~\ref{asm:bounded} together ensure that the utility $\bm{u}(\tau) = \bm{M} \phi(\tau)$ is bounded for every $\tau \in \Ts$.

We can now state our sample complexity result.

\begin{restatable}{thr}{samplecomplexity}
\label{thr:sample_complexity}
    Let $d \ge 2$, let $\bm{\theta}^{*} \in \Theta$ be the expert's parametrization, and let $\mathcal{D} = \{(\tau_i, \tau_i', \comp_i)\}_{i \in \dsb{N}}$ be a dataset of $N$ i.i.d.\ samples from $\mathcal{P}_{\bm{\theta}^{*}}$, for some $\mathcal{Q} \in \Delta(\Ts^2)$. Let $\widehat{\bm{\theta}} \in \argmin_{\bm{\theta} \in \Theta} \widehat{\mathcal{L}}(\bm{\theta} | \mathcal{D})$, and let $\mathcal{P}_{\widehat{\bm{\theta}}}$ be its induced distribution under $\mathcal{Q}$. Under Assumptions~\ref{asm:features} and~\ref{asm:bounded}, for every $\delta \in (0,1)$, with probability at least $1 - \delta$, it holds that:
    \begin{equation}
    \label{eq:sample_complexity}
        \dkl{\mathcal{P}_{\bm{\theta}^{*}}}{\mathcal{P}_{\widehat{\bm{\theta}}}} \le \widetilde{\mathcal{O}} \bigg( R \Lambda d^{3/2} k \sqrt{\frac{\log(1/\delta)}{N}} \bigg).
    \end{equation}
\end{restatable}
Some comments are in order. First, assuming that $\mathcal{D}$ comprises $N$ i.i.d.\ triples sampled from $\mathcal{P}_{\bm{\theta}^{*}}$, under the chain rule of the KL divergence, Equation~\eqref{eq:sample_complexity} bounds the excess \emph{population} NLL $\loss (\widehat{\bm{\theta}}) - \loss (\bm{\theta}^{*})$. Second, the bound decreases with the number of samples and increases linearly in $R$, $\Lambda$, and $k$, and polynomially in $d$.
Moreover, our result extends to the case in which the dataset is sampled from a distribution $\mathcal{P}^{*}$ with marginal $\mathcal{Q}$ over trajectory pairs that fall outside the MOBT class of models. In such a case, the same proof yields $\dkl{\mathcal{P}^{*}}{\mathcal{P}_{\widehat{\bm{\theta}}}} \le \inf_{\bm{\theta} \in \Theta} \dkl{\mathcal{P}^{*}}{\mathcal{P}_{\bm{\theta}}}$ plus the right-hand side of Equation~\eqref{eq:sample_complexity}, where the first term quantifies the approximation error of the MOBT class.

The optimization of Equation~\eqref{eq:mle_theta} incurs two complexities. 
First, Theorem~\ref{thr:sample_complexity} defines $\widehat{\bm{\theta}}$ as the \emph{solution} to Equation~\eqref{eq:mle_theta}, i.e., its global optimum. However, this optimization is, in general, non-convex (Proposition~\ref{prop:noncvx}). Indeed, since the scores $h_{\prelsinv}$, $h_{\prels}$, and $h_{\equivrel}$ are affine in $\bm{\delta}$ and $h_{\incomprel}$ is convex, $-\log f(\comp_i | \tau_i,\tau_i',\bm{\theta})$ is convex for every $\comp_i \neq \incomprel$. This implies that, whenever $\mathcal{D}$ contains trajectory pairs labeled as incomparable, the overall NLL is, in general, non-convex, and numerical optimization methods \citep{nocedal2006numerical} can, in general, at most return a \emph{local} optimum.
Second, there are infinitely many parametrizations that induce the same MOBT model. In particular, there are infinite combinations of  $(\bm{M}, \gamma)$ such that the parametrizations are equivalent, as discussed in the following result.
\begin{remark}[Conflict-Scale Invariance]
\label{rem:scale_invariance}
    For every $c > 0$, let $T_c(\bm{M}) \coloneqq \bm{A} \bm{M} + c \bm{\Pi} \bm{M}$ be a \emph{transformation} operator. Since $\bm{1}_d^\top T_c(\bm{M}) = \bm{1}_d^\top \bm{M}$ and $\bm{\Pi} T_c(\bm{M}) = c \bm{\Pi} \bm{M}$, it follows from straightforward algebra that the two parametrizations $(\mathrm{vec}(\bm{M}), \alpha, \beta, \gamma)$ and $(\mathrm{vec}(T_c(\bm{M})), \alpha, \beta, \gamma / c)$ induce the same MOBT model. Thus, $\gamma$ is not identifiable jointly with $\bm{M}$.
\end{remark}

Hence, to simplify the optimization problem, we fix $\gamma$ to some $\gamma_0$. This is without loss of generality at the level of comparison distributions. Moreover, to ensure boundedness as in Assumption \ref{asm:bounded}, we can choose $\gamma_0 \ge \gamma^{*}$ to guarantee that $T_{\gamma^{*}/\gamma_0}(\bm{M}^{*}) \!\in\! [-R, R]^{dk}$. Thus, we can set 
$\gamma_0 \ge \gamma_\dagger$ if an upper bound $\gamma^{*} \le \gamma_\dagger$ is known. Over-estimating $\gamma^{*}$, i.e., $c = \gamma^{*}/\gamma_0 < 1$, only affects the set of true scalarization weights that the recovered policies explore (Proposition~\ref{prop:scalarization}). Notably, $\gamma_0$ parametrizes only the learned model, so that we can allow $\gamma_0 > \underline{\gamma}_d$.
We now bound the KL divergence between the distribution induced by \emph{any} attainable local optimum and $\mathcal{P}_{\bm{\theta}^{*}}$.

\begin{restatable}{thr}{samplecomplexitylocal}
\label{thr:sample_complexity_local}
    Let $d \ge 2$. Denote the expert's parametrization as $\bm{\theta}^{*} = (\mathrm{vec}(\bm{M}^{*}), \alpha^{*}, \beta^{*}, \gamma^{*}) \in \Theta$, and let $\mathcal{D} = \{(\tau_i, \tau_i', \comp_i)\}_{i \in \dsb{N}}$ be a dataset of $N$ i.i.d.\ triples sampled from $\mathcal{P}_{\bm{\theta}^{*}}$, for some $\mathcal{Q} \in \Delta(\Ts^2)$. Let $\gamma_0 > \underline{\gamma}_d$ be s.t.\ $\gamma_0 \ge \gamma^{*}$. Let $\widehat{\bm{\theta}}_{\mathrm{loc}}$ be a local minimum of Equation~\eqref{eq:mle_theta} over $\Theta_{\gamma_0} \!\coloneqq\! [-R, R]^{dk} \!\times\! (0, R] \!\times\! [-R, 0] \!\times\! \{\gamma_0\}$, i.e., the parameter box of Assumption~\ref{asm:bounded} with $\gamma$ fixed to $\gamma_0$, and let $\mathcal{P}_{\widehat{\bm{\theta}}_{\mathrm{loc}}}$ be its induced distribution under $\mathcal{Q}$.
    Under Assumptions~\ref{asm:features} and~\ref{asm:bounded}, for every $\delta \in (0,1)$, w.p.\ at least $1 - \delta$, it holds that:
    \begin{equation}
    \label{eq:sample_complexity_local}
    \begin{aligned}
        \dkl{\mathcal{P}_{\bm{\theta}^{*}}}{\mathcal{P}_{\widehat{\bm{\theta}}_{\mathrm{loc}}}} \le \widetilde{\mathcal{O}} \bigg( & \gamma^* R \Lambda d^{3/2} k \sqrt{\frac{\log(1/\delta)}{N}} \\
        & + \gamma^{*} R \Lambda \sqrt{dk} \, \mathcal{P}_{\bm{\theta}^{*}}(\incomprel) \bigg),
    \end{aligned}
    \end{equation}
    where $\mathcal{P}_{\bm{\theta}^{*}}(\incomprel) \coloneqq \mathbb{E}_{(\tau, \tau') \sim \mathcal{Q}} [ f(\incomprel | \tau, \tau'; \bm{\theta}^{*}) ]$ is the marginal probability of observing an incomparability.
\end{restatable}
The first term of Equation~\eqref{eq:sample_complexity_local} has the same dependencies as the bound of Theorem~\ref{thr:sample_complexity}. 
Notably, moving from the global optimum, which we are not guaranteed to reach, to \emph{any} local optimum adds to the bound only a term proportional to $\mathcal{P}_{\bm{\theta}^{*}}(\incomprel)$, which is an intrinsic property of the problem, depending only on the expert's parametrization and on the design $\mathcal{Q}$. Thus, it can be interpreted as an \emph{index of conflict} of the problem. Lemma~\ref{lem:incomp_rate} further bounds this term as:
\begin{equation}
\label{eq:incomp_rate_main}
    \mathcal{P}_{\bm{\theta}^{*}}(\incomprel) \le f(\incomprel | \bm{0}_d) + c_{\beta^{*}} \gamma^{*} \sqrt{d} \mathbb{E}_{(\tau, \tau') \sim \mathcal{Q}} [ \std ( \bm{\delta} ) ],
\end{equation}
where $\bm{\delta} \!=\! \bm{M}^{*}(\bm{\phi}(\tau) - \bm{\phi}(\tau'))$, 
$f(\incomprel | \bm{0}_d)$
is the probability of incomparability given two trajectories of \emph{equal} utility, and $c_{\beta^{*}} \!\coloneqq\! \min \{ 1/4, \exp(A + \beta^{*}) / (2 + \exp(\alpha^{*})) \}$, with $A \coloneqq d R \Lambda \sqrt{k} / (2\sqrt{d-1})$. The second term of Equation~\eqref{eq:incomp_rate_main} represents the extent to which the design $\mathcal{Q}$ presents the expert with conflicting pairs. Notably, both terms vanish in the limit $\exp(\beta^{*}) \to 0$, consistently with the fact that, without incomparabilities, every local minimum is also global.
In Appendix~\ref{apx:global_2d}, we propose an efficient optimization procedure to compute a global minimum of Equation~\eqref{eq:nll_theta} over $\Theta_{\gamma_0}$ for $d=2$, whenever attained, as the minimum of finitely many, in particular $\mathrm{poly}(N)$, convex functions, one per region of the arrangement of hyperplanes induced by the incomparability-labeled samples, for a fixed number of features $k$. However, this convenient property does not translate straightforwardly to $d \ge 3$, leaving it as an open question.

\subsection{Identifiability of the Reward}
\label{sec:identifiability}

Theorems~\ref{thr:sample_complexity} and~\ref{thr:sample_complexity_local} control the distribution \emph{induced} by the learned parameters, but cannot control the parameters themselves. 
Nonetheless, we need to extract a reward function from the learned parameters to feed to the downstream MORL algorithm.
We now study how observational equivalence translates to the underlying parametrizations. In particular, we show that, for $d=2$, observationally equivalent parametrizations differ at most by a permutation of the dimensions, while for $d \ge 3$ there exist infinite orthogonal transformations that maintain equivalence.
We now present a design condition on the comparison distribution $\mathcal{Q}$.
\begin{asm}[Rich Comparison Design]
\label{asm:rich_design}
    Let $\mathcal{X} \!\subseteq\! \mathbb{R}^k$ be the support of the distribution of $\bm{\phi}(\tau) - \bm{\phi}(\tau')$, with $(\tau, \tau') \sim \mathcal{Q}$. Then, ($i$) $\mathrm{span} \{ \bm{x} \bm{x}^\top : \bm{x} \in \mathcal{X} \} = \mathbb{S}^k$, where $\mathbb{S}^k$ is the space of $k \times k$ symmetric matrices, and ($ii$) there exist $\bm{x} \in \mathcal{X}$ and $t \neq \{-1, 1\}$ s.t.\ $t \bm{x} \in \mathcal{X}$.
\end{asm}
Item ($i$) is a standard condition to identify a quadratic form \citep[see, e.g.,][]{candes2013phaselift}. It requires at least $k(k+1)/2$ distinct feature differences, and implies $\mathrm{span}(\mathcal{X}) = \mathbb{R}^k$, which is the usual full-rank condition to identify a linear BT reward \citep{zhu2023principled}. Item ($ii$) simply requires that $\bm{x}$ and $t\bm{x}$ are observable under $\mathcal{Q}$, and allows isolating the expert's baseline propensity towards incomparability $\beta$ from the incomparability caused by conflict, since $h_{\incomprel}(\bm{M}\bm{x}) = \gamma \|\bm{\Pi} \bm{M} \bm{x}\|_2 + \beta$ and $h_{\incomprel}(t\bm{M}\bm{x}) = \gamma |t| \|\bm{\Pi} \bm{M} \bm{x}\|_2 + \beta$.
\begin{restatable}[Identifiability of the Reward]{thr}{identifiabilitymain}
\label{thr:identifiability_main}
    Let $d \ge 1$ and $\mathcal{Q} \in \Delta(\Ts^2)$. Let $\mathcal{G}_d \coloneqq \{ \bm{U} \in \mathbb{R}^{d \times d} : \bm{U}^\top \bm{U} = \bm{I}_d, \, \bm{U} \bm{1}_d = \bm{1}_d \}$ be the set of \emph{orthogonal transformations}.
    Fix $\gamma = \gamma_0 > 0$ and define the (unknown) quantity $c \coloneqq \gamma^{*} / \gamma_0$.
    Under Assumptions~\ref{asm:features} and~\ref{asm:rich_design}, $\bm{\theta} = (\mathrm{vec}(\bm{M}), \alpha, \beta, \gamma)$ is \emph{observationally equivalent} to $\bm{\theta}^{*}$, i.e., $\mathcal{P}_{\bm{\theta}} = \mathcal{P}_{\bm{\theta}^{*}}$, if and only if $\alpha = \alpha^{*}$, $\beta = \beta^{*}$, and $\bm{M} = T_c(\bm{U} \bm{M}^{*})$ for some $\bm{U} \!\in\! \mathcal{G}_d$, where $T_c$ is defined as in Remark~\ref{rem:scale_invariance}.
    Moreover, the aggregate utility is identified for every $\tau \!\in\! \Ts$, i.e., $\bm{1}_d^\top \bm{M} \bm{\phi}(\tau) \!=\! \bm{1}_d^\top \bm{M}^{*} \bm{\phi}(\tau)$, while the conflict component $\bm{\Pi}\bm{M}$ is identified only up to scale $c$ and:
    \begin{itemize}[topsep=-1pt, noitemsep, leftmargin=35pt]
        \item[$(d=1)$] no further ambiguity, as $\bm{\Pi} = \bm{0}$;
        \item[$(d=2)$] up to a relabeling of the two objectives;
        \item[$(d \ge 3)$] up to an uncountable family of orthogonal transformations (including rotations and reflections) of the plane orthogonal to $\bm{1}_d$, if $\bm{\Pi} \bm{M}^{*} \neq \bm{0}$.
    \end{itemize}
\end{restatable}
The MOBT model depends on $\bm{\delta}$ only through the utility difference and the standard deviation, and every $\bm{U} \!\in\! \mathcal{G}_d$ preserves both. Thus, comparisons cannot identify individual objectives. The consequences differ sharply with the dimension. For $d=1$, no further ambiguity is present since the set of transformations $\mathcal{G}_1 = \{1\}$ is a singleton containing the identity. For $d = 2$, the only ambiguity besides the scale $c$ is a relabeling of the objectives. Here, the set of transformations is made of the two permutation matrices $\textstyle\mathcal{G}_2 = \left\{ \begin{pmatrix} 1 & 0 \\ 0 & 1 \end{pmatrix}, 
\begin{pmatrix} 0 & 1 \\ 1 & 0 \end{pmatrix} \right\}$.
Finally, for $d \ge 3$ the identified set is uncountable even when $\gamma_0 = \gamma^{*}$. Moreover, when $\gamma_0 \ge \gamma^{*}$, we have $c \le 1$, shrinking the conflict component of the learned reward toward the aggregate utility.

\subsection{Recovering the Scalarized Pareto Front}
\label{sec:pareto}

We now study how the invariance in the recovered reward of Theorem~\ref{thr:identifiability_main} affects downstream guarantees. First, we study how such an invariance translates under linear scalarization of the rewards.\footnote{We limit to linear scalarization that allows recovering the full Pareto front when it is convex \citep{roijers2013survey}.} Then, we provide optimality guarantees of the learned policies under such a linear scalarization.
By Theorem~\ref{thr:identifiability_main}, any parametrization $\widehat{\bm{\theta}}$ that is observationally equivalent to the expert's, \ie $\mathcal{P}_{\widehat{\bm{\theta}}} = \mathcal{P}_{\bm{\theta}^{*}}$, has utility matrix $\widehat{\bm{M}} = T_c(\bm{U} \bm{M}^{*})$ for some unknown $\bm{U} \in \mathcal{G}_d$, with $c = \gamma^{*} / \gamma_0 \le 1$ once $\gamma$ is fixed to $\gamma_0 \ge \gamma^{*}$.
Computing the Pareto front given $\widehat{\bm{M}}$ is then a standard MORL problem with reward $\widehat{\bm{M}} \bm{\varphi}(s, a)$ \citep[see, e.g.,][]{roijers2013survey, hayes2022practical, chen2026reward}.
We now discuss downstream guarantees.

We denote as $\bm{\psi}(\pi) \coloneqq \mathbb{E}_{\tau \sim \pi} [\bm{\phi}(\tau)] \in \mathbb{R}^k$ the feature expectation induced by $\pi$ \citep{abbeel2004apprenticeship}, so that, under Assumption~\ref{asm:features}, the expected utility of $\pi$ under $\bm{M}$ is $\bm{M} \bm{\psi}(\pi)$. Given $\varepsilon \ge 0$, we say that a policy $\pi \in \Pi$ is \emph{$\varepsilon$-Pareto optimal} for $\bm{M}$ if there exists no $\pi' \in \Pi$ such that $\bm{M} (\bm{\psi}(\pi') - \bm{\psi}(\pi)) > \varepsilon \bm{1}_d$ \citep{papadimitriou2000approximability}, \emph{weakly Pareto optimal} if this holds with $\varepsilon = 0$, and \emph{Pareto optimal} if there exists no $\pi' \in \Pi$ such that $\bm{M} (\bm{\psi}(\pi') - \bm{\psi}(\pi)) > \mathbf{0}_d$ \citep{miettinen1999nonlinear}. Given a set of weights $\mathcal{W} \subseteq \Delta_{d-1}$,
we define the \emph{scalarized front} of $\bm{M}$ over $\mathcal{W}$ as the set of policies $\Pi^{*}(\bm{M}; \mathcal{W}) \coloneqq \bigcup_{\bm{w} \in \mathcal{W}} \argmax_{\pi \in \Pi} \bm{w}^\top \bm{M} \bm{\psi}(\pi)$. Finally, for every $r \ge 0$, define as $\mathcal{W}_r \coloneqq \{ \bm{w} \in \Delta_{d-1} : \| \bm{w} - \bm{1}_d / d \|_2 \le r \}$ the set of scalarization weights distant at most $r$ from the uniform one, and define $\rho_{\mathrm{in}} \coloneqq 1 / \sqrt{d(d-1)}$ as the inradius of the simplex, yielding $\mathcal{W}_{\rho_{\mathrm{in}}} \subseteq \Delta_{d-1}$.
We now study how invariant matrix transformations affect scalarization weights and, consequently, Pareto optimality guarantees.
\begin{restatable}[Scalarization under Invariance]{prop}{paretoscalarization}
\label{prop:scalarization}
    Let $d \ge 2$, $\bm{M} \in \mathbb{R}^{d \times k}$, $\bm{U} \in \mathcal{G}_d$, and $c > 0$. For every $\bm{w} \in \Delta_{d-1}$, it holds that $\bm{w}^\top T_c(\bm{U} \bm{M}) = \widetilde{\bm{w}}^\top \bm{M}$, where $\widetilde{\bm{w}} \coloneqq \bm{1}_d / d + c \, \bm{U}^\top (\bm{w} - \bm{1}_d / d)$ satisfies $\bm{1}_d^\top \widetilde{\bm{w}} = 1$ and $\| \widetilde{\bm{w}} - \bm{1}_d / d \|_2 = c \| \bm{w} - \bm{1}_d / d \|_2$.
    Moreover, if $c \le 1$, then, for every $r \in [0, \rho_{\mathrm{in}}]$, it follows that:
    \begin{equation}
        \label{eq:scalarization_main}
        \Pi^{*}(T_c(\bm{U} \bm{M}); \mathcal{W}_r) = \Pi^{*}(\bm{M}; \mathcal{W}_{c r}).
    \end{equation}
\end{restatable}
Proposition~\ref{prop:scalarization} shows that the invariance due to $T_c$ and $\bm{U}$ yields a map $\bm{w} \mapsto \widetilde{\bm{w}}$.
However, even if $\bm{w} \in \Delta_{d-1}$, $\widetilde{\bm{w}}$ is not guaranteed to belong to $\Delta_{d-1}$, since the simplex is not invariant under $\bm{U}^\top$.
Notably, this cannot happen when $\|\bm{w} - \bm{1}_d/d \|_2 \le \rho_{\mathrm{in}}$ and $c \le 1$, since in that case $\| \widetilde{\bm{w}} - \bm{1}_d / d \|_2 \le c \rho_{\mathrm{in}} \le \rho_{\mathrm{in}}$, and thus $\widetilde{\bm{w}} \in \Delta_{d-1}$.
Moreover, Equation~\eqref{eq:scalarization_main} follows since, for every $\bm{w} \in \Delta_{d-1}$, the functions $\pi \mapsto \bm{w}^\top T_c(\bm{U} \bm{M}) \bm{\psi}(\pi)$ and $\pi \mapsto \widetilde{\bm{w}}^\top \bm{M} \bm{\psi}(\pi)$ coincide, and thus have the same maximum and maximizers, and since, for $r \le \rho_{\mathrm{in}}$ and $c \le 1$, the map $\bm{w} \mapsto \widetilde{\bm{w}}$ is a bijection between $\mathcal{W}_r$ and $\mathcal{W}_{c r}$. This implies that, for weights close enough to the uniform one, scalarizing the transformed reward $T_c(\bm{U}\bm{M})$ by a weight $\bm{w}$ is equivalent to scalarizing the original reward $\bm{M}$ by the weight $\widetilde{\bm{w}} \in \Delta_{d-1}$. Hence, we can transfer optimality guarantees from $\widehat{\bm{M}}$ to $\bm{M}^{*}$.

\begin{restatable}[Scalarized Front Guarantees]{thr}{paretoguarantees}
\label{thr:pareto}
    Let $d \ge 2$, $\bm{M}^{*} \in \mathbb{R}^{d \times k}$, $\bm{U} \in \mathcal{G}_d$, and $c \in (0, 1]$. 
    Define $\widehat{\bm{M}} \coloneqq T_c(\bm{U} \bm{M}^{*})$.
    Optimality under the aggregate reward is always preserved, i.e.:
    \begin{align}
    \nonumber
        \argmax_{\pi \in \Pi} \bigg(\frac{\bm{1}_d}{d}\bigg)^\top \bm{M}^{*} \bm{\psi}(\pi) = \argmax_{\pi \in \Pi} \bigg(\frac{\bm{1}_d}{d}\bigg)^\top \widehat{\bm{M}} \bm{\psi}(\pi),
    \end{align}
    and every $\pi \in \Pi^{*} (\widehat{\bm{M}}; \mathcal{W}_r)$ is Pareto optimal w.r.t.\ $\bm{M}^{*}$ for every $r \in [0, \rho_{\mathrm{in}})$.
    Moreover, if $d=2$, $\Pi^{*} (\widehat{\bm{M}}; \Delta_1) \subseteq \Pi^{*} (\bm{M}^{*}; \Delta_1)$, holding with equality if $c=1$.
    Finally, if $d=3$, letting $B \coloneqq \max_{i \in \dsb{d}} \sup_{\pi, \pi' \in \Pi} \bm{e}_i^\top \bm{M}^{*} (\bm{\psi}(\pi) - \bm{\psi}(\pi'))$ be the largest range of an objective, for every $r \ge \rho_{\mathrm{in}}$, every $\pi \in \Pi^{*}(\widehat{\bm{M}}; \mathcal{W}_r)$ is ($\eta B / (1 + \eta)$)-Pareto optimal w.r.t.\ $\bm{M}^{*}$, with $\eta \coloneqq \max \{0, \sqrt{2/3} c r - 1/3\}$.
\end{restatable}
Theorem~\ref{thr:pareto} shows that policies that are optimal for the \emph{aggregate} objective, i.e., $\bm{w} = \bm{1}_d/d$, are always recovered. Then, the behavior varies with the number of objectives, with Pareto optimality guaranteed for $d=2$, and guarantees depending on the distance from the aggregate objective for $d=3$ (see Figure~\ref{fig:simplex_regions}).
Notably, under Assumption~\ref{asm:bounded}, $B \le R \Lambda \sqrt{k}$.
We remark that Theorem~\ref{thr:pareto} is a one-sided guarantee. Indeed, the policies recovered from the learned reward are (approximately) Pareto optimal w.r.t.\ $\bm{M}^{*}$, but they do not cover the whole Pareto front of $\bm{M}^{*}$, since, by Proposition~\ref{prop:scalarization}, only the weights in $\mathcal{W}_{cr}$ are explored.
Finally, in Appendix~\ref{apx:finite_downstream}, we show how the finite-sample guarantees of Theorem~\ref{thr:sample_complexity_local} transfer to the Pareto front, as an immediate consequence of Theorems~\ref{thr:identifiability_main} and~\ref{thr:pareto}. 
In Appendix~\ref{apx:example}, we provide an illustrative example on a GridWorld environment with $d \in \{2, 3\}$ objectives, highlighting the invariance characterized in Theorem~\ref{thr:identifiability_main}.

\begin{figure}[t]
    \centering
\begin{tikzpicture}[
    radius arrow/.style={-latex, semithick, black},
]
    \def\R{1.55}

    \coordinate (C) at (0,0);
    \coordinate (e1) at (210:\R);
    \coordinate (e2) at (330:\R);
    \coordinate (e3) at (90:\R);

    \fill[vibrantOrange!35] (e1) -- (e2) -- (e3) -- cycle;
    \fill[vibrantTeal!45] (C) circle (\R/2);
    \draw[vibrantTeal, thin] (C) circle (\R/2);
    \draw[thick, line join=round] (e1) -- (e2) -- (e3) -- cycle;

    \draw[radius arrow] (C) -- (270:\R/2) node[midway, right, inner sep=2pt] {$\rho_{\mathrm{in}}$};

    \fill (C) circle (1.3pt);
    \node[above, inner sep=2pt] at (C) {$\bm{1}_3/3$};

    \node[below left, inner sep=2pt] at (e1) {$\bm{e}_1$};
    \node[below right, inner sep=2pt] at (e2) {$\bm{e}_2$};
    \node[above, inner sep=3pt] at (e3) {$\bm{e}_3$};

    \begin{scope}[shift={(1.85, 0.9)}, every node/.style={anchor=west, align=left, font=\small}]
        \fill[vibrantTeal!45] (0, -0.12) rectangle (0.45, 0.12);
        \draw[vibrantTeal, thin] (0, -0.12) rectangle (0.45, 0.12);
        \node at (0.55, 0) {$\mathcal{W}_{\rho_{\mathrm{in}}}$:\\weakly Pareto optimal};
        \fill[vibrantOrange!35] (0, -1.12) rectangle (0.45, -0.88);
        \node at (0.55, -1) {$\Delta_2 \setminus \mathcal{W}_{\rho_{\mathrm{in}}}$:\\$\varepsilon$-Pareto optimal};
    \end{scope}
\end{tikzpicture}
    \caption{Scalarization weights on the simplex $\Delta_2$. Weights within distance $\rho_{\mathrm{in}}$ of the uniform weight (teal) yield weakly Pareto optimality w.r.t.\ $\bm{M}^{*}$, while the remaining weights (orange) only guarantee $\varepsilon$-Pareto optimality (Theorem~\ref{thr:pareto}).}
    \label{fig:simplex_regions}
\end{figure}

\section{Discussion and Open Challenges}
\label{sec:conclusion}
We now discuss open questions related to our approach and how they may be addressed.

\textbf{Identifiability.}~~First, our MOBT model makes use of four outcomes to learn the parameters and, ultimately, the reward function and the Pareto front. Theorems~\ref{thr:identifiability_main} and~\ref{thr:pareto} illustrate how the identifiability properties degrade as $d$ increases and become vacuous for $d > 3$. This is intuitively explainable since, just as the BT model with two outcomes can identify a scalar reward up to translations, by moving to a four-outcome feedback setting we can identify more, but not for an arbitrary dimension $d$. We conjecture that, with four outcomes, there is an intrinsic limit on the maximum $d$ for which the reward and/or Pareto front can be identified to a satisfactory extent (i.e., for instance, up to dimension relabeling).

\textbf{Other Desiderata.}~~As further support for this, $d \in \{2,3\}$ allows us to enforce additional intuitive desiderata on the model, such as \emph{dominance compliance}, i.e., the requirement that a clear preference has higher probability than incomparability when all components of $\bm{\delta}$ are concordant (all positive or all negative) (see Definition~\ref{defi:dominance}). Indeed, this desideratum is guaranteed when, in addition to the requirements of Proposition~\ref{prop:gammarangesimp}, we impose $\gamma < \overline{\gamma}_d \coloneqq \frac{\sqrt{d}}{2\sqrt{d-1}}$, and the set $\Gamma_d = (\underline{\gamma}_d,\overline{\gamma}_d)$ is non-empty only for $d\in\{2,3\}$ (see Proposition~\ref{prop:gammarange}). This is attributable to the definition of the incomparability score, which is based on the standard deviation and quantifies disagreement among objectives in \emph{magnitude} but not in \emph{sign}.

\textbf{Non-convexity.}~~Finally, the non-convexity of the NLL induces the additive term of Theorem~\ref{thr:sample_complexity_local}, which does not decrease with $N$. For $d = 2$, this term can be avoided, since we can compute the global minimum by solving a finite number of convex programs (Proposition~\ref{prop:global_2d}). Devising an efficient global optimization procedure for $d = 3$ remains an open problem.

Overall, we studied the opportunities and limitations of four-class comparison feedback generated by a human expert. We showed that an observationally equivalent model can be recovered efficiently, although guarantees on the downstream recovery of the reward function and the Pareto front may still be limited by identifiability issues.

\clearpage


\bibliography{biblio}

@inproceedings{abbeel2004apprenticeship,
  title={Apprenticeship learning via inverse reinforcement learning},
  author={Abbeel, Pieter and Ng, Andrew Y.},
  booktitle={Proceedings of the International Conference on Machine Learning (ICML)},
  series       = {{ACM} International Conference Proceeding Series},
  volume       = {69},
  publisher    = {{ACM}},
  year         = {2004}
}

@inproceedings{abdelkareem2022advances,
  title={Advances in preference-based reinforcement learning: A review},
  author={Abdelkareem, Youssef and Shehata, Shady and Karray, Fakhri},
  booktitle={IEEE International Conference on Systems, Man, and Cybernetics (SMC)},
  pages={2527--2532},
  year={2022},
  organization={IEEE}
}

@article{adikari2026lemur,
  title={LEMUR: Learning to Align with Multi-Objective Reinforcement Learning from Preference Feedback},
  author={Adikari, Manith and Peng, Bei and Vinanzi, Samuele and Cangelosi, Angelo},
  journal={arXiv preprint arXiv:2607.29559},
  year={2026}
}

@article{amodei2016concrete,
  title={Concrete problems in {AI} safety},
  author={Amodei, Dario and Olah, Chris and Steinhardt, Jacob and Christiano, Paul and Schulman, John and Man{\'e}, Dan},
  journal={arXiv preprint arXiv:1606.06565},
  year={2016}
}

@article{astudillo2025preferential,
  title={Preferential Multi-Objective {B}ayesian Optimization},
  author={Astudillo, Raul and Li, Kejun and Tucker, Maegan and Cheng, Chu Xin and Ames, Aaron D. and Yue, Yisong},
  journal={Transactions on Machine Learning Research},
  year={2025}
}

@article{aumann1962utility,
  title={Utility theory without the completeness axiom},
  author={Aumann, Robert J},
  journal={Econometrica: Journal of the Econometric Society},
  volume={30},
  number={3},
  pages={445--462},
  year={1962},
  publisher={JSTOR}
}

@inproceedings{biyik2020active,
  title={Active Preference-Based {G}aussian Process Regression for Reward Learning},
  author={B{\i}y{\i}k, Erdem and Huynh, Nicolas and Kochenderfer, Mykel J and Sadigh, Dorsa},
  booktitle={Robotics: Science and Systems (RSS)},
  year={2020}
}

@book{boyd2004convex,
  title={Convex Optimization},
  author={Boyd, Stephen and Vandenberghe, Lieven},
  year={2004},
  publisher={Cambridge University Press},
}

@article{bradley1952rank,
 author = {Ralph Allan Bradley and Milton E. Terry},
 journal = {Biometrika},
 number = {3/4},
 pages = {324--345},
 publisher = {Oxford University Press, Biometrika Trust},
 title = {Rank Analysis of Incomplete Block Designs: {I}. {The} Method of Paired Comparisons},
 volume = {39},
 year = {1952}
}

@inproceedings{buening2025strategyproof,
  title={Strategyproof reinforcement learning from human feedback},
  author={Buening, Thomas Kleine and Gan, Jiarui and Mandal, Debmalya and Kwiatkowska, Marta},
  booktitle={Advances in Neural Information Processing Systems (NeurIPS)},
  volume={38},
  pages={112501--112534},
  year={2025}
}

@inproceedings{chakraborty2024maxmin,
  title={{MaxMin-RLHF}: Alignment with Diverse Human Preferences},
  author={Chakraborty, Souradip and Qiu, Jiahao and Yuan, Hui and Koppel, Alec and Manocha, Dinesh and Huang, Furong and Bedi, Amrit and Wang, Mengdi},
  booktitle={Proceedings of the International Conference on Machine Learning (ICML)},
  pages={6116--6135},
  year={2024},
  volume={235},
  series = 	 {Proceedings of Machine Learning Research},
  publisher =    {PMLR}
}

@inproceedings{chen2026reward,
  title={A Reward-Free Viewpoint on Multi-Objective Reinforcement Learning},
  author={Chen, Ying-Tu and Hung, Wei and Wu, Bing-Shu and Hong, Zhang-Wei and Hsieh, Ping-Chun},
  booktitle    = {International Conference on Learning Representations (ICLR)},
  year         = {2026}
}

@article{cherukuri2025learning,
  title={Learning Pareto-Optimal Rewards from Noisy Preferences: A Framework for Multi-Objective Inverse Reinforcement Learning},
  author={Cherukuri, Kalyan and Lala, Aarav},
  journal={arXiv preprint arXiv:2505.11864},
  year={2025}
}

@inproceedings{christiano2017deep,
 author = {Christiano, Paul F. and Leike, Jan and Brown, Tom and Martic, Miljan and Legg, Shane and Amodei, Dario},
 booktitle = {Advances in Neural Information Processing Systems (NIPS)},
 title = {Deep Reinforcement Learning from Human Preferences},
 volume = {30},
 year = {2017}
}

@inproceedings{dai2024safe,
  title={Safe {RLHF}: Safe Reinforcement Learning from Human Feedback},
  author={Dai, Josef and Pan, Xuehai and Sun, Ruiyang and Ji, Jiaming and Xu, Xinbo and Liu, Mickel and Wang, Yizhou and Yang, Yaodong},
  booktitle={International Conference on Learning Representations (ICLR)},
  year={2024}
}

@article{davidson1970extending,
  title={On extending the {Bradley-Terry} model to accommodate ties in paired comparison experiments},
  author={Davidson, Roger R},
  journal={Journal of the American Statistical Association},
  volume={65},
  number={329},
  pages={317--328},
  year={1970},
  publisher={Taylor \& Francis}
}

@inproceedings{dewey2014reinforcement,
  title={Reinforcement learning and the reward engineering principle},
  author={Dewey, Daniel},
  booktitle={AAAI Spring Symposium Series},
  year={2014}
}

@InProceedings{drago2025towards,
  title = 	 {Towards Theoretical Understanding of Sequential Decision Making with Preference Feedback},
  author =       {Drago, Simone and Mussi, Marco and Metelli, Alberto M.},
  booktitle = 	 {Proceedings of the International Conference on Machine Learning (ICML)},
  pages = 	 {14499--14514},
  year = 	 {2025},
  volume = 	 {267},
  series = 	 {Proceedings of Machine Learning Research},
  publisher =    {PMLR}
}

@article{dubra2004expected,
  title={Expected utility theory without the completeness axiom},
  author={Dubra, Juan and Maccheroni, Fabio and Ok, Efe A},
  journal={Journal of Economic Theory},
  volume={115},
  number={1},
  pages={118--133},
  year={2004},
  publisher={Elsevier}
}

@article{eliaz2006indifference,
  title={Indifference or indecisiveness? Choice-theoretic foundations of incomplete preferences},
  author={Eliaz, Kfir and Ok, Efe A},
  journal={Games and economic behavior},
  volume={56},
  number={1},
  pages={61--86},
  year={2006},
  publisher={Elsevier}
}

@article{felsner2017complexity,
  author       = {Stefan Felsner and
                  Irina Mustata and
                  Martin Pergel},
  title        = {The Complexity of the Partial Order Dimension Problem: Closing the Gap},
  journal      = {SIAM Journal on Discrete Mathematics (SIDMA)},
  volume       = {31},
  number       = {1},
  pages        = {172--189},
  year         = {2017}
}

@article{friedman1952expected,
  title={The expected-utility hypothesis and the measurability of utility},
  author={Friedman, Milton and Savage, Leonard J},
  journal={Journal of Political Economy},
  volume={60},
  number={6},
  pages={463--474},
  year={1952},
  publisher={The University of Chicago Press}
}

@article{furnkranz2012preference,
  title={Preference-based reinforcement learning: a formal framework and a policy iteration algorithm},
  author={F{\"u}rnkranz, Johannes and H{\"u}llermeier, Eyke and Cheng, Weiwei and Park, Sang-Hyeun},
  journal={Machine Learning},
  volume={89},
  number={1--2},
  pages={123--156},
  year={2012},
  publisher={Springer}
}

@inproceedings{ge2024axioms,
  title={Axioms for {AI} Alignment from Human Feedback},
  author={Ge, Luise and Halpern, Daniel and Micha, Evi and Procaccia, Ariel D. and Shapira, Itai and Vorobeychik, Yevgeniy and Wu, Junlin},
  booktitle={Advances in Neural Information Processing Systems (NeurIPS)},
  volume={37},
  pages={80439--80465},
  year={2024}
}

@article{hayes2022practical,
  title={A practical guide to multi-objective reinforcement learning and planning},
  author={Hayes, Conor F and R{\u{a}}dulescu, Roxana and Bargiacchi, Eugenio and K{\"a}llstr{\"o}m, Johan and Macfarlane, Matthew and Reymond, Mathieu and Verstraeten, Timothy and Zintgraf, Luisa M and Dazeley, Richard and Heintz, Fredrik and Howley, Enda and Irissappane, Athirai A and Mannion, Patrick and Now{\'e}, Ann and Ramos, Gabriel and Restelli, Marcello and Vamplew, Peter and Roijers, Diederik M},
  journal={Autonomous Agents and Multi-Agent Systems},
  volume={36},
  number={1},
  pages={26},
  year={2022},
  publisher={Springer US New York}
}

@article{hendrickx2024machine,
  title={Machine learning with a reject option: A survey},
  author={Hendrickx, Kilian and Perini, Lorenzo and Van der Plas, Dries and Meert, Wannes and Davis, Jesse},
  journal={Machine Learning},
  volume={113},
  number={5},
  pages={3073--3110},
  year={2024},
  publisher={Springer}
}

@InProceedings{hong2024energy,
  title = 	 {Energy-Based Preference Model Offers Better Offline Alignment than the {Bradley-Terry} Preference Model},
  author =       {Hong, Yuzhong and Zhang, Hanshan and Bao, Junwei and Jiang, Hongfei and Song, Yang},
  booktitle = 	 {Proceedings of the International Conference on Machine Learning (ICML)},
  pages = 	 {23787--23804},
  year = 	 {2025},
  volume = 	 {267},
  series = 	 {Proceedings of Machine Learning Research},
  publisher =    {PMLR}
}

@inproceedings{jeon2020reward,
  title={Reward-rational (implicit) choice: A unifying formalism for reward learning},
  author={Jeon, Hong Jun and Milli, Smitha and Dragan, Anca},
  booktitle={Advances in Neural Information Processing Systems (NeurIPS)},
  volume={33},
  pages={4415--4426},
  year={2020}
}

@inproceedings{jin2020reward,
  title={Reward-free exploration for reinforcement learning},
  author={Jin, Chi and Krishnamurthy, Akshay and Simchowitz, Max and Yu, Tiancheng},
  booktitle={Proceedings of the International Conference on Machine Learning (ICML)},
  pages={4870--4879},
  year={2020},
  volume={119},
  series = 	 {Proceedings of Machine Learning Research},
  organization={PMLR}
}

@article{jin2023provably,
  title={Provably efficient reinforcement learning with linear function approximation},
  author={Jin, Chi and Yang, Zhuoran and Wang, Zhaoran and Jordan, Michael I.},
  journal={Mathematics of Operations Research},
  volume={48},
  number={3},
  pages={1496--1521},
  year={2023},
  publisher={INFORMS}
}

@article{kaufmann2025a,
    title={A Survey of Reinforcement Learning from Human Feedback},
    author={Timo Kaufmann and Paul Weng and Viktor Bengs and Eyke H{\"u}llermeier},
    journal={Transactions on Machine Learning Research},
    year={2025}
}

@inproceedings{kingma2015adam,
  author       = {Diederik P. Kingma and
                  Jimmy Ba},
  title        = {Adam: {A} Method for Stochastic Optimization},
  booktitle    = {International Conference on Learning Representations (ICLR)},
  year         = {2015},
}

@book{lattimore2020bandit,
  title={Bandit Algorithms},
  author={Lattimore, Tor and Szepesv{\'a}ri, Csaba},
  year={2020},
  publisher={Cambridge University Press}
}

@incollection{lecun2006tutorial,
  title={A tutorial on energy-based learning},
  author={LeCun, Yann and Chopra, Sumit and Hadsell, Raia and Ranzato, {Marc'Aurelio} and Huang, Fu Jie},
  booktitle={Predicting Structured Data},
  editor={Bak{\i}r, G{\"o}khan and Hofmann, Thomas and Sch{\"o}lkopf, Bernhard and Smola, Alexander J. and Taskar, Ben and Vishwanathan, S. V. N.},
  publisher={MIT Press},
  year={2006}
}

@InProceedings{lee2021pebble,
  title = {PEBBLE: Feedback-Efficient Interactive Reinforcement Learning via Relabeling Experience and Unsupervised Pre-training},
  author = {Lee, Kimin and Smith, Laura M. and Abbeel, Pieter},
  booktitle = {Proceedings of the International Conference on Machine Learning (ICML)},
  pages = {6152--6163},
  year = {2021},
  volume = 	 {139},
  series = 	 {Proceedings of Machine Learning Research},
  publisher =    {PMLR}
}

@book{luce1959individual,
  title={Individual Choice Behavior},
  author={Luce, R Duncan},
  year={1959},
  publisher={Wiley}
}

@article{mu2025preference,
  title={Preference-based multi-objective reinforcement learning},
  author={Mu, Ni and Luan, Yao and Jia, Qing-Shan},
  journal={IEEE Transactions on Automation Science and Engineering},
  volume={22},
  pages={18737--18749},
  year={2025},
  publisher={IEEE}
}

@inproceedings{munos2024nash,
  title={Nash learning from human feedback},
  author={Munos, R{\'e}mi and Valko, Michal and Calandriello, Daniele and Azar, Mohammad Gheshlaghi and Rowland, Mark and Guo, Zhaohan Daniel and Tang, Yunhao and Geist, Matthieu and Mesnard, Thomas and Fiegel, C{\^o}me and Michi, Andrea and Selvi, Marco and Girgin, Sertan and Momchev, Nikola and Bachem, Olivier and Mankowitz, Daniel J and Precup, Doina and Piot, Bilal},
  booktitle={Proceedings of the International Conference on Machine Learning (ICML)},
  pages={36743--36768},
  year={2024},
  volume={235},
  series={Proceedings of Machine Learning Research},
  publisher={PMLR}
}

@inproceedings{myers2021learning,
  title={Learning multimodal rewards from rankings},
  author={Myers, Vivek and B{\i}y{\i}k, Erdem and Anari, Nima and Sadigh, Dorsa},
  booktitle={Proceedings of the Conference on Robot Learning (CoRL)},
  pages={342--352},
  year={2021},
  volume={164},
  series={Proceedings of Machine Learning Research},
  publisher={PMLR}
}

@inproceedings{ng2000algorithms,
  title={Algorithms for Inverse Reinforcement Learning},
  author={Ng, Andrew Y. and Russell, Stuart},
  booktitle={Proceedings of the International Conference on Machine Learning (ICML)},
  pages        = {663--670},
  publisher    = {Morgan Kaufmann},
  year={2000}
}

@book{nocedal2006numerical,
  title={Numerical Optimization},
  series={Springer Series in Operations Research and Financial Engineering},
  author={Nocedal, Jorge and Wright, Stephen},
  publisher={Springer},
  year={2006}
}

@article{ok2002utility,
  title={Utility representation of an incomplete preference relation},
  author={Ok, Efe A},
  journal={Journal of Economic Theory},
  volume={104},
  number={2},
  pages={429--449},
  year={2002},
  publisher={New York: Academic Press.}
}

@inproceedings{ouyang2022training,
 author = {Ouyang, Long and Wu, Jeffrey and Jiang, Xu and Almeida, Diogo and Wainwright, Carroll and Mishkin, Pamela and Zhang, Chong and Agarwal, Sandhini and Slama, Katarina and Ray, Alex and Schulman, John and Hilton, Jacob and Kelton, Fraser and Miller, Luke and Simens, Maddie and Askell, Amanda and Welinder, Peter and Christiano, Paul F. and Leike, Jan and Lowe, Ryan},
 booktitle = {Advances in Neural Information Processing Systems (NeurIPS)},
 pages = {27730--27744},
 title = {Training language models to follow instructions with human feedback},
 volume = {35},
 year = {2022}
}

@inproceedings{pan2022effects,
  title={The Effects of Reward Misspecification: Mapping and Mitigating Misaligned Models},
  author={Pan, Alexander and Bhatia, Kush and Steinhardt, Jacob},
  booktitle={International Conference on Learning Representations (ICLR)},
  year={2022}
}

@inproceedings{booth2023perils,
  title={The Perils of Trial-and-Error Reward Design: Misdesign through Overfitting and Invalid Task Specifications},
  author={Booth, Serena and Knox, W. Bradley and Shah, Julie and Niekum, Scott and Stone, Peter and Allievi, Alessandro},
  booktitle={Proceedings of the AAAI Conference on Artificial Intelligence},
  volume={37},
  pages={5920--5929},
  year={2023}
}

@inproceedings{paszke2019pytorch,
 author = {Paszke, Adam and Gross, Sam and Massa, Francisco and Lerer, Adam and Bradbury, James and Chanan, Gregory and Killeen, Trevor and Lin, Zeming and Gimelshein, Natalia and Antiga, Luca and Desmaison, Alban and Kopf, Andreas and Yang, Edward and DeVito, Zachary and Raison, Martin and Tejani, Alykhan and Chilamkurthy, Sasank and Steiner, Benoit and Fang, Lu and Bai, Junjie and Chintala, Soumith},
 booktitle = {Advances in Neural Information Processing Systems (NeurIPS)},
 title = {PyTorch: An Imperative Style, High-Performance Deep Learning Library},
 volume = {32},
 year = {2019}
}

@article{plackett1975analysis,
  title={The analysis of permutations},
  author={Plackett, Robin L.},
  journal={Journal of the Royal Statistical Society Series C: Applied Statistics},
  volume={24},
  number={2},
  pages={193--202},
  year={1975},
  publisher={Oxford University Press}
}

@article{qiu2024traversing,
  title={Traversing Pareto optimal policies: Provably efficient multi-objective reinforcement learning},
  author={Qiu, Shuang and Zhang, Dake and Yang, Rui and Lyu, Boxiang and Zhang, Tong},
  journal={arXiv preprint arXiv:2407.17466},
  year={2024}
}

@inproceedings{rafailov2024direct,
 author = {Rafailov, Rafael and Sharma, Archit and Mitchell, Eric and Manning, Christopher D. and Ermon, Stefano and Finn, Chelsea},
 booktitle = {Advances in Neural Information Processing Systems (NeurIPS)},
 pages = {53728--53741},
 title = {Direct Preference Optimization: Your Language Model is Secretly a Reward Model},
 volume = {36},
 year = {2023}
}

@article{rao1967ties,
  title={Ties in paired-comparison experiments: A generalization of the {Bradley-Terry} model},
  author={Rao, Pejaver V. and Kupper, Lawrence L.},
  journal={Journal of the American Statistical Association},
  volume={62},
  number={317},
  pages={194--204},
  year={1967},
  publisher={Taylor \& Francis}
}

@article{roijers2013survey,
  title={A survey of multi-objective sequential decision-making},
  author={Roijers, Diederik M. and Vamplew, Peter and Whiteson, Shimon and Dazeley, Richard},
  journal={Journal of Artificial Intelligence Research},
  volume={48},
  pages={67--113},
  year={2013}
}

@article{roy1984relational,
  title={Relational systems of preference with one or more pseudo-criteria: Some new concepts and results},
  author={Roy, Bernard and Vincke, Ph},
  journal={Management Science},
  volume={30},
  number={11},
  pages={1323--1335},
  year={1984},
  publisher={INFORMS}
}

@inproceedings{sadigh2017active,
  title={Active preference-based learning of reward functions},
  author={Sadigh, Dorsa and Dragan, Anca and Sastry, Shankar and Seshia, Sanjit},
  year={2017},
  booktitle={Robotics: Science and Systems}
}

@book{sen1970collective,
  title={Collective Choice and Social Welfare},
  author={Sen, Amartya K},
  publisher={Holden-Day},
  address={San Francisco, CA, USA},
  year={1970}
}

@InProceedings{silver2014deterministic,
  title = 	 {Deterministic Policy Gradient Algorithms},
  author = 	 {Silver, David and Lever, Guy and Heess, Nicolas and Degris, Thomas and Wierstra, Daan and Riedmiller, Martin},
  booktitle = 	 {Proceedings of the International Conference on Machine Learning (ICML)},
  pages = 	 {387--395},
  year = 	 {2014},
  volume = 	 {32},
  series = 	 {Proceedings of Machine Learning Research},
  publisher =    {PMLR}
}

@inproceedings{siththaranjan2024distributional,
title={Distributional Preference Learning: Understanding and Accounting for Hidden Context in {RLHF}},
author={Anand Siththaranjan and Cassidy Laidlaw and Dylan Hadfield-Menell},
booktitle={International Conference on Learning Representations (ICLR)},
year={2024}
}

@InProceedings{sorensen2024position,
  title = 	 {Position: A Roadmap to Pluralistic Alignment},
  author =       {Sorensen, Taylor and Moore, Jared and Fisher, Jillian and Gordon, Mitchell L. and Mireshghallah, Niloofar and Rytting, Christopher Michael and Ye, Andre and Jiang, Liwei and Lu, Ximing and Dziri, Nouha and Althoff, Tim and Choi, Yejin},
  booktitle = 	 {Proceedings of the International Conference on Machine Learning (ICML)},
  pages = 	 {46280--46302},
  year = 	 {2024},
  volume = 	 {235},
  series = 	 {Proceedings of Machine Learning Research},
  publisher =    {PMLR}
}

@misc{sutton2004reward,
    author = {Richard S. Sutton},
    title = {The reward hypothesis},
    howpublished = {\url{http://incompleteideas.net/rlai.cs.ualberta.ca/RLAI/rewardhypothesis.html}},
    year = {2004}
}

@book{sutton2018reinforcement,
  title={Reinforcement learning: An introduction},
  author={Sutton, Richard S. and Barto, Andrew G.},
  year={2018},
  publisher={MIT Press}
}

@article{thurstone1927law,
	author = {L. L. Thurstone},
	journal = {Psychological Review},
	number = {4},
	pages = {273--286},
	title = {A Law of Comparative Judgment},
	volume = {34},
	year = {1927}
}

@inproceedings{towers2025gymnasium,
 author = {Towers, Mark and Kwiatkowski, Ariel and Balis, John and De Cola, Gianluca and Deleu, Tristan and Goul\~{a}o, Manuel and Kallinteris, Andreas and Krimmel, Markus and KG, Arjun and Perez-Vicente, Rodrigo and Terry, Jordan and Pierr\'{e}, Andrea and Schulhoff, Sander and Tai, Jun J. and Tan, Hannah and Younis, Omar G.},
 booktitle = {Advances in Neural Information Processing Systems (NeurIPS)},
 title = {Gymnasium: A Standard Interface for Reinforcement Learning Environments},
 volume = {38},
 pages = {163114--163129},
 year = {2025}
}

@article{tsoukias1995new,
  title={A new axiomatic foundation of partial comparability},
  author={Tsoukias, Alexis and Vincke, Philippe},
  journal={Theory and Decision},
  volume={39},
  number={1},
  pages={79--114},
  year={1995},
  publisher={Springer}
}

@book{vonneumann1947theory,
  author = {von Neumann, John and Morgenstern, Oskar},
  publisher = {Princeton University Press},
  title = {Theory of games and economic behavior},
  year = {1947}
}

@inproceedings{wang2020reward,
 author = {Wang, Ruosong and Du, Simon and Yang, Lin and Salakhutdinov, Russ R.},
 booktitle = {Advances in Neural Information Processing Systems (NeurIPS)},
 pages = {17816--17826},
 title = {On Reward-Free Reinforcement Learning with Linear Function Approximation},
 volume = {33},
 year = {2020}
}

@inproceedings{wang2024helpsteer2,
title={HelpSteer2-Preference: Complementing Ratings with Preferences},
author={Zhilin Wang and Alexander Bukharin and Olivier Delalleau and Daniel Egert and Gerald Shen and Jiaqi Zeng and Oleksii Kuchaiev and Yi Dong},
booktitle={International Conference on Learning Representations (ICLR)},
year={2025}
}

@article{wirth2017survey,
  title={A survey of preference-based reinforcement learning methods},
  author={Wirth, Christian and Akrour, Riad and Neumann, Gerhard and F{\"u}rnkranz, Johannes},
  journal={Journal of Machine Learning Research},
  volume={18},
  number={136},
  pages={1--46},
  year={2017}
}

@inproceedings{zhang2021reward,
 author = {Zhang, Weitong and Zhou, Dongruo and Gu, Quanquan},
 booktitle = {Advances in Neural Information Processing Systems (NeurIPS)},
 pages = {1582--1593},
 title = {Reward-Free Model-Based Reinforcement Learning with Linear Function Approximation},
 volume = {34},
 year = {2021}
}

@inproceedings{zhang2024beyond,
  title = 	 {Beyond Bradley-Terry Models: A General Preference Model for Language Model Alignment},
  author =       {Zhang, Yifan and Zhang, Ge and Wu, Yue and Xu, Kangping and Gu, Quanquan},
  booktitle = 	 {Proceedings of the International Conference on Machine Learning (ICML)},
  pages = 	 {76939--76965},
  year = 	 {2025},
  volume = 	 {267},
  series = 	 {Proceedings of Machine Learning Research},
  publisher =    {PMLR}
}

@article{zhao2023survey,
  title   = {A Survey of Large Language Models},
  author  = {Wayne Xin Zhao and
                  Kun Zhou and
                  Junyi Li and
                  Tianyi Tang and
                  Zican Dong and
                  Yupeng Hou and
                  Beichen Zhang and
                  Yingqian Min and
                  Junjie Zhang and
                  Peiyu Liu and
                  Xiaolei Wang and
                  Yifan Du and
                  Chen Yang and
                  Yushuo Chen and
                  Zhipeng Chen and
                  Jinhao Jiang and
                  Ruiyang Ren and
                  Yifan Li and
                  Xinyu Tang and
                  Zikang Liu and
                  Yiwen Hu and
                  Jian{-}Yun Nie and
                  Ji{-}Rong Wen},
  journal = {Frontiers of Computer Science},
  volume  = {20},
  number  = {12},
  pages   = {2012627},
  year    = {2026},
  publisher = {Springer}
}

@inproceedings{zhao2024rapbrl,
 author = {Zhao, Yujie and Aguilar Escamilla, Jose Efraim and Lu, Weyl and Wang, Huazheng},
 booktitle = {Advances in Neural Information Processing Systems (NeurIPS)},
 pages = {60835--60871},
 title = {{RA-PbRL}: Provably Efficient Risk-Aware Preference-Based Reinforcement Learning},
 volume = {37},
 year = {2024}
}

@inproceedings{zhu2023principled,
  title={Principled Reinforcement Learning with Human Feedback from Pairwise or {K}-wise Comparisons},
  author={Zhu, Banghua and Jordan, Michael I. and Jiao, Jiantao},
  booktitle={Proceedings of the International Conference on Machine Learning (ICML)},
  pages={43037--43067},
  year={2023},
  volume={202},
  series={Proceedings of Machine Learning Research},
  publisher={PMLR}
}

@article{cover1965geometrical,
  title={Geometrical and Statistical Properties of Systems of Linear Inequalities with Applications in Pattern Recognition},
  author={Cover, Thomas M.},
  journal={IEEE Transactions on Electronic Computers},
  volume={EC-14},
  number={3},
  pages={326--334},
  year={1965},
  publisher={IEEE}
}

@article{candes2013phaselift,
  title={{PhaseLift}: Exact and Stable Signal Recovery from Magnitude Measurements via Convex Programming},
  author={Cand{\`e}s, Emmanuel J. and Strohmer, Thomas and Voroninski, Vladislav},
  journal={Communications on Pure and Applied Mathematics},
  volume={66},
  number={8},
  pages={1241--1274},
  year={2013},
  publisher={Wiley}
}

@inproceedings{faury2020improved,
  title={Improved Optimistic Algorithms for Logistic Bandits},
  author={Faury, Louis and Abeille, Marc and Calauz{\`e}nes, Cl{\'e}ment and Fercoq, Olivier},
  booktitle={Proceedings of the International Conference on Machine Learning (ICML)},
  pages={3052--3060},
  year={2020},
  volume={119},
  series={Proceedings of Machine Learning Research},
  publisher={PMLR}
}

@inproceedings{zhan2024provable,
  title={Provable Offline Preference-Based Reinforcement Learning},
  author={Zhan, Wenhao and Uehara, Masatoshi and Kallus, Nathan and Lee, Jason D. and Sun, Wen},
  booktitle={International Conference on Learning Representations (ICLR)},
  year={2024}
}

@inproceedings{tu2016low,
  title={Low-rank Solutions of Linear Matrix Equations via {Procrustes} Flow},
  author={Tu, Stephen and Boczar, Ross and Simchowitz, Max and Soltanolkotabi, Mahdi and Recht, Benjamin},
  booktitle={Proceedings of the International Conference on Machine Learning (ICML)},
  pages={964--973},
  year={2016},
  volume={48},
  series={Proceedings of Machine Learning Research},
  publisher={PMLR}
}

@article{lethi2018dc,
  title={{DC} programming and {DCA}: thirty years of developments},
  author={Le Thi, Hoai An and Pham Dinh, Tao},
  journal={Mathematical Programming},
  volume={169},
  number={1},
  pages={5--68},
  year={2018},
  publisher={Springer}
}

@article{banerjee2005clustering,
  title={Clustering with {Bregman} Divergences},
  author={Banerjee, Arindam and Merugu, Srujana and Dhillon, Inderjit S. and Ghosh, Joydeep},
  journal={Journal of Machine Learning Research},
  volume={6},
  pages={1705--1749},
  year={2005}
}

@inproceedings{papadimitriou2000approximability,
  title={On the approximability of trade-offs and optimal access of {Web} sources},
  author={Papadimitriou, Christos H. and Yannakakis, Mihalis},
  booktitle={Proceedings of the Annual Symposium on Foundations of Computer Science (FOCS)},
  pages={86--92},
  year={2000},
  publisher={IEEE}
}

@article{hoeffding1963probability,
  title={Probability Inequalities for Sums of Bounded Random Variables},
  author={Hoeffding, Wassily},
  journal={Journal of the American Statistical Association},
  volume={58},
  number={301},
  pages={13--30},
  year={1963},
  publisher={Taylor \& Francis}
}

@book{miettinen1999nonlinear,
  title={Nonlinear Multiobjective Optimization},
  author={Miettinen, Kaisa},
  year={1999},
  publisher={Kluwer Academic Publishers}
}

@book{cover2006elements,
  title={Elements of Information Theory},
  author={Cover, Thomas M. and Thomas, Joy A.},
  edition={2nd},
  year={2006},
  publisher={Wiley-Interscience}
}

@book{rockafellar1970convex,
  title={Convex Analysis},
  author={Rockafellar, R. Tyrrell},
  series={Princeton Mathematical Series},
  volume={28},
  year={1970},
  publisher={Princeton University Press}
}

@inproceedings{akrour2011preference,
  title={Preference-Based Policy Learning},
  author={Akrour, Riad and Schoenauer, Marc and Sebag, Mich{\`e}le},
  booktitle={Machine Learning and Knowledge Discovery in Databases (ECML PKDD)},
  pages={12--27},
  year={2011},
  publisher={Springer}
}

@inproceedings{cheng2011preference,
  title={Preference-Based Policy Iteration: Leveraging Preference Learning for Reinforcement Learning},
  author={Cheng, Weiwei and F{\"u}rnkranz, Johannes and H{\"u}llermeier, Eyke and Park, Sang-Hyeun},
  booktitle={Machine Learning and Knowledge Discovery in Databases (ECML PKDD)},
  pages={312--327},
  year={2011},
  publisher={Springer}
}

@article{mosteller1951remarks,
  title={Remarks on the Method of Paired Comparisons: {I}. {The} Least Squares Solution Assuming Equal Standard Deviations and Equal Correlations},
  author={Mosteller, Frederick},
  journal={Psychometrika},
  volume={16},
  number={1},
  pages={3--9},
  year={1951}
}

@inproceedings{ziebart2008maximum,
  title={Maximum Entropy Inverse Reinforcement Learning},
  author={Ziebart, Brian D. and Maas, Andrew L. and Bagnell, J. Andrew and Dey, Anind K.},
  booktitle={Proceedings of the AAAI Conference on Artificial Intelligence},
  pages={1433--1438},
  year={2008}
}

@article{knox2024models,
  title={Models of Human Preference for Learning Reward Functions},
  author={Knox, W. Bradley and Hatgis-Kessell, Stephane and Booth, Serena and Niekum, Scott and Stone, Peter and Allievi, Alessandro},
  journal={Transactions on Machine Learning Research},
  year={2024}
}

@article{causeur2005two,
  title={A 2-Dimensional Extension of the {Bradley--Terry} Model for Paired Comparisons},
  author={Causeur, David and Husson, Fran{\c{c}}ois},
  journal={Journal of Statistical Planning and Inference},
  volume={135},
  number={2},
  pages={245--259},
  year={2005}
}

@inproceedings{chen2016modeling,
  title={Modeling Intransitivity in Matchup and Comparison Data},
  author={Chen, Shuo and Joachims, Thorsten},
  booktitle={Proceedings of the ACM International Conference on Web Search and Data Mining (WSDM)},
  pages={227--236},
  year={2016}
}

@inproceedings{chen2025extending,
  title={On Extending Direct Preference Optimization to Accommodate Ties},
  author={Chen, Jinghong and Yang, Guangyu and Lin, Weizhe and Mei, Jingbiao and Lyu, Chenxu and Byrne, Bill},
  booktitle={Advances in Neural Information Processing Systems (NeurIPS)},
  volume={38},
  pages={60076--60121},
  year={2025}
}

@article{ok2022indifference,
  title={Indifference, Indecisiveness, Experimentation, and Stochastic Choice},
  author={Ok, Efe A. and Tserenjigmid, Gerelt},
  journal={Theoretical Economics},
  volume={17},
  number={2},
  pages={651--686},
  year={2022}
}

@book{roy1996multicriteria,
  title={Multicriteria Methodology for Decision Aiding},
  author={Roy, Bernard},
  series={Nonconvex Optimization and Its Applications},
  volume={12},
  publisher={Kluwer Academic Publishers},
  year={1996}
}

@article{sen1997maximization,
  title={Maximization and the Act of Choice},
  author={Sen, Amartya},
  journal={Econometrica},
  volume={65},
  number={4},
  pages={745--779},
  year={1997}
}

@article{tversky1992choice,
  title={Choice under Conflict: The Dynamics of Deferred Decision},
  author={Tversky, Amos and Shafir, Eldar},
  journal={Psychological Science},
  volume={3},
  number={6},
  pages={358--361},
  year={1992}
}

@article{dhar1997consumer,
  title={Consumer Preference for a No-Choice Option},
  author={Dhar, Ravi},
  journal={Journal of Consumer Research},
  volume={24},
  number={2},
  pages={215--231},
  year={1997}
}

@article{costagomes2022choice,
  title={Choice, Deferral, and Consistency},
  author={Costa-Gomes, Miguel A. and Cueva, Carlos and Gerasimou, Georgios and Teji{\v{s}}{\v{c}}{\'a}k, Mat{\'u}{\v{s}}},
  journal={Quantitative Economics},
  volume={13},
  number={3},
  pages={1297--1318},
  year={2022}
}

@inproceedings{skalse2023limitations,
  title={On the Limitations of {Markovian} Rewards to Express Multi-Objective, Risk-Sensitive, and Modal Tasks},
  author={Skalse, Joar and Abate, Alessandro},
  booktitle={Proceedings of the Conference on Uncertainty in Artificial Intelligence (UAI)},
  pages={1974--1984},
  year={2023},
  volume={216},
  series={Proceedings of Machine Learning Research},
  publisher={PMLR}
}

@inproceedings{bowling2023settling,
  title={Settling the Reward Hypothesis},
  author={Bowling, Michael and Martin, John D. and Abel, David and Dabney, Will},
  booktitle={Proceedings of the International Conference on Machine Learning (ICML)},
  pages={3003--3020},
  year={2023},
  volume={202},
  series={Proceedings of Machine Learning Research},
  publisher={PMLR}
}

@inproceedings{yang2019generalized,
  title={A Generalized Algorithm for Multi-Objective Reinforcement Learning and Policy Adaptation},
  author={Yang, Runzhe and Sun, Xingyuan and Narasimhan, Karthik},
  booktitle={Advances in Neural Information Processing Systems (NeurIPS)},
  volume={32},
  year={2019}
}

@inproceedings{lu2023multi,
  title={Multi-Objective Reinforcement Learning: Convexity, Stationarity and {Pareto} Optimality},
  author={Lu, Haoye and Herman, Daniel and Yu, Yaoliang},
  booktitle={International Conference on Learning Representations (ICLR)},
  year={2023}
}

@inproceedings{wu2021accommodating,
  title={Accommodating Picky Customers: Regret Bound and Exploration Complexity for Multi-Objective Reinforcement Learning},
  author={Wu, Jingfeng and Braverman, Vladimir and Yang, Lin},
  booktitle={Advances in Neural Information Processing Systems (NeurIPS)},
  volume={34},
  year={2021}
}

@inproceedings{zintgraf2018ordered,
  title={Ordered Preference Elicitation Strategies for Supporting Multi-Objective Decision Making},
  author={Zintgraf, Luisa M. and Roijers, Diederik M. and Linders, Sjoerd and Jonker, Catholijn M. and Now{\'e}, Ann},
  booktitle={Proceedings of the International Conference on Autonomous Agents and Multiagent Systems (AAMAS)},
  pages={1477--1485},
  year={2018}
}

@inproceedings{lin2022preference,
  title={Preference Exploration for Efficient {Bayesian} Optimization with Multiple Outcomes},
  author={Lin, Zhiyuan Jerry and Astudillo, Raul and Frazier, Peter I. and Bakshy, Eytan},
  booktitle={Proceedings of the International Conference on Artificial Intelligence and Statistics (AISTATS)},
  pages={4235--4258},
  year={2022},
  volume={151},
  series={Proceedings of Machine Learning Research},
  publisher={PMLR}
}

@inproceedings{shao2023eliciting,
  title={Eliciting User Preferences for Personalized Multi-Objective Decision Making through Comparative Feedback},
  author={Shao, Han and Cohen, Lee and Blum, Avrim and Mansour, Yishay and Saha, Aadirupa and Walter, Matthew},
  booktitle={Advances in Neural Information Processing Systems (NeurIPS)},
  volume={36},
  pages={12192--12221},
  year={2023}
}

@inproceedings{wu2023fine,
  title={Fine-Grained Human Feedback Gives Better Rewards for Language Model Training},
  author={Wu, Zeqiu and Hu, Yushi and Shi, Weijia and Dziri, Nouha and Suhr, Alane and Ammanabrolu, Prithviraj and Smith, Noah A. and Ostendorf, Mari and Hajishirzi, Hannaneh},
  booktitle={Advances in Neural Information Processing Systems (NeurIPS)},
  volume={36},
  pages={59008--59033},
  year={2023}
}

@article{boldi2025pareto,
  title={{Pareto} Optimal Learning from Preferences with Hidden Context},
  author={Bahlous-Boldi, Ryan and Ding, Li and Spector, Lee and Niekum, Scott},
  journal={Reinforcement Learning Journal},
  volume={6},
  pages={650--670},
  year={2025}
}

@inproceedings{saha2023dueling,
  title={Dueling {RL}: Reinforcement Learning with Trajectory Preferences},
  author={Saha, Aadirupa and Pacchiano, Aldo and Lee, Jonathan},
  booktitle={Proceedings of the International Conference on Artificial Intelligence and Statistics (AISTATS)},
  pages={6263--6289},
  year={2023},
  volume={206},
  series={Proceedings of Machine Learning Research},
  publisher={PMLR}
}

@inproceedings{wang2023rlhf,
  title={Is {RLHF} More Difficult than Standard {RL}? {A} Theoretical Perspective},
  author={Wang, Yuanhao and Liu, Qinghua and Jin, Chi},
  booktitle={Advances in Neural Information Processing Systems (NeurIPS)},
  volume={36},
  pages={76006--76032},
  year={2023}
}

@inproceedings{kim2021reward,
  title={Reward Identification in Inverse Reinforcement Learning},
  author={Kim, Kuno and Garg, Shivam and Shiragur, Kirankumar and Ermon, Stefano},
  booktitle={Proceedings of the International Conference on Machine Learning (ICML)},
  pages={5496--5505},
  year={2021},
  volume={139},
  series={Proceedings of Machine Learning Research},
  publisher={PMLR}
}

@inproceedings{cao2021identifiability,
  title={Identifiability in Inverse Reinforcement Learning},
  author={Cao, Haoyang and Cohen, Samuel and Szpruch, Lukasz},
  booktitle={Advances in Neural Information Processing Systems (NeurIPS)},
  volume={34},
  year={2021}
}

@inproceedings{skalse2023invariance,
  title={Invariance in Policy Optimisation and Partial Identifiability in Reward Learning},
  author={Skalse, Joar and Farrugia-Roberts, Matthew and Russell, Stuart and Abate, Alessandro and Gleave, Adam},
  booktitle={Proceedings of the International Conference on Machine Learning (ICML)},
  pages={32033--32058},
  year={2023},
  volume={202},
  series={Proceedings of Machine Learning Research},
  publisher={PMLR}
}

@article{dushnik1941partially,
  title={Partially Ordered Sets},
  author={Dushnik, Ben and Miller, E. W.},
  journal={American Journal of Mathematics},
  volume={63},
  number={3},
  pages={600--610},
  year={1941}
}

@article{yannakakis1982complexity,
  title={The Complexity of the Partial Order Dimension Problem},
  author={Yannakakis, Mihalis},
  journal={SIAM Journal on Algebraic Discrete Methods},
  volume={3},
  number={3},
  pages={351--358},
  year={1982}
}

@article{harris2020array,
  title={Array programming with {NumPy}},
  author={Harris, Charles R. and Millman, K. Jarrod and van der Walt, St{\'e}fan J. and Gommers, Ralf and Virtanen, Pauli and Cournapeau, David and Wieser, Eric and Taylor, Julian and Berg, Sebastian and Smith, Nathaniel J. and Kern, Robert and Picus, Matti and Hoyer, Stephan and van Kerkwijk, Marten H. and Brett, Matthew and Haldane, Allan and del R{\'i}o, Jaime Fern{\'a}ndez and Wiebe, Mark and Peterson, Pearu and G{\'e}rard-Marchant, Pierre and Sheppard, Kevin and Reddy, Tyler and Weckesser, Warren and Abbasi, Hameer and Gohlke, Christoph and Oliphant, Travis E.},
  journal={Nature},
  volume={585},
  number={7825},
  pages={357--362},
  year={2020}
}

@article{virtanen2020scipy,
  title={{SciPy} 1.0: Fundamental Algorithms for Scientific Computing in {Python}},
  author={Virtanen, Pauli and Gommers, Ralf and Oliphant, Travis E. and Haberland, Matt and Reddy, Tyler and Cournapeau, David and Burovski, Evgeni and Peterson, Pearu and Weckesser, Warren and Bright, Jonathan and van der Walt, St{\'e}fan J. and Brett, Matthew and Wilson, Joshua and Millman, K. Jarrod and Mayorov, Nikolay and Nelson, Andrew R. J. and Jones, Eric and Kern, Robert and Larson, Eric and Carey, C. J. and Polat, {\.I}lhan and Feng, Yu and Moore, Eric W. and VanderPlas, Jake and Laxalde, Denis and Perktold, Josef and Cimrman, Robert and Henriksen, Ian and Quintero, E. A. and Harris, Charles R. and Archibald, Anne M. and Ribeiro, Ant{\^o}nio H. and Pedregosa, Fabian and van Mulbregt, Paul and {SciPy 1.0 Contributors}},
  journal={Nature Methods},
  volume={17},
  number={3},
  pages={261--272},
  year={2020}
}

@article{kuhn1955hungarian,
  title={The {Hungarian} method for the assignment problem},
  author={Kuhn, Harold W.},
  journal={Naval Research Logistics Quarterly},
  volume={2},
  number={1--2},
  pages={83--97},
  year={1955}
}
\bibliographystyle{plainnat}

\clearpage
\appendix
\thispagestyle{empty}
\onecolumn

\section{Extended Related Works}
\label{apx:related}

In this appendix, we provide an extended discussion of the literature related to our work, which lies at the intersection of PbRL, MORL, and decision theory. First, we review PbRL and the rationality models it relies on, and discuss how non-preference feedback has been treated in decision theory and machine learning, together with the behavioral evidence supporting incomparability as a rational feedback. Then, we discuss MORL and the existing approaches that combine multi-objective problems with human feedback, highlighting what each of them requires of the expert. Finally, we discuss the limitations of scalar reward learning from the perspective of social choice theory, the statistical theory of PbRL, and the literature on reward identifiability.

\textbf{Preference-Based Reinforcement Learning.}~~Preference-based RL was formalized to integrate the \emph{learning from preferences} and RL fields \citep{akrour2011preference, cheng2011preference, furnkranz2012preference}, and has gained renewed attention with the rise of LLMs \citep{kaufmann2025a, zhao2023survey}.
PbRL is based on the mathematical framework of the \emph{Markov decision process with preferences} \citep[MDPP,][]{wirth2017survey}, combining an \mdpr \citep{abbeel2004apprenticeship} with a probability distribution $\rho$ to model the probability of observing a clear preference given two trajectories.
In this work, we take inspiration from the MDPP to formalize its multi-objective counterpart, the \emph{Markov decision process with comparisons} (MDPC). Indeed, our framework combines the \mdpr with a probability distribution $\rho$ defined over a set of \emph{four} possible comparison feedback classes, i.e., the direct and inverse preferences of the MDPP plus \emph{indifference} and \emph{incomparability}.
Following the taxonomy of \citet{wirth2017survey}, PbRL approaches either learn a policy directly, learn a preference model, or learn a utility function.
In this work, we focus on the latter approach, also known as \emph{reward learning} \citep{christiano2017deep}, which comprises two steps: ($i$) inferring a reward model from observed preferences and ($ii$) learning an optimal policy via standard RL methods.
We propose a rationality model, namely MOBT, to infer a multi-dimensional reward from comparisons, and analyze its maximum likelihood estimator.

\textbf{Alternative Rationality Models.}~~The most widely considered rationality model in PbRL is the Bradley-Terry model \citep[BT,][]{bradley1952rank}, which models the probability of observing a preference as the sigmoid of the difference in utility of the two trajectories.
The \emph{Plackett-Luce} model (\citealt{plackett1975analysis, luce1959individual}) generalizes BT to model \emph{rankings} of objects.
A classical alternative to BT is the \emph{Thurstone-Mosteller} model \citep{thurstone1927law, mosteller1951remarks}, which considers an additive Gaussian noise in the preference generation process.
The determinism of the expert's decisions is commonly modeled via a \emph{rationality} coefficient that rescales the utilities, as in Boltzmann rationality \citep{luce1959individual, ziebart2008maximum}, and \citet{jeon2020reward} showed that Boltzmann-rational choice unifies many different types of feedback, from comparisons to demonstrations, under a single formalism, in which incomparability can be interpreted as a novel type of feedback.
In recent years, novel rationality models have been proposed to tackle specific limitations of BT.
For instance, the \emph{infinite preference model} \citep{hong2024energy} tackles the problem of BT not guaranteeing a unique optimum to the MLE problem in \emph{direct preference optimization} \citep{rafailov2024direct}, proposing an approach based on \emph{energy-based models} \citep{lecun2006tutorial}.
Finally, the \emph{general preference embedding model} \citep{zhang2024beyond} addresses intransitive and/or cyclic preferences, embedding each response in a multi-dimensional space to capture complex structures that evade the limitations of BT's scalar utility, and is used as the basis of the \emph{general preference optimization} method.
Multi-dimensional utilities have also been considered in the statistics literature on paired comparisons. \citet{causeur2005two} propose a BT model in which the worth of each object is bidimensional, for settings, such as sensory evaluation, in which a linear ranking of the objects is not meaningful, and \citet{chen2016modeling} represent each item via multi-dimensional vectors to model intransitive outcomes.
However, these models are descriptive and, differently from MOBT, they neither include incomparability as a distinct outcome nor are constrained to comply with Pareto dominance.

\textbf{Non-Preference Feedback and Incomplete Preferences.}~~The standard PbRL approach prescribes that the BT model characterizes only direct and inverse preferences, without accounting for \emph{non-preference} feedback, i.e., indifference and incomparability.
Several approaches have been proposed over the years to model \emph{indifference} at varying levels of complexity.
The most common one does not modify the BT model, defining the expert's preference as a distribution over the two trajectories instead, and allowing for indifference by dividing this probability between the two trajectories \citep{christiano2017deep}.
An alternative approach modifies the BT model, characterizing the probability of observing an indifference as well as the probabilities of observing the two preferences.
\citet{rao1967ties} define a \quotes{threshold of the sensory perception of the judge}, i.e., a value such that, if the difference in utility is lower than this value, the expert is unable to state a clear preference.
\citet{davidson1970extending}, instead, models the \quotes{no preference} probability as proportional to the geometric mean of the probabilities of the two preferences, so that it is most likely when the two alternatives are equally good. Notably, MOBT falls back to this model when $d = 1$ (Proposition~\ref{prop:mobt1d}).
Both models have recently been used to extend DPO to accommodate ties, showing that tied pairs can be included in the training data without the performance degradation they cause in standard DPO \citep{chen2025extending}.

Differently from indifference, \emph{incomparability} arises when the expert is unable, or rationally chooses not, to take a definite position w.r.t.\ a comparison. Its roots are in decision theory: the completeness axiom of utility theory was already questioned by \citet{vonneumann1947theory}, and \citet{aumann1962utility} developed utility theory without it. Incomplete preferences can be represented via a set of utility functions, with a clear preference arising only when all of them agree \citep{ok2002utility}.
\citet{eliaz2006indifference} provide the choice-theoretic foundations to distinguish indifference from \emph{indecisiveness}, which we follow to split non-preference feedback into two distinct outcomes, and \citet{ok2022indifference} show how the two can be told apart from stochastic choice data, which is the regime of our rationality model.
In multi-criteria decision making, incomparability is among the basic preference situations, together with strict preference, weak preference, and indifference \citep{roy1984relational, roy1996multicriteria}. \citet{tsoukias1995new} provide its axiomatic foundation, identifying two reasons for incomparability.
In social choice theory, \citet{sen1970collective} studies collective choice under incomplete orderings, and \citet{sen1997maximization} argues that, under incompleteness, the rational choice is the set of \emph{maximal} alternatives, \ie those that are not dominated by any other, which corresponds to the Pareto front in the multi-objective setting.

The most common approach in the presence of incomparability in PbRL is to discard the query and consider the sample as erroneous \citep{christiano2017deep}.
Yet, incomparability has been acknowledged in PbRL since its formalization, as preferences over trajectories are only required to induce a partial order \citep{furnkranz2012preference}, and its explicit treatment has recently been identified as an open problem, suggesting that it could be exploited to learn a Pareto front of policies \citep{abdelkareem2022advances}.
Importantly, incomparability should not be confused with \emph{abstention}. In machine learning, the reject option allows a \emph{learner} to abstain from a prediction when uncertain \citep{hendrickx2024machine}.
Abstention is driven by uncertainty, which corresponds to the first reason for incomparability discussed in Section~\ref{sec:framework}, commonly assumed in PbRL. In contrast, we model incomparability as driven by a conflict among the objectives, which is ruled out when one trajectory dominates the other.
Our work explicitly models this feedback, treating it as a primary signal to identify the presence of multiple, conflicting objectives.

\textbf{Motivations from Behavioral Literature.}~~The idea that an expert may rationally choose to avoid stating a preference when the alternatives are in conflict is supported by the behavioral literature. \citet{tversky1992choice} show that the tendency to defer a decision increases with the conflict between the options, and \citet{dhar1997consumer} shows that deferral is also more likely when the options are similarly attractive, a distinction that mirrors the one between incomparability and indifference in our framework. More recently, \citet{costagomes2022choice} find that the choices of subjects who are allowed to defer are more consistent than those of subjects who are forced to choose, and that the combined deferral and choice behavior is better explained by a model with incomplete preferences than by rational choice. Together, these works suggest that forcing a clear preference in the presence of conflicting objectives may result in less consistent feedback, similarly to what was argued by \citet{aumann1962utility}.

\textbf{Multi-Objective Reinforcement Learning.}~~MORL is the extension of RL to scenarios in which the agent must optimize multiple, often conflicting, criteria, with the goal of discovering the set of Pareto-optimal policies. Similar results to those of decision theory, highlighting the limitations of standard RL, have been proposed in recent years. \citet{bowling2023settling} characterize the axioms, among which completeness, under which goals can be represented by a scalar Markovian reward, whereas \citet{skalse2023limitations} show that most multi-objective tasks cannot be expressed via a scalar Markovian reward.
A common approach to MORL is to consider a scalarization function, mapping each weighting\footnote{We use the term \quotes{weight} to refer to the scalarization weight of the multiple objectives rather than the term \quotes{preference} commonly used in MORL to avoid ambiguity.} of the objectives to a point on the Pareto front, and allowing learning a single policy conditioned on the scalarization weight, generalizing across weights \citep{yang2019generalized}.
Linear scalarization is guaranteed to recover the whole front when the set of achievable values is convex, e.g., for stochastic policies in finite MDPs \citep{lu2023multi}.
Recent research has also emphasized the controllability of the front. \citet{qiu2024traversing} analyze several optimization targets in terms of their ability to find all Pareto-optimal policies and to control the learned policy via the weights, identifying Tchebycheff scalarization as a favorable choice, and propose a \emph{preference-free} approach that first explores the environment and then computes the optimal policies for any weight. Preference-free exploration for MORL was first studied by \citet{wu2021accommodating}, and \citet{chen2026reward} adapted \emph{reward-free RL} \citep{jin2020reward} to deep MORL, using its training objective as an auxiliary task to learn a weight-conditioned policy that can adapt to different weights at inference time.

In this work, we consider linear scalarization, and provide results of what can be identified and guaranteed under significantly weaker signals, namely pairwise trajectory-level comparisons instead of the dense, multi-dimensional instantaneous rewards of MORL.
A related line of work considers the \emph{dual} problem, in which the objectives are known and the weights of the user are unknown and elicited via comparisons, either between outcome vectors \citep{zintgraf2018ordered, lin2022preference} or between policies and sets of trajectories \citep{shao2023eliciting}. In CbRL, instead, the objectives themselves are unknown and must be learned from comparisons, without eliciting any weight.

\textbf{Multi-Objective Problems with Human Feedback.}~~In recent years, approaches have been proposed to combine multi-objective problems with human feedback. Such approaches can be classified in terms of what they require of the expert. \citet{mu2025preference} formalize PbRL in a multi-objective setting by providing the expert with a weighting of the objectives alongside the two trajectories, and eliciting a (possibly indeterminate) preference under such a weighting, which allows them to recover the whole convex Pareto front.
In our opinion, this approach trades the use of a convenient rationality model for an additional burden on the human expert. Indeed, scalarizing the objectives provides the exact information that comparisons alone cannot reveal, resolving the identifiability ambiguity.
On the other hand, eliciting a preference under a specific numerical scalarization requires not only that the expert combine the utilities but that they can also quantify such utilities in the first place.
Other works elicit a separate feedback for each objective, \eg distinct preferences for helpfulness and harmlessness \citep{dai2024safe}, fine-grained feedback on different categories of errors \citep{wu2023fine}, multi-attribute ratings complemented by preferences \citep{wang2024helpsteer2}, one comparison per objective in preferential Bayesian optimization \citep{astudillo2025preferential}, or per-objective teachers \citep{adikari2026lemur}.
\citet{cherukuri2025learning}, instead, assume that the expert internally scalarizes the objectives with an unknown weighting before providing a binary preference, and \citet{boldi2025pareto} treat the preferences of different groups of annotators as distinct objectives, learning a set of Pareto-optimal policies.
None of these works formalizes incomparability as a distinct outcome. CbRL lies at the opposite end of the trade-off between human effort and informativeness of the feedback, requesting only a four-class label given two trajectories, at the cost of a less informative feedback.

\textbf{Limitations of PbRL from the Social Choice Theory Perspective.}~~The standard assumption in PbRL that human preferences can be effectively represented by a scalar utility function has come under scrutiny again in recent years.
\citet{chakraborty2024maxmin} derive a fundamental impossibility result, showing that a single scalar representation is insufficient for capturing the heterogeneity of preferences across a diverse set of users. Learning a single scalar reward can thus lead to a \emph{representational failure}, where the reward model aligns with a majority group of users while ignoring the conflicting preferences of the minority.
From an axiomatic perspective, \citet{ge2024axioms} show that BT-based reward learning can fail to respect a preference shared by all annotators, i.e., violate Pareto optimality. Similarly, \citet{siththaranjan2024distributional} show that, when preferences depend on a \emph{hidden context}, e.g., different annotators or different evaluation criteria, BT-based preference learning implicitly aggregates them according to the Borda count.
In the direction of overcoming these limitations, \citet{munos2024nash} replace the scalar utility with a general pairwise preference model, which can also represent non-transitive preferences, aligning the policy with the Nash equilibrium of the induced game.
This can be seen as a shift of the paradigm towards a pluralistic alignment, emphasizing the need for an agent that can represent and align with a multitude of human values simultaneously, rather than converging to a single, potentially biased average \citep{sorensen2024position}.
Shifting from scalar to vector-based rewards allows accounting for the inherent multi-objective nature of human judgment.
Although these works consider multiple annotators, while we consider a single expert evaluating multiple conflicting objectives, the two settings are closely related, with the objectives playing the role of the annotators, and Pareto dominance playing the role of unanimity. In this sense, our dominance compliance is the single-expert counterpart of the Pareto optimality axiom of \citet{ge2024axioms}, and the fact that comparisons identify the aggregate utility, but not the individual objectives, parallels the Borda aggregation of \citet{siththaranjan2024distributional}.

\textbf{Statistical Theory of PbRL and Query Design.}~~The statistical analysis of learning from preferences has mostly focused on the BT model. \citet{zhu2023principled} study the maximum likelihood estimator of a linear BT reward from pairwise and $K$-wise comparisons, and \citet{zhan2024provable} study offline PbRL under general function approximation, both relying on coverage conditions on the comparison data. \citet{saha2023dueling} provide regret guarantees for online RL with trajectory preferences, and \citet{wang2023rlhf} show that, under suitable conditions, PbRL can be reduced to standard reward-based RL.
Our sample complexity results extend this line of work to four-class comparison feedback, where the non-convexity of the likelihood, which we show is unavoidable for every compliant model, requires accounting for local optima, and the design conditions of Assumption~\ref{asm:rich_design} play the role of the coverage conditions of offline PbRL.
To reduce the number of queries, strategies coming from active learning and preference elicitation \citep{sadigh2017active, biyik2020active} might be used to strategically select the most informative trajectory pairs, ensuring that the learning process remains feasible for human experts.

\textbf{Reward Identifiability.}~~The reward that explains a given behavior or feedback is, in general, not unique, as already observed in inverse RL \citep{ng2000algorithms}. Conditions under which rewards can be identified from demonstrations have been studied by \citet{kim2021reward} and \citet{cao2021identifiability}. \citet{skalse2023invariance} characterize the ambiguity induced by several data sources, including trajectory comparisons. The rationality model also plays a role, as the regret model enjoys more desirable identifiability properties than the partial return one \citep{knox2024models}.
Our identifiability result follows the same approach, first characterizing the class of rewards that induce the same comparison distribution, and then studying which downstream guarantees can be provided despite the ambiguity.

\section{Derivation of the Incomparability Score}
\label{apx:derivation}

In this appendix, we derive the incomparability score $h_{\incomprel}$ of Equation~\eqref{eq:mobt:incomparability}.
As discussed in the main paper, Equation~\eqref{eq:desiderata:incomparability} requires incomparability to be the most probable feedback along all the non-standard diagonals of the $d$-dimensional hyperspace.

\paragraph{Dispersion as the Base for Incomparability.}
Intuitively, we base our modeling approach on the idea that points along the non-standard diagonals, i.e., all the rays that can be defined as $x\bm{t}$ for $x > 0$ and $\bm{t} \in \{-1,1\}^d \setminus \{-\bm{1}_d, \bm{1}_d\}$, \emph{grow farther} from the standard diagonal.
Thus, we quantify our incomparability score based on the distance of the point represented by $\bm{\delta}(\tau, \tau')$ from its projection on the standard diagonal.

Let $\bm{s} \coloneqq \bm{1}_d / \sqrt{d}$ be a vector such that the standard diagonal of $\mathbb{R}^d$ can be defined as $c \bm{s}$ for $c \in \mathbb{R}$. Denote as $\bm{p} = \bar{c}\bm{s}$ be the projection of $\bm{\delta}(\tau, \tau')$ onto the main diagonal.
We define the following function:
\begin{equation}
\label{eq:derivation:score}
    \widetilde{h}_{\incomprel}(\bm{\delta}(\tau, \tau')) \coloneqq \left\| \bm{\delta}(\tau, \tau') - \bm{p} \right\|_2.
\end{equation}
By definition of $\bm{p}$, it follows that $\bm{s}$ is orthogonal to $\bm{p} - \bm{\delta}(\tau, \tau')$, i.e.:
\begin{equation}
    \label{eq:derivation:orthogonal}
    \bm{s}^\top \left( \bm{p} - \bm{\delta}(\tau, \tau') \right) = 0.
\end{equation}
By applying the definition of $\bm{p}$, we can rewrite Equation~\eqref{eq:derivation:orthogonal} as $\bar{c} - \bm{s}^\top \bm{\delta}(\tau,\tau') = 0$. Then, applying the definition of $\bm{s}$ and solving for $\bar{c}$, we get:
\begin{equation}
    \bar{c} = \frac{1}{\sqrt{d}} \bm{1}_d^\top \bm{\delta}(\tau, \tau').
\end{equation}
Plugging this equation into the definition of $\bm{p}$, we can rewrite it as:
\begin{equation}
\label{eq:derivation:projection}
    \bm{p} = \bar{c}\bm{s} = \frac{1}{d} \bm{1}_d \bm{1}_d^\top \bm{\delta}(\tau, \tau') = \bm{A}_d \bm{\delta}(\tau, \tau').
\end{equation}
Now, we define the \emph{mean element} of the vector $\bm{\delta}(\tau, \tau')$ as:
\begin{equation}
\label{eq:derivation:meandiff}
    \bar{\delta}(\tau, \tau') \coloneqq \frac{1}{d} \sum_{i=1}^{d} \delta_{i}(\tau, \tau') = \frac{1}{d} \bm{1}_d^\top \bm{\delta}(\tau, \tau').
\end{equation}
Pluggin Equation~\eqref{eq:derivation:meandiff} into Equation~\eqref{eq:derivation:projection} allows us to rewrite the projection of $\bm{\delta}(\tau, \tau')$ onto the standard diagonal as $\bm{p} = \bar{\delta}(\tau, \tau') \bm{1}_d$, i.e., as the $d$-dimensional vector having all elements equal to $\bar{\delta}(\tau, \tau')$. Finally, combining this result into Equation~\eqref{eq:derivation:score}, and applying the definitions of $\bm{\Pi}_d = \bm{I}_d - \bm{A}_d$ and of population standard deviation, we obtain, for every $\bm{\delta}(\tau, \tau') \in \mathbb{R}^d$:
\begin{equation}
\label{eq:derivation:dispersion}
    \widetilde{h}_{\incomprel}(\bm{\delta}(\tau, \tau')) = \left\| \bm{\delta}(\tau, \tau') - \bar{\delta}(\tau, \tau')\bm{1}_d \right\|_2 = \| \bm{\Pi}_d \bm{\delta}(\tau, \tau') \|_2 = \sqrt{d} \std (\bm{\delta}(\tau, \tau')),
\end{equation}
allowing us to conveniently base of incomparability score on the \emph{dispersion} of the elements of $\bm{\delta}(\tau, \tau')$, i.e., on how much the utility difference varies across the different objectives.

\paragraph{Dispersion Alone Does Not Guarantee Dominance Compliance.}
Dispersion alone, however, is not a satisfactory incomparability score. This is because Equation~\eqref{eq:derivation:dispersion} measures the distance of from the standard diagonal, whereas Definition~\ref{defi:dominance} constrains the behavior of the model on the positive and negative orthants. Indeed, dispersion is large whenever the utilities for the different objectives disagree in \emph{magnitude}, while Definition~\ref{defi:dominance} tells us that incomparability should meaningfully arise only when they disagree in \emph{sign}.

As an example, consider the case in which $\bm{\delta}(\tau, \tau') = c \bm{e}_1$ with $c > 0$. This utility difference vector represents the case in which the first trajectory is better than the second by a quantity $c$ in terms of the first objective, and exactly equivalent in terms of all the other objectives. Thus, the first trajectory dominates the second, and Definition~\ref{defi:dominance} prescribes that $f(\prelsinv | c \bm{e}_1) > f(\incomprel | c \bm{e}_1)$. Equation~\eqref{eq:derivation:dispersion} then gives us:
\begin{equation}
    \widetilde{h}_{\incomprel} (c \bm{e}_1) = \left\| c\bm{e}_1 - \overline{c \bm{e}_1} \bm{1}_d \right\|_2 = c \sqrt{\frac{d-1}{d}},
\end{equation}
where $\overline{c\bm{e}_1} = c/d$, whereas in the MOBT model we have that $h_{\prelsinv}(c \bm{e}_1) = c/2$. Indeed, $\sqrt{(d-1)/d} > 1/2$ for every $d \ge 2$, thus Definition~\ref{defi:dominance} cannot be satisfied with this formulation. Clearly, this holds symmetrically for $h_{\prels}$.
However, we now show that this can be interpreted as a failure of \emph{scale}, and can thus be addressed, at least for some values of $d$.
Indeed, since both the dispersion of Equation~\eqref{eq:derivation:dispersion} and $(\bm{1}_d^\top \bm{\delta}(\tau, \tau'))/2$, which guides both direct and inverse preferences, are positively homogeneous of degree one, we can interpret their comparison along any direction as the comparison between their slopes, i.e.:
\begin{equation}
    \varrho(\bm{\delta}(\tau, \tau')) \coloneqq \frac{(\bm{1}_d^\top \bm{\delta}(\tau, \tau'))/2}{\|\bm{\delta}(\tau, \tau') - \bar{\delta}(\tau, \tau') \bm{1}_d\|_2},
\end{equation}
for every $\bm{\delta}(\tau, \tau') \in \mathbb{R}^d$ such that $\std (\bm{\delta}(\tau, \tau')) > 0$. This suggests rescaling the dispersion by a factor $\gamma > 0$, choosing it as a dimensionality-dependent threshold on $\varrho$ to ensure compliance with Definitions~\ref{defi:desiderata} and~\ref{defi:dominance}.

\paragraph{Conlusion of the Derivation.}
We conclude the derivation of the incomparability score by adding a conflict sensitivity parameter $\gamma$ and a free propensity parameter $\beta$:
\begin{equation}
    h_{\incomprel}(\bm{\delta}(\tau, \tau')) = \gamma \widetilde{h}_{\incomprel}(\bm{\delta}(\tau, \tau')) + \beta.
\end{equation}
We choose to add a learnable parameter $\beta \in \mathbb{R}$ to the incomparability score for two reasons: ($i$) it does not affect the satisfiability of the desiderata, as it can simply be set to zero, ($ii$) it gives the model an additional degree of freedom to handle situations in which the expert may present a slight irrationality or may be more or less prone to evaluate two trajectories as incomparable.

\section{Additional Desiderata}

Notably, Definition~\ref{defi:desiderata} characterizes only the mode of the outcome, leaving room to model different behaviors, e.g., the expert's degree of irrationality, and does not prescribe how different outcomes should relate to each other.
One such intuitive relation is symmetry with respect to the origin.
\begin{defi}[Symmetry Compliance]
\label{defi:symmetry}
    A function $f: \mathbb{R}^d \to \Delta(\Comps)$ is \emph{symmetry compliant} if, for every $\bm{\delta} \in \mathbb{R}^d$ it holds that $f(\prelsinv|\bm{\delta})=f(\prels|-\bm{\delta})$, $f(\equivrel|\bm{\delta})=f(\equivrel|-\bm{\delta})$, and $f(\incomprel|\bm{\delta}) = f(\incomprel|-\bm{\delta})$.
\end{defi}
Clearly, Definition~\ref{defi:symmetry} only requires that the model should be consistent with exchanging the order of presentation of the trajectories.
Notably, the previous desiderata do not prescribe how the model should behave in the different orthants of $\mathbb{R}^d$. For example, consider two trajectories $\tau, \tau' \in \Ts$  equal w.r.t.\ all objectives except for one, where $\tau$ is better than $\tau'$. Intuitively, we might conclude that the expert should state that $\tau$ is preferred, or at most that the two options are indifferent. Hence, the following desideratum.
\begin{defi}[Dominance Compliance]
\label{defi:dominance}
    A function $f: \mathbb{R}^d \to \Delta(\Comps)$ is \emph{dominance compliant} if, for every $\bm{\delta} \in \mathbb{R}^d_{\ge 0} \setminus \{\bm{0}_d\}$, it holds that $f(\prelsinv| \bm{\delta}) > f(\prels | \bm{\delta})$ and $f(\prelsinv | \bm{\delta}) > f(\incomprel | \bm{\delta})$ and, similarly, for every $\bm{\delta} \in \mathbb{R}^{d}_{\le 0} \setminus \{\bm{0}_d\}$, it holds that $f(\prels|\bm{\delta})>f(\prelsinv|\bm{\delta})$ and $f(\prels | \bm{\delta}) > f(\incomprel | \bm{\delta})$.
\end{defi}
Definition~\ref{defi:dominance} requires that, when one option \emph{dominates} the other, i.e., is better with respect to all objectives \citep{miettinen1999nonlinear}, clear preference is assigned a higher probability than incomparability.
We observe that Definitions~\ref{defi:desiderata} and~\ref{defi:dominance} do not imply each other, as the latter prescribes the model's behavior only in the positive and negative orthants.
Moreover, Definition~\ref{defi:dominance} does not prescribe how preferences should relate to indifferences. Indeed, the existence of the limit in Equation~\eqref{eq:desiderata:indifference}, together with the fact that $\equivrel$ is its unique maximizer, already implies that $\equivrel$ is the mode of $f$ at every point in the \emph{neighborhood} of $\bm{0}_d$. Thus, dominance compliance would be unattainable if defined also w.r.t.\ indifference.

We now discuss model compliance. The standard deviation, which intuitively relates incomparability with the distance from the standard diagonal, does not guarantee dominance compliance, as it only measures disagreement \emph{in magnitude}. Hence, the conflict sensitivity parameter $\gamma$, which induces compliance with Definition~\ref{defi:dominance}.

\begin{restatable}[Compliant Conflict Sensitivity]{prop}{gammarange}
\label{prop:gammarange}
    Let $d \ge 2$. Define:
    \begin{align}
    \label{eq:gammarange}
        \Gamma_d \coloneqq (\underline{\gamma}_d, \overline{\gamma}_d) = \left( \frac{(d-2)\sqrt{d}}{4\sqrt{d-1}}, \; \frac{\sqrt{d}}{2\sqrt{d-1}} \right).
    \end{align}
    The MOBT model satisfies Definitions~\ref{defi:desiderata} and~\ref{defi:symmetry} for every $\alpha > 0 \ge \beta$ and $\gamma > \underline{\gamma}_d$. Moreover, the model also satisfies Definition~\ref{defi:dominance} if $\gamma < \overline{\gamma}_d$.
\end{restatable}
Notably, $\Gamma_d \neq \emptyset$ if and only if $d \le 3$, obtaining $\Gamma_2 = (0, 1/\sqrt{2})$ and $\Gamma_3 = (\sqrt{6}/8, \sqrt{6}/4)$.
Hence, compliance with Definitions~\ref{defi:desiderata} and~\ref{defi:symmetry} can always be guaranteed, but the MOBT can jointly guarantee compliance with Definition~\ref{defi:dominance} only if $d \le 3$.
In Section~\ref{sec:conclusion}, we discuss this open challenge.

\section{Omitted Proofs}
\label{apx:proofs}

In this appendix, we report proofs of the results in the main paper.

\noncvx*
\begin{proof}
We demonstrate the proposition by contradiction.
Assume that $f: \mathbb{R}^d \to \Delta(\Comps)$ satisfies Definition~\ref{defi:desiderata} and that, for every $\comp \in \Comps$, $-\log f(\comp | \bm{\delta})$ is convex.
Since $d \ge 2$, we have that the set $\{-1, 1\}^d \setminus \{\bm{1}_d, -\bm{1}_d\}$ is non-empty, and Equation~\eqref{eq:desiderata:incomparability} is thus well-defined.
Thus, we can fix $\bm{t} \in \{-1, 1\}^d \setminus \{\bm{1}_d, -\bm{1}_d\}$.
For ease of reading, we slightly abuse the notation and write $f_{\comp} (x) \coloneqq f(\comp | x \bm{t})$, for $x \in \mathbb{R}$.

By convexity of $-\log f(\comp | \bm{\delta})$, it follows that $x \mapsto$ $- \log f_{\comp}(x)$ is convex on $\mathbb{R}$ for every $\comp \in \Comps$, as it is the restriction of a convex function to an affine line \citep[][Section 3.1.1]{boyd2004convex}.

We will repeatedly use the following consequence of convexity. Let $g : \mathbb{R} \to \mathbb{R} \cup \{+\infty\}$ be convex and such that $\lim_{x \to +\infty} g(x) \in \mathbb{R}$, and assume, by contradiction, that $g(x_1) < g(x_2)$ for some $x_1 < x_2$. Then, $g$ is finite on $[x_1, +\infty)$, since its domain is an interval that contains $x_1$ and a right half-line. By the monotonicity of the slopes of the chords of convex functions \citep[see, \eg][]{rockafellar1970convex}, for every $x > x_2$, it follows that $g(x) \ge g(x_2) + (x - x_2) (g(x_2) - g(x_1)) / (x_2 - x_1)$, whose right-hand side diverges to $+\infty$ as $x \to +\infty$, contradicting the finiteness of the limit. Applying the same argument to $x \mapsto g(-x)$, we obtain:
\begin{equation}
\label{eq:proof:convex-monotone}
    \lim_{x \to +\infty} g(x) \in \mathbb{R} \implies g \text{ is non-increasing on } \mathbb{R},
    \qquad \text{and} \qquad
    \lim_{x \to -\infty} g(x) \in \mathbb{R} \implies g \text{ is non-decreasing on } \mathbb{R},
\end{equation}
and, if both limits are finite, $g$ is constant on $\mathbb{R}$.

The desiderata of Equations~\eqref{eq:desiderata:indifference} and~\eqref{eq:desiderata:incomparability} presuppose the existence of the limits appearing in them. Moreover, since $-\bm{t} \in \{-1, 1\}^d \setminus \{\bm{1}_d, -\bm{1}_d\}$ as well, Equation~\eqref{eq:desiderata:incomparability} constrains the behavior of $f$ at \emph{both} ends of the line $\{x \bm{t} : x \in \mathbb{R}\}$. Accordingly, we denote, for every $\comp \in \Comps$:
\begin{equation}
\nonumber
    p_{\comp}^{+} \coloneqq \lim_{x \to + \infty} f_{\comp} (x),
    \qquad
    p_{\comp}^{-} \coloneqq \lim_{x \to - \infty} f_{\comp} (x),
    \qquad
    q_{\comp} \coloneqq \lim_{x \to 0} f_{\comp} (x),
\end{equation}
where $q_{\comp}$ coincides with $\lim_{\bm{\delta} \to \bm{0}_d} f(\comp | \bm{\delta})$, since $x \bm{t} \to \bm{0}_d$ as $x \to 0$. Since $f_{\comp}(x) \ge 0$ and $\sum_{\comp \in \Comps} f_{\comp}(x) = 1$ for every $x \in \mathbb{R}$, passing to the limit in a sum of four convergent terms yields:
\begin{equation}
\nonumber
    \sum_{\comp \in \Comps} p_{\comp}^{+}
    = \sum_{\comp \in \Comps} p_{\comp}^{-}
    = \sum_{\comp \in \Comps} q_{\comp}
    = 1.
\end{equation}

First, we prove that $f_{\incomprel}$ is constant. Equation~\eqref{eq:desiderata:incomparability}, applied to the $\bm{t}$ and $-\bm{t}$ directions, prescribes that $\incomprel$ is the \emph{unique} maximizer of the limiting distribution at both ends, i.e.:
\begin{equation}
\nonumber
    p_{\incomprel}^{+} > p_{\comp}^{+}, \qquad p_{\incomprel}^{-} > p_{\comp}^{-},
    \qquad \text{for } \comp \in \{\prelsinv, \prels, \equivrel\}.
\end{equation}
Observing that $\sum_{\comp \in \Comps} f_{\comp}(x) = 1$ for every $x \in \mathbb{R}$ and $|\Comps|=4$, it follows that $p_{\incomprel}^{+} > 1/4$ and $p_{\incomprel}^{-} > 1/4$.
In particular, both limits are strictly positive, so $- \log f_{\incomprel}(x)$ admits the finite limits $-\log p_{\incomprel}^{+}$ and $-\log p_{\incomprel}^{-}$ at $+\infty$ and $-\infty$, respectively. This, combined with the convexity of $-\log f_{\incomprel}$ implies that $f_{\incomprel}(x)$ is constant for every $x \in \mathbb{R}$, following Equation~\eqref{eq:proof:convex-monotone}, and consequently:
\begin{equation}
    f_{\incomprel} (x) = p_{\incomprel}
    \qquad \forall x \in \mathbb{R},
    \qquad \text{with} \quad
    p_{\incomprel} \coloneqq p_{\incomprel}^{+} = p_{\incomprel}^{-} = q_{\incomprel} > \frac{1}{4}.
\end{equation}

Second, we show that $p_{\incomprel} < 1/2$. By Equation~\eqref{eq:desiderata:indifference}, it follows that $\equivrel$ is the unique maximizer of $f$ in the origin, hence $q_{\equivrel} > q_{\incomprel} = p_{\incomprel}$. Since $q_{\equivrel} + q_{\incomprel} \le 1$, it follows that $2 p_{\incomprel} < q_{\equivrel} + p_{\incomprel} \le 1$, and thus:
\begin{equation}
\label{eq:proof:p-half}
    \frac{1}{4} < p_{\incomprel} < \frac{1}{2}.
\end{equation}

Third, we show that $p_{\equivrel}^{+} = p_{\equivrel}^{-} = 0$. Suppose, by contradiction, that $p_{\equivrel}^{+} > 0$. Then, $\lim_{x \to +\infty} - \log f_{\equivrel} (x) = - \log p_{\equivrel}^{+}$ is finite, and the first implication of Equation~\eqref{eq:proof:convex-monotone}, applied to $- \log f_{\equivrel}$, implies that $f_{\equivrel}$ is non-decreasing in $\mathbb{R}$. Consequently, $f_{\equivrel}(x) \le p_{\equivrel}^{+}$ for every $x \in \mathbb{R}$, and passing to the limit as $x \to 0$, we get that:
\begin{equation}
\nonumber
    q_{\equivrel} \le p_{\equivrel}^{+} < p_{\incomprel},
\end{equation}
where the last inequality comes from compliance with Equation~\eqref{eq:desiderata:incomparability}. This contradicts $q_{\equivrel} > q_{\incomprel} = p_{\incomprel}$, which is required by assumption from Equation~\eqref{eq:desiderata:indifference}.
The case of $p_{\equivrel}^{-}$ is analogous, and follows from the second implication of Equation~\eqref{eq:proof:convex-monotone}.
Hence, both limits of $f_{\equivrel}$ vanish, and $\sum_{\comp \in \Comps} f_{\comp}(x) = 1$ yields:
\begin{equation}
\label{eq:proof:pref-mass}
    p_{\prelsinv}^{+} + p_{\prels}^{+} = p_{\prelsinv}^{-} + p_{\prels}^{-} = 1 - p_{\incomprel}.
\end{equation}

Fourth, we show that all the limits of direct and inverse preference probabilities are strictly positive. By Equation~\eqref{eq:desiderata:incomparability}, we have $p_{\prelsinv}^{+}, p_{\prelsinv}^{-}, p_{\prels}^{+}, p_{\prels}^{-} < p_{\incomprel}$. Combining these with Equation~\eqref{eq:proof:pref-mass} yields:
\begin{equation}
\nonumber
    p_{\prelsinv}^{+} = 1 - p_{\incomprel} - p_{\prels}^{+} > 1 - 2 p_{\incomprel} > 0, \quad \text{and} \quad p_{\prelsinv}^{-} = 1 - p_{\incomprel} - p_{\prels}^{-} > 1 - 2 p_{\incomprel} > 0,
\end{equation}
and, symmetrically, $p_{\prels}^{+}, p_{\prels}^{-} > 0$, where both strict inequalities follow from Equation~\eqref{eq:proof:p-half}.

Finally, since $p_{\prelsinv}^{+}, p_{\prelsinv}^{-}, p_{\prels}^{+}, p_{\prels}^{-} > 0$, Equation~\eqref{eq:proof:convex-monotone}, applied to $- \log f_{\prelsinv}$ and $- \log f_{\prels}$, implies that both $f_{\prelsinv}$ and $f_{\prels}$ are constant on $\mathbb{R}$, and equal to their respective limiting values. Since $\sum_{\comp \in \Comps} f_{\comp}(x)=1$ and by Equation~\eqref{eq:proof:pref-mass}, it follows that $f_{\equivrel}$ is also constant, and equal to its limits at $\pm \infty$, which vanish. Hence $f_{\equivrel} \equiv 0$ and, in particular, $q_{\equivrel} = 0$, which contradicts $q_{\equivrel} > p_{\incomprel} > 1/4$. Therefore, no such function $f$ exists, thus concluding the proof.
\end{proof}

\gammarangesimp*
\begin{proof}
    This result is included in the more general Proposition \ref{prop:gammarange}.
\end{proof}

\gammarange*
\begin{proof}
Before demonstrating the result, we establish some auxiliary notation that we will use throughout the proof.
Let us denote, following Equation~\eqref{eq:derivation:dispersion}, as $\widetilde{h}_{\incomprel}(\bm{\delta}) = \sqrt{d} \std (\bm{\delta}) = \|\bm{\delta} - \bar{\delta}\bm{1}_d \|_2$ the dispersion of the elements of $\bm{\delta}$, where $\bar{\delta} = (\bm{1}_d^\top \bm{\delta})/d$, upon which the incomparability score is based.
The following identity holds for every $\bm{\delta} \in \mathbb{R}^d$:
\begin{equation}
\label{eq:gammarange:square}
    \widetilde{h}_{\incomprel}(\bm{\delta})^2 = \|\bm{\delta}\|_2^2 - 2 \bar{\delta} \bm{1}_d^\top \bm{\delta} + d \bar{\delta}^2 = \|\bm{\delta}\|_2^2 - \frac{(\bm{1}_d^\top \bm{\delta})^2}{d}.
\end{equation}
We observe some properties of $\widetilde{h}_{\incomprel}$: ($i$) it is continuous and non-negative, since it is the composition of the Euclidean norm with a linear map, ($ii$) it satisfies $\widetilde{h}_{\incomprel}(\bm{0}_d) = 0$, ($iii$) it is positively homogeneous of degree one, since $\overline{x\bm{\delta}} = x \bar{\delta}$ for every $x \in \mathbb{R}$, and ($iv$) it is even, i.e., $\widetilde{h}_\incomprel(-\bm{\delta}) = \widetilde{h}_\incomprel(\bm{\delta})$.

Let $\bm{t} \in \mathbb{R}^d$ and $x > 0$. By Definition~\ref{defi:mobt} and the positive homogeneity of $\widetilde{h}_{\incomprel}$ and the map $\bm{\delta} \mapsto \bm{1}_d^\top \bm{\delta}/2$, the four MOBT scores are affine in $x$ along the ray of direction $\bm{t}$ for every $\comp \in \Comps$, i.e.:
\begin{equation}
\label{eq:gammarange:affine}
    h_\comp (x \bm{t}) = x a_\comp (\bm{t}) + b_\comp,
\end{equation}
where:
\begin{equation}
\label{eq:gammarange:slopes}
    \bm{a}(\bm{t}) \coloneqq \left( a_{\prelsinv}(\bm{t}), a_{\prels}(\bm{t}), a_{\equivrel}(\bm{t}), a_{\incomprel}(\bm{t}) \right)^\top = \left( \frac{\bm{1}_d^\top \bm{t}}{2}, - \frac{\bm{1}_d^\top \bm{t}}{2}, 0, \gamma \widetilde{h}_{\incomprel}(\bm{t}) \right)^\top \text{and } \bm{b} \coloneqq \left( b_{\prelsinv}, b_{\prels}, b_{\equivrel}, b_{\incomprel} \right)^\top = \left( 0, 0, \alpha, \beta \right)^\top.
\end{equation}
We are now ready to demonstrate the result.

First, since the four scores in Definition~\ref{defi:mobt} are continuous and $f$ is their softmax, as per Equation~\eqref{eq:mobt_model}, it follows that $f$ is continuous in $\mathbb{R}^d$. In particular, $\lim_{\bm{\delta} \to \bm{0}_d} f(\comp | \bm{\delta}) = f(\comp | \bm{0}_d)$ for every $\comp \in \Comps$.

Moreover, fix $\bm{t} \in \mathbb{R}^d$, and let $a^{\star}(\bm{t}) \coloneqq \max_{\comp \in \Comps} a_{\comp}(\bm{t})$ and $\Comps^{\star}(\bm{t}) \coloneqq \{ \comp \in \Comps : a_{\comp}(\bm{t}) = a^{\star}(\bm{t}) \}$, which is non-empty since $\Comps$ is finite. By Equations~\eqref{eq:mobt_model} and~\eqref{eq:gammarange:affine}, dividing both the numerator and the denominator of $f(\comp | x \bm{t})$ by $\exp(x a^{\star}(\bm{t}))$, we have that, for every $x > 0$:
\begin{equation}
\nonumber
    f(\comp | x \bm{t}) = \frac{\exp \left( (a_{\comp}(\bm{t}) - a^{\star}(\bm{t})) x + b_{\comp} \right)}{\sum_{\comp' \in \Comps} \exp \left( (a_{\comp'}(\bm{t}) - a^{\star}(\bm{t})) x + b_{\comp'} \right)}.
\end{equation}
As $x \to +\infty$, every term with $\comp' \notin \Comps^{\star}(\bm{t})$ vanishes, since $a_{\comp'}(\bm{t}) - a^{\star}(\bm{t}) < 0$, while every term with $\comp' \in \Comps^{\star}(\bm{t})$ is constantly equal to $\exp(b_{\comp'})$. Thus, the limit exists, and it holds that:
\begin{equation}
\label{eq:gammarange:raylimit}
    \lim_{x \to +\infty} f(\comp | x \bm{t}) = \Ind \{ \comp \in \Comps^{\star}(\bm{t}) \} \frac{\exp(b_{\comp})}{\sum_{\comp' \in \Comps^{\star}(\bm{t})} \exp(b_{\comp'})},
    \qquad \forall \comp \in \Comps.
\end{equation}
Since the limit in Equation~\eqref{eq:gammarange:raylimit} vanishes outside $\Comps^{\star}(\bm{t})$ and is strictly increasing in $b_{\comp}$ on $\Comps^{\star}(\bm{t})$, for every $\comp^{\star} \in \Comps$, the limiting distribution along the ray of direction $\bm{t}$ has $\comp^{\star}$ as its unique maximizer if and only if $a_{\comp^{\star}}(\bm{t}) = a^{\star}(\bm{t})$ and $b_{\comp^{\star}} > b_{\comp}$ for every $\comp \in \Comps^{\star}(\bm{t}) \setminus \{ \comp^{\star} \}$.

Second, we prove compliance with the desiderata of Definition~\ref{defi:desiderata}.
Let $\bm{t} = \bm{1}_d$. Since all its elements are equal, it follows that $\widetilde{h}_{\incomprel}(\bm{1}_d) = 0$ and $\bm{1}_d^\top \bm{1}_d = d$, so that Equation~\eqref{eq:gammarange:slopes} gives $\bm{a}(\bm{1}_d) = (d/2, -d/2, 0, 0)^\top$.
Clearly, $\argmax_{\comp \in \Comps} a_{\comp} (\bm{1}_d) = \{\prelsinv\}$, \ie $\Comps^{\star}(\bm{1}_d) = \{\prelsinv\}$, and Equation~\eqref{eq:gammarange:raylimit} yields Equation~\eqref{eq:desiderata:direct_pref}. Symmetrically, given that $\widetilde{h}_{\incomprel}$ is even, the same argument yields Equation~\eqref{eq:desiderata:inverse_pref}. Both results hold for every $\gamma > 0$ and $\alpha, \beta \in \mathbb{R}$.

By the continuity of $f$ established above, it follows that $\lim_{\bm{\delta} \to \bm{0}_d} f(\comp | \bm{\delta}) = f(\comp | \bm{0}_d)$ for every $\comp \in \Comps$. Observing that $\widetilde{h}_{\incomprel} (\bm{0}_d) = 0$ and $\bm{1}_d^\top \bm{0}_d = 0$, the scores at the origin are $(0, 0, \alpha, \beta)$, and $\equivrel$ is the unique maximizer if and only if $\alpha > 0$ and $\alpha > \beta$, both of which are implied by the assumption $\alpha > 0 \ge \beta$. This, combined with the monotonicity of the exponential and the common denominator of the softmax, yields Equation~\eqref{eq:desiderata:indifference}, for every $\gamma > 0$.

Let $\bm{t} \in \{-1, 1\}^d \setminus \{-\bm{1}_d, \bm{1}_d\}$ and let $n_+ \coloneqq |\{i \in \dsb{d} : t_i = 1\}| \in \dsb{d-1}$. Then exactly $n_+$ entries of $\bm{t}$ are equal to $1$ and the remaining $d-n_+$ are equal to $-1$, so that $\bm{1}_d^\top \bm{t} = 2n_+ - d$ and $\|\bm{t}\|_2^2 = d$. Then:
\begin{equation}
    \widetilde{h}_{\incomprel}(\bm{t}) = \sqrt{d - \frac{(2n_+-d)^2}{d}} = 2\sqrt{\frac{n_+(d-n_+)}{d}} > 0,
\end{equation}
where the inequality follows from $n_+ \in \dsb{d-1}$. By Equation~\eqref{eq:gammarange:slopes}, we have that $\max\{a_{\prelsinv}(\bm{t}), a_{\prels}(\bm{t})\} = |2n_+ - d|/2$, $a_{\equivrel}(\bm{t}) = 0$, and $a_{\incomprel}(\bm{t}) = \gamma \widetilde{h}_{\incomprel}(\bm{t}) > 0$ for every $\gamma > 0$. By Equation~\eqref{eq:gammarange:raylimit}, $\incomprel$ is the unique maximizer in the limit along the ray of direction $\bm{t}$ if and only if its slope is the unique maximizer of $\bm{a}(\bm{t})$, i.e.:
\begin{equation}
\label{eq:gammarange:iff}
    \gamma \widetilde{h}_{\incomprel}(\bm{t}) = 2 \gamma \sqrt{\frac{n_+(d-n_+)}{d}} > \frac{|2n_+-d|}{2} \iff \gamma > \frac{|2n_+-d|\sqrt{d}}{4 \sqrt{n_+(d-n_+)}}.
\end{equation}
Notably, Equation~\eqref{eq:gammarange:iff} is required to hold strictly for $\incomprel$ to be the unique maximizer, since if the equation were to hold with equality, $\incomprel$ would have the same slope as one between $\prelsinv$ and $\prels$, but intercept $\beta \le 0$ by hypothesis, thus $\incomprel$ would not be the unique maximizer.
We need to quantify Equation~\eqref{eq:gammarange:iff} over $\bm{t}$, i.e., over $n_+ \in \dsb{d-1}$. Letting $j \coloneqq |2n_+-d|$, and observing that $4n_+(d-n_+) = d^2 - j^2$, then the RHS of Equation~\eqref{eq:gammarange:iff} can be rewritten as $j\sqrt{d}/(2\sqrt{d^2 - j^2})$, which is strictly increasing in $j \in [0, d)$. Its maximum is attained at the largest admissible value of $j$, i.e., $j = d-2$ which is attained when either $n_+=1$ or $n_+=d-1$. Putting these together, we get that the desiderata of Equation~\eqref{eq:desiderata:incomparability} is satisfied for every $\bm{t} \in \{-1, 1\}^d \setminus \{-\bm{1}_d, \bm{1}_d\}$ if and only if:
\begin{equation}
\label{eq:gammarange:left}
    \gamma > \frac{(d-2)\sqrt{d}}{4 \sqrt{d-1}}.
\end{equation}

Third, we prove compliance with Definition~\ref{defi:symmetry}. Let $\bm{\delta} \in \mathbb{R}^d$. By Definition~\ref{defi:mobt} and the evenness of $\widetilde{h}_{\incomprel}$, we have that $h_{\prelsinv}(-\bm{\delta}) = - \bm{1}_d^\top \bm{\delta} / 2 = h_{\prels}(\bm{\delta})$, $h_{\prels}(-\bm{\delta}) = h_{\prelsinv}(\bm{\delta})$, $h_{\equivrel}(-\bm{\delta}) = \alpha = h_{\equivrel}(\bm{\delta})$, and $h_{\incomprel}(-\bm{\delta}) = h_{\incomprel}(\bm{\delta})$. Hence, the scores at $-\bm{\delta}$ are obtained from those at $\bm{\delta}$ by exchanging the scores of $\prelsinv$ and $\prels$, so that the denominators of Equation~\eqref{eq:mobt_model} at $\bm{\delta}$ and at $-\bm{\delta}$ coincide, and thus $f(\prelsinv | \bm{\delta}) = f(\prels | -\bm{\delta})$, $f(\equivrel | \bm{\delta}) = f(\equivrel | -\bm{\delta})$, and $f(\incomprel | \bm{\delta}) = f(\incomprel | -\bm{\delta})$. This holds for every $\gamma > 0$ and $\alpha, \beta \in \mathbb{R}$.

Fourth, we prove compliance with Definition~\ref{defi:dominance}. Let $\bm{\delta} \in \mathbb{R}^d_{\ge 0} \setminus \{\bm{0}_d\}$ and denote $s \coloneqq \bm{1}_d^\top \bm{\delta} = \|\bm{\delta}\|_1 > 0$.
By the monotonicity of the exponential and the common denominator of the softmax, Definition~\ref{defi:dominance} requires $h_{\prelsinv}(\bm{\delta}) > h_{\prels}(\bm{\delta})$, which holds since $h_{\prelsinv}(\bm{\delta}) - h_{\prels}(\bm{\delta}) = s > 0$, and $h_{\prelsinv}(\bm{\delta}) > h_{\incomprel}(\bm{\delta})$, which translates to $\gamma \widetilde{h}_{\incomprel}(\bm{\delta}) + \beta < s/2$.
Recall that $\beta \le 0$ by hypothesis.
We now bound the incomparability score on this set.
By Equation~\eqref{eq:gammarange:square}, and since $(\bm{1}_d^\top \bm{\delta})^2 = s^2$, we have:
\begin{equation}
\label{eq:gammarange:squarebound}
    \widetilde{h}_{\incomprel}(\bm{\delta})^2 = \|\bm{\delta}\|_2^2 - \frac{s^2}{d} \le \|\bm{\delta}\|_1 \|\bm{\delta}\|_{\infty} - \frac{s^2}{d} \le s^2 \left( 1 - \frac{1}{d}\right),
\end{equation}
where the first inequality follows from $\|\bm{\delta}\|_2^2 = \sum_{i=1}^d |\delta_i| |\delta_i| \le \max_{i \in \dsb{d}} |\delta_i| \sum_{i=1}^d |\delta_i| = \|\bm{\delta}\|_1 \|\bm{\delta}\|_{\infty}$ since all elements of $\bm{\delta}$ are non-negative, and the second inequality follows from $\|\bm{\delta}\|_{\infty} \le \|\bm{\delta}\|_1 = s$. Both inequalities hold with equality at $\bm{\delta} = c \bm{e}_i$, for every $c > 0$ and $i \in \dsb{d}$.
By putting Equation~\eqref{eq:gammarange:squarebound} into the definition of $h_{\incomprel}$, and considering $\beta \le 0$, we obtain:
\begin{equation}
    h_{\incomprel}(\bm{\delta}) = \gamma \widetilde{h}_{\incomprel}(\bm{\delta}) + \beta \le \gamma \widetilde{h}_{\incomprel}(\bm{\delta}) \le \gamma s \sqrt{\frac{d-1}{d}}.
\end{equation}
Since $h_{\incomprel}(\bm{\delta}) < s/2$ is required for compliance with Definition~\ref{defi:dominance}, it follows that the inequality holds if:
\begin{equation}
\label{eq:gammarange:right}
    \gamma < \frac{\sqrt{d}}{2 \sqrt{d-1}}.
\end{equation}
Equation~\eqref{eq:gammarange:right} is also necessary. Indeed, take $\beta = 0$, which is admissible since compliance is required for every $\alpha > 0 \ge \beta$, and $\bm{\delta} = \bm{e}_1 \in \mathbb{R}^d_{\ge 0} \setminus \{\bm{0}_d\}$, so that $s = 1$ and, by the equality case of Equation~\eqref{eq:gammarange:squarebound}, $\widetilde{h}_{\incomprel}(\bm{e}_1) = \sqrt{(d-1)/d}$. The requirement $h_{\incomprel}(\bm{e}_1) < s/2$ then corresponds to $\gamma \sqrt{(d-1)/d} < 1/2$, \ie $\gamma < \sqrt{d}/(2\sqrt{d-1})$.

By the evenness of $\widetilde{h}_{\incomprel}$ and considering that $\bm{1}_d^\top (-\bm{\delta}) = - \bm{1}_d^\top \bm{\delta}$, the same holds symmetrically for the negative orthant.

Finally, the definition of $\Gamma_d$ then simply follows by combining the requirements of Equations~\eqref{eq:gammarange:left} and~\eqref{eq:gammarange:right}. Observing that the derivations above hold for $\alpha > 0 \ge \beta$ and $\gamma \in \Gamma_d$, compliance with Definitions~\ref{defi:desiderata}, \ref{defi:symmetry}, and~\ref{defi:dominance} holds.
Notably, $\Gamma_d \neq \emptyset$ if and only if $d-2 < 2$, i.e., $d \le 3$.
\end{proof}

\mobtoned*
\begin{proof}
Fix $\tau, \tau' \in \Ts$. Since $d = 1$, the utility difference of Definition~\ref{defi:mobt} is the scalar $\bm{\delta} = \delta \coloneqq u(\tau) - u(\tau') \in \mathbb{R}$, and $\bm{1}_1^{\top} \bm{\delta} = \delta$.
Since $\bm{\delta}$ is a $1$-dimensional vector, its mean element is $\bar{\delta}=\delta$ and its standard deviation is zero. It follows that $h_{\incomprel}(\delta) = \beta$ for every $\delta \in \mathbb{R}$, and the incomparability score does not depend on $\gamma$.
Thus, the scores of Definition~\ref{defi:mobt} can be written as:
\begin{equation}
    h_{\prelsinv}(\delta) = \frac{\delta}{2}, \qquad h_{\prels}(\delta) = - \frac{\delta}{2}, \qquad h_{\equivrel}(\delta) = \alpha, \qquad \text{and } h_{\incomprel}(\delta) = \beta.
\end{equation}
Since $\beta = -\infty$, we adopt the convention $\exp(-\infty) = 0$, so that $\exp(h_{\incomprel}(\delta)) = 0$ and $f(\incomprel | \delta) = 0$ for every $\delta \in \mathbb{R}$, regardless of $\alpha$. Then, let us denote the normalization constant:
\begin{equation}
    D(\delta) \coloneqq \sum_{\comp \in \Comps} \exp\left( h_{\comp} (\delta) \right) = \exp \left( \frac{\delta}{2} \right) + \exp \left( - \frac{\delta}{2} \right) + \nu,
\end{equation}
where $\nu \coloneqq \exp(\alpha)$, as in the statement.

First, we consider the case in which $\alpha$ is finite, \ie $\alpha \in \mathbb{R}$. By the positivity of the exponential, it follows directly that $\nu > 0$. Applying Equation~\eqref{eq:mobt_model} with the scores above, the outcome probabilities are:
\begin{align}
    f(\prelsinv | \delta) &= \exp \left( \frac{\delta}{2} \right) \Big/ D(\delta), \\
    f(\prels | \delta) &= \exp \left( - \frac{\delta}{2} \right) \Big/ D(\delta), \\
    f(\equivrel | \delta) &= \exp (\alpha) / D(\delta) = \nu / D(\delta), \\
    f(\incomprel | \delta) &= 0.
\end{align}
which corresponds to the Davidson model \citep{davidson1970extending}.

It remains to show that these expressions coincide with the formulation of \citet{davidson1970extending}. Denoting $\pi_{\tau} \coloneqq \exp(u(\tau))$ the \emph{worth} of trajectory $\tau$, and observing that $\sqrt{\pi_{\tau} \pi_{\tau'}} = \exp ( (u(\tau) + u(\tau'))/2 )$, the definition of $\delta$ yields:
\begin{equation}
    \exp \left( \frac{\delta}{2} \right) \sqrt{\pi_{\tau} \pi_{\tau'}} = \exp \left( \frac{u(\tau) - u(\tau')}{2} + \frac{u(\tau) + u(\tau')}{2} \right) = \exp \left( u(\tau) \right) = \pi_{\tau},
\end{equation}
and, symmetrically, $\exp ( -\delta/2 ) \sqrt{\pi_{\tau} \pi_{\tau'}} = \pi_{\tau'}$. Hence, multiplying numerator and denominator by $\sqrt{\pi_{\tau} \pi_{\tau'}} > 0$, the normalization becomes $D(\delta) \sqrt{\pi_{\tau} \pi_{\tau'}} = \pi_{\tau} + \pi_{\tau'} + \nu \sqrt{\pi_{\tau} \pi_{\tau'}}$, and we recover:
\begin{align}
    f(\prelsinv | \delta) &= \frac{\pi_{\tau}}{\pi_{\tau} + \pi_{\tau'} + \nu \sqrt{\pi_{\tau} \pi_{\tau'}}}, \\
    f(\prels | \delta) &= \frac{\pi_{\tau'}}{\pi_{\tau} + \pi_{\tau'} + \nu \sqrt{\pi_{\tau} \pi_{\tau'}}}, \\
    f(\equivrel | \delta) &= \frac{\nu \sqrt{\pi_{\tau} \pi_{\tau'}}}{\pi_{\tau} + \pi_{\tau'} + \nu \sqrt{\pi_{\tau} \pi_{\tau'}}},
\end{align}
which corresponds to the model formulation in Equation~(2.6) of \citet{davidson1970extending}. Notice that both the numerators and the denominator above are positively homogeneous of degree one in $(\pi_{\tau}, \pi_{\tau'})$, so that the three ratios are homogeneous of degree zero, \ie invariant to the common rescaling $(\pi_{\tau}, \pi_{\tau'}) \mapsto (c \pi_{\tau}, c \pi_{\tau'})$ for every $c > 0$, equivalently to the common shift $u \mapsto u + \log c$. Thus, they depend on the trajectory pair only through $\delta$, consistently with Definition~\ref{defi:mobt}.

Let us now consider the case in which $\alpha = -\infty$.
By the same convention, it follows that $\nu = 0$. Thus:
\begin{align}
    f(\prelsinv | \delta) &= \frac{\exp \left( \frac{1}{2} \delta \right)}{\exp \left( \frac{1}{2} \delta \right) + \exp \left( -\frac{1}{2} \delta \right)} = \frac{\exp \left(\delta \right)}{\exp \left(\delta \right) + 1} = \sigma(\delta), \\
    f(\prels | \delta) &= \frac{\exp \left( -\frac{1}{2} \delta \right)}{\exp \left( \frac{1}{2} \delta \right) + \exp \left( -\frac{1}{2} \delta \right)} = \frac{\exp \left(-\delta \right)}{\exp \left(-\delta \right) + 1} = \sigma(-\delta) = 1 - \sigma(\delta), \\
    f(\equivrel | \delta) &= 0, \\
    f(\incomprel | \delta) &= 0,
\end{align}
which corresponds exactly to the standard Bradley-Terry model, concluding the proof.
\end{proof}

\samplecomplexity*
\begin{proof}
We prove that, with probability at least $1 - \delta$, it holds that:
\begin{equation}
\label{eq:samplecomplexity:explicit}
    \dkl{\mathcal{P}_{\bm{\theta}^{*}}}{\mathcal{P}_{\widehat{\bm{\theta}}}} \le 2 \varsigma \sqrt{\frac{(dk+3) \log \left( 1 + 4 R L \sqrt{(dk+3) N} \right) + \log(4/\delta)}{N}},
\end{equation}
where $\varsigma$ and $L$ are defined in Lemmas~\ref{lem:loss_bounded} and~\ref{lem:loss_lipschitz}, respectively, which implies Equation~\eqref{eq:sample_complexity}. Indeed, since $R, \Lambda \ge 1$ and $d \ge 2$, we have that $d R \Lambda \sqrt{k} \ge 2 > \log 4$, and thus $\varsigma \le 2 d R \Lambda \sqrt{k}$. Moreover, both $\log \left( 1 + 4 R L \sqrt{(dk+3) N} \right)$ and $\log(4/\delta)$ are at least $1$, since $R \ge 1$, $L \ge 2$, $dk + 3 \ge 5$, and $\delta < 1$. Since $a + b \le 2ab$ for every $a, b \ge 1$, the numerator of the fraction in Equation~\eqref{eq:samplecomplexity:explicit} is at most $2 (dk+3) \log \left( 1 + 4 R L \sqrt{(dk+3) N} \right) \log(4/\delta) \le 5 dk \log \left( 1 + 4 R L \sqrt{(dk+3) N} \right) \log(4/\delta)$, where the last inequality holds since $dk \ge 2$. Hence, the right-hand side of Equation~\eqref{eq:samplecomplexity:explicit} is at most:
\begin{equation}
\nonumber
    4 \sqrt{5} \, R \Lambda d^{3/2} k \sqrt{\frac{\log \left( 1 + 4 R L \sqrt{(dk+3) N} \right) \log(4/\delta)}{N}},
\end{equation}
where the logarithmic factor $\log \left( 1 + 4 R L \sqrt{(dk+3) N} \right)$ does not depend on $\delta$ and is omitted by the $\widetilde{\mathcal{O}}(\cdot)$ notation. We now prove Equation~\eqref{eq:samplecomplexity:explicit}.

Throughout the proof, we use the following auxiliary notation: we denote as $n \coloneqq dk+3$ the total number of learnable parameters, as $\mathfrak{B} \coloneqq [-R, R]^{dk} \times (0,R] \times [-R,0] \times (\underline{\gamma}_d,R]$ the bounding box of the parameter space, i.e., $\Theta \subseteq \mathfrak{B}$, and as $\ell(\comp | \tau, \tau'; \bm{\theta}) \coloneqq - \log f (\comp | \tau, \tau'; \bm{\theta})$ the negative log-likelihood for every $\bm{\theta} \in \mathfrak{B}$.
We also recall the \emph{population counterpart} of the empirical NLL of Equation~\eqref{eq:nll_theta} as:
\begin{equation}
\nonumber
    \mathcal{L}(\bm{\theta}) \coloneqq \mathbb{E}_{(\tau, \tau', \comp) \sim \mathcal{P}_{\bm{\theta}^{*}}} \left[ \ell(\comp | \tau, \tau'; \bm{\theta}) \right],
    \qquad \forall \bm{\theta} \in \mathfrak{B},
\end{equation}
which is well-defined and finite, as follows from Lemma~\ref{lem:loss_bounded}.

We structure the proof in three logical steps: first, we relate the excess population loss to the KL divergence, second, we bound the KL divergence with the uniform deviation between empirical and population losses, and third, we discretize the parameter space and apply a concentration argument.

Let $\bm{\theta} \in \mathfrak{B}$. Since, under Assumptions~\ref{asm:features} and~\ref{asm:bounded}, the scores are finite, it follows that $f(\comp | \tau, \tau'; \bm{\theta}) > 0$ for every $\tau, \tau' \in \Ts$ and $\comp \in \Comps$.
Observing that $\mathcal{P}_{\bm{\theta}^*} \in \Delta(\Ts \times \Ts \times \Comps)$ can be interpreted as first drawing a trajectory pair $(\tau, \tau')\sim \mathcal{Q}$ and then a comparison label $\comp \sim f(\cdot | \tau, \tau'; \bm{\theta}^*)$, and since $\Comps$ is a finite set, the law of total expectation yields, for every $\bm{\theta}' \in \mathfrak{B}$:
\begin{equation}
\nonumber
    \mathcal{L}(\bm{\theta}') = \mathbb{E}_{(\tau, \tau') \sim \mathcal{Q}} \left[ \sum_{\comp \in \Comps} f(\comp | \tau, \tau'; \bm{\theta}^{*}) \, \ell(\comp | \tau, \tau'; \bm{\theta}') \right].
\end{equation}
Applying this identity to $\bm{\theta}$ and $\bm{\theta}^*$ and taking the difference of the two finite expectations yields:
\begin{align}
    \mathcal{L}(\bm{\theta}) - \mathcal{L}(\bm{\theta}^{*})
    & = \mathbb{E}_{(\tau, \tau') \sim \mathcal{Q}} \left[ \sum_{\comp \in \Comps} f(\comp | \tau, \tau'; \bm{\theta}^{*}) \log \frac{f(\comp | \tau, \tau'; \bm{\theta}^{*})}{f(\comp | \tau, \tau'; \bm{\theta})} \right] \nonumber \\
    & = \mathbb{E}_{(\tau, \tau') \sim \mathcal{Q}} \left[ \dkl{f(\cdot | \tau, \tau'; \bm{\theta}^{*})}{f(\cdot | \tau, \tau'; \bm{\theta})} \right]
    = \dkl{\mathcal{P}_{\bm{\theta}^{*}}}{\mathcal{P}_{\bm{\theta}}}, \label{eq:samplecomplexity:kl}
\end{align}
where the last equality follows from the chain rule of the KL divergence \citep{cover2006elements}, observing that $\mathcal{P}_{\bm{\theta}^{*}}$ and $\mathcal{P}_{\bm{\theta}}$ share the same marginal $\mathcal{Q}$ over trajectory pairs.

We now define an $\varepsilon$-cover of $\Theta$ in $L_\infty$-norm. Let $\varepsilon > 0$. For a generic bounded interval $I \subset \mathbb{R}$ with endpoints $a < b$, whether open, closed, or half-open, we define the grid:
\begin{equation}
\label{eq:samplecomplexity:grid}
    \mathcal{G}_I \coloneqq \left\{ a + \frac{(2j-1)(b-a)}{2 m_I}: j \in \dsb{m_I} \right\} \subseteq (a, b) \subseteq I,
\end{equation}
where $m_I \coloneqq \max \{1, \lceil (b-a)/(2\varepsilon) \rceil\}$, which contains points such that for every $z \in I$ there exists a $z' \in \mathcal{G}_I$ such that $|z - z'|\le \varepsilon$, since its points are the midpoints of $m_I$ consecutive intervals of width $(b-a)/m_I \le 2\varepsilon$, thus covering $[a,b]$.
Moreover, $m_I \le 1 + (b-a)/(2\varepsilon)$, since $1 \le 1 + (b-a)/(2\varepsilon)$ and $\lceil x \rceil < x + 1$ for every $x \in \mathbb{R}$.
We then define the $\varepsilon$-cover of $\Theta$ in terms of grids defined in Equation~\eqref{eq:samplecomplexity:grid} as:
\begin{equation}
\label{eq:samplecomplexity:epscover}
    \Theta_{\varepsilon} \coloneqq \mathcal{G}_{[-R,R]}^{dk} \times \mathcal{G}_{(0,R]} \times \mathcal{G}_{[-R,0]} \times \mathcal{G}_{[\underline{\gamma}_d,R]} \subseteq \mathfrak{B}.
\end{equation}
Hence, for every $\bm{\theta} \in \Theta$, there exists a $\widetilde{\bm{\theta}} \in \Theta_{\varepsilon}$ such that $\|\bm{\theta} - \widetilde{\bm{\theta}}\|_{\infty} \le \varepsilon$, and thus $\|\bm{\theta} - \widetilde{\bm{\theta}}\|_2 \le \sqrt{n} \varepsilon$.
Thus, since $R \ge 1$, the grids in Equation~\eqref{eq:samplecomplexity:epscover} all have width of at most $2R$. Each corresponding $m_I$ is at most $1 + R/\varepsilon$, and since $\Theta_\varepsilon$ is their Cartesian product, it holds that \citep{lattimore2020bandit}:
\begin{equation}
\label{eq:samplecomplexity:covercard}
    \left| \Theta_{\varepsilon} \right| \le \left( 1 + \frac{R}{\varepsilon} \right)^n.
\end{equation}

By Equation~\eqref{eq:samplecomplexity:kl} we can derive that, with probability at least $1-\delta$:
\begin{align}
    \dkl{\mathcal{P}_{\bm{\theta}^{*}}}{\mathcal{P}_{\widehat{\bm{\theta}}}} &= \loss(\widehat{\bm{\theta}}) - \loss(\bm{\theta}^{*}) \\
    &\le \loss(\widehat{\bm{\theta}}) - \losshat(\widehat{\bm{\theta}}|\mathcal{D}) + \losshat(\bm{\theta}^* | \mathcal{D}) - \loss(\bm{\theta}^*) \\
    &\le 2 \sup_{\bm{\theta} \in \Theta} \left| \losshat(\bm{\theta}|\mathcal{D}) - \loss(\bm{\theta}) \right| \label{eq:samplecomplexity:ub:1} \\
    &\le 2 \sup_{\bm{\theta}\in\Theta} \inf_{\bm{\theta}' \in \Theta_{\varepsilon}} \left| \losshat(\bm{\theta}|\mathcal{D}) - \loss(\bm{\theta}) - \losshat(\bm{\theta}'|\mathcal{D}) + \loss(\bm{\theta}') \right| + 2 \sup_{\bm{\theta}' \in \Theta_{\varepsilon}} \left| \loss(\bm{\theta}') - \losshat(\bm{\theta}'|\mathcal{D}) \right| \label{eq:samplecomplexity:ub:2}\\
    &\le 4 L \sup_{\bm{\theta}\in\Theta} \inf_{\bm{\theta}' \in \Theta_{\varepsilon}} \left\| \bm{\theta} - \bm{\theta}' \right\|_{2} + 2 \sup_{\bm{\theta}' \in \Theta_{\varepsilon}}  \left| \loss(\bm{\theta}') - \losshat(\bm{\theta}'|\mathcal{D}) \right| \label{eq:samplecomplexity:ub:3} \\
    &\le 4 L \sqrt{n} \varepsilon + 2 \varsigma \sqrt{\frac{\log \left( 2 \left| \Theta_{\varepsilon} \right| / \delta \right)}{2N}}, \label{eq:samplecomplexity:ub:4}
\end{align}
where Equation~\eqref{eq:samplecomplexity:ub:1} follows by adding and subtracting $\losshat(\widehat{\bm{\theta}} | \mathcal{D})$, observing that $\widehat{\mathcal{L}}(\widehat{\bm{\theta}}|\mathcal{D}) \le \widehat{\mathcal{L}}(\bm{\theta}^* | \mathcal{D})$ since $\widehat{\bm{\theta}} \in \argmin_{\bm{\theta} \in \Theta} \widehat{\mathcal{L}} (\bm{\theta} | \mathcal{D})$ (for an $\eta$-approximate minimizer, this inequality holds up to an additive $\eta$, which carries over unchanged to the final bound) and upper bounding with the $\sup$ since $\widehat{\bm{\theta}}, \bm{\theta}^* \in \Theta$, Equation~\eqref{eq:samplecomplexity:ub:2} follows from the triangle inequality, Equation~\eqref{eq:samplecomplexity:ub:3} is obtained by observing that both $\losshat(\cdot | \mathcal{D})$ and $\loss(\cdot)$ are $L$-Lipschitz continuous in $\mathfrak{B}$ w.r.t.\ the $L_2$-norm as demonstrated in Lemma~\ref{lem:loss_lipschitz}, and Equation~\eqref{eq:samplecomplexity:ub:4} follows since, by Lemma~\ref{lem:loss_bounded}, $\losshat$ is the empirical mean of $N$ i.i.d.\ bounded random variables that lie in $[0, \varsigma]$ and by Hoeffding's inequality \citep{hoeffding1963probability}, taking the union bound over $\Theta_{\varepsilon}$.

Finally, combining Equation~\eqref{eq:samplecomplexity:covercard} and Equation~\eqref{eq:samplecomplexity:ub:4} we get that, with probability at least $1-\delta$:
\begin{equation}
\nonumber
    \dkl{\mathcal{P}_{\bm{\theta}^{*}}}{\mathcal{P}_{\widehat{\bm{\theta}}}} \le 4 L \sqrt{n} \varepsilon + \varsigma \sqrt{\frac{2 \left( n \log \left( 1 + R / \varepsilon \right) + \log(2/\delta) \right)}{N}}.
\end{equation}
We now select $\varepsilon \coloneqq \varsigma / (4 L \sqrt{nN})$, so that $4 L \sqrt{n} \varepsilon = \varsigma / \sqrt{N}$. Moreover, since $\varsigma \ge \log 4 > 1$, we have that $R / \varepsilon = 4 R L \sqrt{nN} / \varsigma \le 4 R L \sqrt{nN}$, hence $\log ( 1 + R / \varepsilon ) \le \log ( 1 + 4 R L \sqrt{nN} )$. Letting $X \coloneqq n \log ( 1 + 4 R L \sqrt{nN} ) + \log(2/\delta) \ge 0$, we obtain:
\begin{align}
    \dkl{\mathcal{P}_{\bm{\theta}^{*}}}{\mathcal{P}_{\widehat{\bm{\theta}}}}
    &\le \frac{\varsigma}{\sqrt{N}} \left( 1 + \sqrt{2X} \right) \nonumber \\
    &\le \frac{\varsigma}{\sqrt{N}} \sqrt{2 + 4X} \nonumber \\
    &\le \frac{2\varsigma}{\sqrt{N}} \sqrt{X + \log 2} \nonumber \\
    &= 2 \varsigma \sqrt{\frac{(dk+3) \log \left( 1 + 4 R L \sqrt{(dk+3) N} \right) + \log(4/\delta)}{N}}, \nonumber
\end{align}
where the second inequality follows from $(a + b)^2 \le 2a^2 + 2b^2$, and the third one from $2 \le 4 \log 2$. This is Equation~\eqref{eq:samplecomplexity:explicit}.

Finally, if $\mathcal{D}$ is sampled from a distribution $\mathcal{P}^{*}$ with marginal $\mathcal{Q}$ that falls outside the MOBT class, defining $\loss$ as the expected NLL under $\mathcal{P}^{*}$, the argument of Equation~\eqref{eq:samplecomplexity:kl} yields $\loss(\bm{\theta}) - \loss(\bm{\theta}') = \dkl{\mathcal{P}^{*}}{\mathcal{P}_{\bm{\theta}}} - \dkl{\mathcal{P}^{*}}{\mathcal{P}_{\bm{\theta}'}}$ for every $\bm{\theta}, \bm{\theta}' \in \mathfrak{B}$. Repeating the proof with $\bm{\theta}^{*}$ replaced by any $\bm{\theta}' \in \Theta$, which only uses that the losses are bounded and Lipschitz, and taking the infimum over $\bm{\theta}'$, we obtain that $\dkl{\mathcal{P}^{*}}{\mathcal{P}_{\widehat{\bm{\theta}}}} - \inf_{\bm{\theta}' \in \Theta} \dkl{\mathcal{P}^{*}}{\mathcal{P}_{\bm{\theta}'}}$ is bounded by the right-hand side of Equation~\eqref{eq:samplecomplexity:explicit}. This concludes the proof.
\end{proof}

\samplecomplexitylocal*
\begin{proof}
We prove that, with probability at least $1 - \delta$, it holds that:
\begin{equation}
\label{eq:samplecomplexitylocal:explicit}
\begin{aligned}
    \dkl{\mathcal{P}_{\bm{\theta}^{*}}}{\mathcal{P}_{\widehat{\bm{\theta}}_{\mathrm{loc}}}} \le \; & 2 \varsigma \sqrt{\frac{(dk+3) \log \left( 1 + 4 R L \sqrt{(dk+3) N} \right) + \log(8/\delta)}{N}} \\
    & + 2 \gamma^{*} R \Lambda \sqrt{dk} \left( \mathcal{P}_{\bm{\theta}^{*}}(\incomprel) + \sqrt{\frac{\log(2/\delta)}{2N}} \right),
\end{aligned}
\end{equation}
with $\varsigma$ and $L$ as in Equation~\eqref{eq:samplecomplexity:explicit}, which implies Equation~\eqref{eq:sample_complexity_local}. Indeed, since $\log(8/\delta) \le 2 \log(4/\delta)$ for every $\delta \in (0,1)$, the first term on the right-hand side of Equation~\eqref{eq:samplecomplexitylocal:explicit} is at most $\sqrt{2}$ times the right-hand side of Equation~\eqref{eq:samplecomplexity:explicit}, which is bounded as in the proof of Theorem~\ref{thr:sample_complexity}.
We now prove Equation~\eqref{eq:samplecomplexitylocal:explicit}.

Throughout the proof, we use the notation of the proof of Theorem~\ref{thr:sample_complexity}, \ie $\mathfrak{B} \coloneqq [-R, R]^{dk} \times (\underline{\gamma}_d,R] \times [-R,0] \times (0,R]$, $\ell(\comp | \tau, \tau'; \bm{\theta}) \coloneqq - \log f (\comp | \tau, \tau'; \bm{\theta})$, and $\mathcal{L}(\bm{\theta}) \coloneqq \mathbb{E}_{(\tau, \tau', \comp) \sim \mathcal{P}_{\bm{\theta}^{*}}} [ \ell(\comp | \tau, \tau'; \bm{\theta}) ]$ for every $\bm{\theta} \in \mathfrak{B}$, and we observe that $\Theta_{\gamma_0} \subseteq \mathfrak{B}$, since $\gamma_0 \in (\underline{\gamma}_d,R]$.
Moreover, for every $i \in \dsb{N}$, we denote $\bm{x}_i \coloneqq \bm{\phi}(\tau_i) - \bm{\phi}(\tau_i')$. We denote the set of incomparability-labeled samples as $\mathcal{I}_{\incomprel} \coloneqq \{ i \in \dsb{N} : \comp_i = \incomprel \}$ and its relative frequency as $\widehat{\mathcal{P}}(\incomprel) \coloneqq |\mathcal{I}_{\incomprel}| / N$. Finally, we denote as $\bm{\Pi} = \bm{\Pi}_d$ the orthogonal projection onto the orthogonal complement of the standard diagonal defined in Section~\ref{sec:intro}, so that $\widetilde{h}_{\incomprel}(\bm{\delta}) = \| \bm{\Pi} \bm{\delta} \|_2$ for every $\bm{\delta} \in \mathbb{R}^d$ by Equation~\eqref{eq:derivation:dispersion}. We recall that $\bm{\Pi}$ is symmetric and satisfies $\bm{\Pi}^2 = \bm{\Pi}$, $\bm{\Pi} \bm{1}_d = \bm{0}_d$, and $\bm{1}_d^\top \bm{\Pi} = \bm{0}_d^\top$.

We structure the proof in four logical steps: first, we construct a parametrization in $\Theta_{\gamma_0}$ that is equivalent to $\bm{\theta}^{*}$, second, we relate the KL divergence to the empirical excess loss, third, we bound the latter by decomposing the empirical NLL into a convex and a concave part, and fourth, we apply concentration arguments.

\paragraph{Equivalent Parametrization.}
Following Remark~\ref{rem:scale_invariance}, let $c \coloneqq \gamma^{*} / \gamma_0$, and observe that $c \in (0, 1]$, since $0 < \gamma^{*} \le \gamma_0$. We define:
\begin{equation}
\nonumber
    \bm{M}^{*}_{\gamma_0} \coloneqq T_c(\bm{M}^{*}) = \bm{A}_d \bm{M}^{*} + c \, \bm{\Pi} \bm{M}^{*}
    \qquad \text{and} \qquad
    \bm{\theta}^{*}_{\gamma_0} \coloneqq \left( \mathrm{vec}(\bm{M}^{*}_{\gamma_0}), \alpha^{*}, \beta^{*}, \gamma_0 \right).
\end{equation}
Since $\gamma_0 = \gamma^{*} / c$, Remark~\ref{rem:scale_invariance} ensures that $\bm{\theta}^{*}_{\gamma_0}$ and $\bm{\theta}^{*}$ induce the same MOBT model (see also Lemma~\ref{lem:identifiability}, with $\bm{U} = \bm{I}_d$), and, in particular, that $\bm{\Pi} \bm{M}^{*}_{\gamma_0} = c \, \bm{\Pi} \bm{M}^{*}$.
Therefore, $f(\cdot | \tau, \tau'; \bm{\theta}^{*}_{\gamma_0}) = f(\cdot | \tau, \tau'; \bm{\theta}^{*})$ for every $\tau, \tau' \in \Ts$, which implies that $\mathcal{L}(\bm{\theta}^{*}_{\gamma_0}) = \mathcal{L}(\bm{\theta}^{*})$ and $\losshat(\bm{\theta}^{*}_{\gamma_0} | \mathcal{D}) = \losshat(\bm{\theta}^{*} | \mathcal{D})$.
Moreover, $\bm{\theta}^{*}_{\gamma_0} \in \Theta_{\gamma_0}$. Indeed, for every $i \in \dsb{d}$ and $j \in \dsb{k}$, the corresponding element of $\bm{M}^{*}_{\gamma_0}$ is $(1 - c) \bar{m}_j + c M^{*}_{ij}$, where $\bar{m}_j \coloneqq \frac{1}{d} \sum_{l \in \dsb{d}} M^{*}_{lj}$, \ie a convex combination of two elements of $[-R, R]$, since $c \in (0, 1]$, while $(\alpha^{*}, \beta^{*}) \in (0, R] \times [-R, 0]$ by Assumption~\ref{asm:bounded}.

\paragraph{From KL Divergence to Empirical Excess Loss.}
Equation~\eqref{eq:samplecomplexity:kl} holds for every $\bm{\theta} \in \mathfrak{B}$, and, in particular, for $\widehat{\bm{\theta}}_{\mathrm{loc}} \in \Theta_{\gamma_0} \subseteq \mathfrak{B}$. Thus, recalling that $\mathcal{L}(\bm{\theta}^{*}) = \mathcal{L}(\bm{\theta}^{*}_{\gamma_0})$, we have:
\begin{align}
    \dkl{\mathcal{P}_{\bm{\theta}^{*}}}{\mathcal{P}_{\widehat{\bm{\theta}}_{\mathrm{loc}}}}
    & = \loss(\widehat{\bm{\theta}}_{\mathrm{loc}}) - \loss(\bm{\theta}^{*}_{\gamma_0}) \nonumber \\
    & = \loss(\widehat{\bm{\theta}}_{\mathrm{loc}}) - \losshat(\widehat{\bm{\theta}}_{\mathrm{loc}} | \mathcal{D}) + \losshat(\bm{\theta}^{*}_{\gamma_0} | \mathcal{D}) - \loss(\bm{\theta}^{*}_{\gamma_0}) + \losshat(\widehat{\bm{\theta}}_{\mathrm{loc}} | \mathcal{D}) - \losshat(\bm{\theta}^{*}_{\gamma_0} | \mathcal{D}) \nonumber \\
    & \le 2 \sup_{\bm{\theta} \in \mathfrak{B}} \left| \losshat(\bm{\theta} | \mathcal{D}) - \loss(\bm{\theta}) \right| + \losshat(\widehat{\bm{\theta}}_{\mathrm{loc}} | \mathcal{D}) - \losshat(\bm{\theta}^{*}_{\gamma_0} | \mathcal{D}), \label{eq:samplecomplexitylocal:decomposition}
\end{align}
where the inequality holds since $\widehat{\bm{\theta}}_{\mathrm{loc}}, \bm{\theta}^{*}_{\gamma_0} \in \mathfrak{B}$.

\paragraph{Convex-plus-Concave Decomposition.}
Since $\gamma = \gamma_0$ is fixed in $\Theta_{\gamma_0}$, we identify $\Theta_{\gamma_0}$ with the convex set $[-R, R]^{dk} \times (0, R] \times [-R, 0] \subseteq \mathbb{R}^{dk+2}$ of the parameters $(\mathrm{vec}(\bm{M}), \alpha, \beta)$. For every $\bm{\theta} \in \Theta_{\gamma_0}$ and $i \in \dsb{N}$, let $\bm{h}_i(\bm{\theta}) \coloneqq (h_{\comp}(\bm{M} \bm{x}_i))_{\comp \in \Comps}$ be the vector of the scores of Definition~\ref{defi:mobt} computed with parameters $(\gamma_0, \alpha, \beta)$, and let $\mathrm{LSE}(\bm{v}) \coloneqq \log \sum_{\comp \in \Comps} \exp(v_{\comp})$ for every $\bm{v} \in \mathbb{R}^{\Comps}$, as in the proof of Lemma~\ref{lem:loss_lipschitz}. Since $\ell(\comp_i | \tau_i, \tau_i'; \bm{\theta}) = \mathrm{LSE}(\bm{h}_i(\bm{\theta})) - h_{\comp_i}(\bm{M} \bm{x}_i)$ and $h_{\incomprel}(\bm{M} \bm{x}_i) = \gamma_0 \| \bm{\Pi} \bm{M} \bm{x}_i \|_2 + \beta$, we can decompose the empirical NLL of Equation~\eqref{eq:nll_theta} as $\losshat(\bm{\theta} | \mathcal{D}) = \losshat_{\mathrm{cvx}}(\bm{\theta}) + \losshat_{\mathrm{ccv}}(\bm{\theta})$, where:
\begin{equation}
\nonumber
    \losshat_{\mathrm{cvx}}(\bm{\theta}) \coloneqq \frac{1}{N} \sum_{i \in \dsb{N}} \mathrm{LSE}(\bm{h}_i(\bm{\theta})) - \frac{1}{N} \sum_{i \in \dsb{N} \setminus \mathcal{I}_{\incomprel}} h_{\comp_i}(\bm{M} \bm{x}_i) - \widehat{\mathcal{P}}(\incomprel) \, \beta
    \qquad \text{and} \qquad
    \losshat_{\mathrm{ccv}}(\bm{\theta}) \coloneqq - \frac{\gamma_0}{N} \sum_{i \in \mathcal{I}_{\incomprel}} \| \bm{\Pi} \bm{M} \bm{x}_i \|_2.
\end{equation}
The function $\losshat_{\mathrm{cvx}}$ is convex in $\Theta_{\gamma_0}$. Indeed, for every $i \in \dsb{N}$, the scores $h_{\prelsinv}(\bm{M} \bm{x}_i)$, $h_{\prels}(\bm{M} \bm{x}_i)$, and $h_{\equivrel}(\bm{M} \bm{x}_i) = \alpha$ are affine in $\bm{\theta}$, while $h_{\incomprel}(\bm{M} \bm{x}_i)$ is convex in $\bm{\theta}$, being the sum of an affine term and of the composition of the Euclidean norm with a linear map, scaled by $\gamma_0 > 0$. Since $\mathrm{LSE}$ is convex and non-decreasing in each argument, $\bm{\theta} \mapsto \mathrm{LSE}(\bm{h}_i(\bm{\theta}))$ is convex \citep[][Section 3.2.4]{boyd2004convex}. Finally, the remaining terms of $\losshat_{\mathrm{cvx}}$ are affine in $\bm{\theta}$, since $\comp_i \neq \incomprel$ in the second sum.

We now construct a supergradient of $\losshat_{\mathrm{ccv}}$ at $\widehat{\bm{\theta}}_{\mathrm{loc}} = (\mathrm{vec}(\widehat{\bm{M}}_{\mathrm{loc}}), \widehat{\alpha}_{\mathrm{loc}}, \widehat{\beta}_{\mathrm{loc}}, \gamma_0)$. For every $i \in \mathcal{I}_{\incomprel}$, let:
\begin{equation}
\nonumber
    \bm{v}_i \coloneqq \begin{cases}
        \bm{\Pi} \widehat{\bm{M}}_{\mathrm{loc}} \bm{x}_i / \| \bm{\Pi} \widehat{\bm{M}}_{\mathrm{loc}} \bm{x}_i \|_2 & \text{if } \bm{\Pi} \widehat{\bm{M}}_{\mathrm{loc}} \bm{x}_i \neq \bm{0}_d, \\
        \bm{0}_d & \text{otherwise},
    \end{cases}
\end{equation}
so that $\|\bm{v}_i\|_2 \le 1$, $\bm{\Pi} \bm{v}_i = \bm{v}_i$, and $\bm{v}_i^\top \bm{\Pi} \widehat{\bm{M}}_{\mathrm{loc}} \bm{x}_i = \| \bm{\Pi} \widehat{\bm{M}}_{\mathrm{loc}} \bm{x}_i \|_2$. Let $\bm{s} \in \mathbb{R}^{dk+2}$ be the vector whose block corresponding to $\mathrm{vec}(\bm{M})$ is $\mathrm{vec}(\bm{S})$, with $\bm{S} \coloneqq - (\gamma_0 / N) \sum_{i \in \mathcal{I}_{\incomprel}} \bm{v}_i \bm{x}_i^\top$, and whose entries corresponding to $\alpha$ and $\beta$ are zero. Then, for every $\bm{\theta} \in \Theta_{\gamma_0}$, we have that $\bm{s}^\top (\bm{\theta} - \widehat{\bm{\theta}}_{\mathrm{loc}}) = \mathrm{tr}(\bm{S}^\top (\bm{M} - \widehat{\bm{M}}_{\mathrm{loc}})) = - (\gamma_0 / N) \sum_{i \in \mathcal{I}_{\incomprel}} \bm{v}_i^\top (\bm{M} - \widehat{\bm{M}}_{\mathrm{loc}}) \bm{x}_i$, and thus:
\begin{align}
    \losshat_{\mathrm{ccv}}(\widehat{\bm{\theta}}_{\mathrm{loc}}) + \bm{s}^\top (\bm{\theta} - \widehat{\bm{\theta}}_{\mathrm{loc}})
    & = - \frac{\gamma_0}{N} \sum_{i \in \mathcal{I}_{\incomprel}} \left( \| \bm{\Pi} \widehat{\bm{M}}_{\mathrm{loc}} \bm{x}_i \|_2 + \bm{v}_i^\top \bm{\Pi} \bm{M} \bm{x}_i - \bm{v}_i^\top \bm{\Pi} \widehat{\bm{M}}_{\mathrm{loc}} \bm{x}_i \right) \nonumber \\
    & = - \frac{\gamma_0}{N} \sum_{i \in \mathcal{I}_{\incomprel}} \bm{v}_i^\top \bm{\Pi} \bm{M} \bm{x}_i \label{eq:samplecomplexitylocal:linearization} \\
    & \ge - \frac{\gamma_0}{N} \sum_{i \in \mathcal{I}_{\incomprel}} \| \bm{\Pi} \bm{M} \bm{x}_i \|_2 = \losshat_{\mathrm{ccv}}(\bm{\theta}), \nonumber
\end{align}
where the first equality follows from $\bm{v}_i^\top \bm{Z} \bm{x}_i = (\bm{\Pi} \bm{v}_i)^\top \bm{Z} \bm{x}_i = \bm{v}_i^\top \bm{\Pi} \bm{Z} \bm{x}_i$ for every $\bm{Z} \in \mathbb{R}^{d \times k}$, since $\bm{\Pi}$ is symmetric, the second one from the definition of $\bm{v}_i$, and the inequality from the Cauchy-Schwarz inequality and $\|\bm{v}_i\|_2 \le 1$. Hence, $\bm{s}$ is a \emph{supergradient} \citep{rockafellar1970convex} of $\losshat_{\mathrm{ccv}}$ at $\widehat{\bm{\theta}}_{\mathrm{loc}}$ over $\Theta_{\gamma_0}$, \ie $\losshat_{\mathrm{ccv}}(\bm{\theta}) \le \losshat_{\mathrm{ccv}}(\widehat{\bm{\theta}}_{\mathrm{loc}}) + \bm{s}^\top (\bm{\theta} - \widehat{\bm{\theta}}_{\mathrm{loc}})$ for every $\bm{\theta} \in \Theta_{\gamma_0}$.

We now linearize the concave part around $\widehat{\bm{\theta}}_{\mathrm{loc}}$, as in DC programming \citep{lethi2018dc}. Assume $\bm{\theta}^{*}_{\gamma_0} \neq \widehat{\bm{\theta}}_{\mathrm{loc}}$, since otherwise the left-hand side of Equation~\eqref{eq:samplecomplexitylocal:gap} below is zero. Since $\widehat{\bm{\theta}}_{\mathrm{loc}}$ is a local minimum of $\losshat(\cdot | \mathcal{D}) = \losshat_{\mathrm{cvx}} + \losshat_{\mathrm{ccv}}$ over $\Theta_{\gamma_0}$, there exists $r > 0$ such that $\losshat(\widehat{\bm{\theta}}_{\mathrm{loc}} | \mathcal{D}) \le \losshat(\bm{\theta} | \mathcal{D})$ for every $\bm{\theta} \in \Theta_{\gamma_0}$ with $\| \bm{\theta} - \widehat{\bm{\theta}}_{\mathrm{loc}} \|_2 < r$. For every $t \in (0, \min\{1, r / \| \bm{\theta}^{*}_{\gamma_0} - \widehat{\bm{\theta}}_{\mathrm{loc}} \|_2\})$, let $\bm{\theta}_t \coloneqq \widehat{\bm{\theta}}_{\mathrm{loc}} + t (\bm{\theta}^{*}_{\gamma_0} - \widehat{\bm{\theta}}_{\mathrm{loc}})$, which belongs to $\Theta_{\gamma_0}$, since $\Theta_{\gamma_0}$ is convex and $\bm{\theta}^{*}_{\gamma_0} \in \Theta_{\gamma_0}$, and satisfies $\| \bm{\theta}_t - \widehat{\bm{\theta}}_{\mathrm{loc}} \|_2 < r$. Hence:
\begin{equation}
\nonumber
    \losshat(\widehat{\bm{\theta}}_{\mathrm{loc}} | \mathcal{D}) \le \losshat_{\mathrm{cvx}}(\bm{\theta}_t) + \losshat_{\mathrm{ccv}}(\bm{\theta}_t) \le (1 - t) \losshat_{\mathrm{cvx}}(\widehat{\bm{\theta}}_{\mathrm{loc}}) + t \losshat_{\mathrm{cvx}}(\bm{\theta}^{*}_{\gamma_0}) + \losshat_{\mathrm{ccv}}(\widehat{\bm{\theta}}_{\mathrm{loc}}) + t \, \bm{s}^\top (\bm{\theta}^{*}_{\gamma_0} - \widehat{\bm{\theta}}_{\mathrm{loc}}),
\end{equation}
where the first inequality follows from the local optimality of $\widehat{\bm{\theta}}_{\mathrm{loc}}$, and the second one from the convexity of $\losshat_{\mathrm{cvx}}$ and from the supergradient inequality with $\bm{\theta} = \bm{\theta}_t$. Subtracting $\losshat(\widehat{\bm{\theta}}_{\mathrm{loc}} | \mathcal{D}) = \losshat_{\mathrm{cvx}}(\widehat{\bm{\theta}}_{\mathrm{loc}}) + \losshat_{\mathrm{ccv}}(\widehat{\bm{\theta}}_{\mathrm{loc}})$ from both sides, dividing by $t > 0$, and adding $\losshat_{\mathrm{ccv}}(\widehat{\bm{\theta}}_{\mathrm{loc}}) - \losshat_{\mathrm{ccv}}(\bm{\theta}^{*}_{\gamma_0})$ to both sides, we get $\losshat(\widehat{\bm{\theta}}_{\mathrm{loc}} | \mathcal{D}) - \losshat(\bm{\theta}^{*}_{\gamma_0} | \mathcal{D}) \le \losshat_{\mathrm{ccv}}(\widehat{\bm{\theta}}_{\mathrm{loc}}) + \bm{s}^\top (\bm{\theta}^{*}_{\gamma_0} - \widehat{\bm{\theta}}_{\mathrm{loc}}) - \losshat_{\mathrm{ccv}}(\bm{\theta}^{*}_{\gamma_0})$. Evaluating the right-hand side via Equation~\eqref{eq:samplecomplexitylocal:linearization} with $\bm{\theta} = \bm{\theta}^{*}_{\gamma_0}$, we obtain:
\begin{align}
    \losshat(\widehat{\bm{\theta}}_{\mathrm{loc}} | \mathcal{D}) - \losshat(\bm{\theta}^{*}_{\gamma_0} | \mathcal{D})
    & \le \frac{\gamma_0}{N} \sum_{i \in \mathcal{I}_{\incomprel}} \left( \| \bm{\Pi} \bm{M}^{*}_{\gamma_0} \bm{x}_i \|_2 - \bm{v}_i^\top \bm{\Pi} \bm{M}^{*}_{\gamma_0} \bm{x}_i \right) \nonumber \\
    & \le \frac{2 \gamma_0}{N} \sum_{i \in \mathcal{I}_{\incomprel}} \| \bm{\Pi} \bm{M}^{*}_{\gamma_0} \bm{x}_i \|_2
    = \frac{2 \gamma^{*}}{N} \sum_{i \in \mathcal{I}_{\incomprel}} \| \bm{\Pi} \bm{M}^{*} \bm{x}_i \|_2
    \le 2 \gamma^{*} R \Lambda \sqrt{dk} \, \widehat{\mathcal{P}}(\incomprel), \label{eq:samplecomplexitylocal:gap}
\end{align}
where the second inequality follows from the Cauchy-Schwarz inequality and $\|\bm{v}_i\|_2 \le 1$, the equality from $\gamma_0 \bm{\Pi} \bm{M}^{*}_{\gamma_0} = c \gamma_0 \bm{\Pi} \bm{M}^{*} = \gamma^{*} \bm{\Pi} \bm{M}^{*}$, and the last inequality from Equations~\eqref{eq:lossbounded:delta} and~\eqref{eq:lossbounded:dispersion}, since $\| \bm{\Pi} \bm{M}^{*} \bm{x}_i \|_2 = \widetilde{h}_{\incomprel}(\bm{M}^{*} \bm{x}_i) \le \| \bm{M}^{*} \bm{x}_i \|_2 \le R \Lambda \sqrt{dk}$ for every $i \in \dsb{N}$.

\paragraph{Concentration.}
We observe that the derivation from the right-hand side of Equation~\eqref{eq:samplecomplexity:ub:1} to the end of the proof of Theorem~\ref{thr:sample_complexity}, including the choice of $\varepsilon$, does not use the optimality of $\widehat{\bm{\theta}}$. Moreover, it holds verbatim when replacing $\Theta$ with $\mathfrak{B}$, since the set $\Theta_{\varepsilon}$ of Equation~\eqref{eq:samplecomplexity:epscover} is an $\varepsilon$-cover of $\mathfrak{B}$ in $L_\infty$-norm, and Lemmas~\ref{lem:loss_bounded} and~\ref{lem:loss_lipschitz} hold in $\mathfrak{B}$. Thus, applying it with confidence $\delta / 2$, we have that, with probability at least $1 - \delta / 2$:
\begin{equation}
\label{eq:samplecomplexitylocal:uniform}
    2 \sup_{\bm{\theta} \in \mathfrak{B}} \left| \losshat(\bm{\theta} | \mathcal{D}) - \loss(\bm{\theta}) \right| \le 2 \varsigma \sqrt{\frac{(dk+3) \log \left( 1 + 4 R L \sqrt{(dk+3) N} \right) + \log(8/\delta)}{N}}.
\end{equation}
Furthermore, $\widehat{\mathcal{P}}(\incomprel) = \frac{1}{N} \sum_{i \in \dsb{N}} \Ind\{\comp_i = \incomprel\}$ is the empirical mean of $N$ i.i.d.\ Bernoulli random variables with mean $\mathbb{E}_{(\tau, \tau', \comp) \sim \mathcal{P}_{\bm{\theta}^{*}}} [\Ind\{\comp = \incomprel\}] = \mathcal{P}_{\bm{\theta}^{*}}(\incomprel)$. Thus, by the one-sided Hoeffding's inequality \citep{hoeffding1963probability}, with probability at least $1 - \delta / 2$:
\begin{equation}
\label{eq:samplecomplexitylocal:hoeffding}
    \widehat{\mathcal{P}}(\incomprel) \le \mathcal{P}_{\bm{\theta}^{*}}(\incomprel) + \sqrt{\frac{\log(2/\delta)}{2N}}.
\end{equation}
By a union bound, Equations~\eqref{eq:samplecomplexitylocal:uniform} and~\eqref{eq:samplecomplexitylocal:hoeffding} hold simultaneously with probability at least $1 - \delta$. Plugging them, together with Equation~\eqref{eq:samplecomplexitylocal:gap}, into Equation~\eqref{eq:samplecomplexitylocal:decomposition} yields Equation~\eqref{eq:samplecomplexitylocal:explicit}, which concludes the proof.
\end{proof}

\identifiabilitymain*
\begin{proof}
By Lemma~\ref{lem:identifiability}, applied with $\bm{\theta}^{*}$ in place of $\bm{\theta}$ and $\bm{\theta}$ in place of $\bm{\theta}'$, we have that $\mathcal{P}_{\bm{\theta}} = \mathcal{P}_{\bm{\theta}^{*}}$ if and only if $\alpha = \alpha^{*}$, $\beta = \beta^{*}$, and there exist $\bm{U} \in \mathcal{G}_d$ and $c' > 0$ such that $\bm{M} = T_{c'}(\bm{U} \bm{M}^{*})$ and $\gamma_0 = \gamma^{*} / c'$, \ie $c' = c$. Thus, $\mathcal{P}_{\bm{\theta}} = \mathcal{P}_{\bm{\theta}^{*}}$ if and only if $\alpha = \alpha^{*}$, $\beta = \beta^{*}$, and $\bm{M} \in \{ T_c(\bm{U} \bm{M}^{*}) : \bm{U} \in \mathcal{G}_d \}$. In this case, as shown in the proof of Lemma~\ref{lem:identifiability}, $\bm{1}_d^\top \bm{M} = \bm{1}_d^\top \bm{M}^{*}$, which yields the identifiability of the aggregate utility. It remains to characterize $\mathcal{G}_d$ for $d \in \{1, 2, 3\}$.

We first observe that every $\bm{U} \in \mathcal{G}_d$ maps $\mathcal{V} \coloneqq \{ \bm{v} \in \mathbb{R}^d : \bm{1}_d^\top \bm{v} = 0 \}$ into itself, since $\bm{1}_d^\top \bm{U} \bm{v} = (\bm{U}^\top \bm{1}_d)^\top \bm{v} = \bm{1}_d^\top \bm{v} = 0$ for every $\bm{v} \in \mathcal{V}$, and that its restriction to $\mathcal{V}$ is an isometry.

\paragraph{Case $d = 1$.}
The only $1 \times 1$ matrix $\bm{U}$ satisfying $\bm{U} \bm{1}_1 = \bm{1}_1$ is $\bm{U} = 1$, hence $\mathcal{G}_1 = \{1\}$. Moreover, $\bm{\Pi} = 1 - 1 = 0$, so that $T_c(\bm{M}^{*}) = \bm{M}^{*}$, and thus $\bm{M} = \bm{M}^{*}$.

\paragraph{Case $d = 2$.}
In this case, $\mathcal{V}$ is spanned by $\bm{e}_1 - \bm{e}_2$, and the only isometries of a one-dimensional space are $\pm$ the identity. Hence, every $\bm{U} \in \mathcal{G}_2$ satisfies $\bm{U} \bm{1}_2 = \bm{1}_2$ and $\bm{U} (\bm{e}_1 - \bm{e}_2) = \pm (\bm{e}_1 - \bm{e}_2)$, which determine $\bm{U}$ uniquely, since $\{\bm{1}_2, (\bm{e}_1 - \bm{e}_2)\}$ is a basis of $\mathbb{R}^2$. The sign $+$ corresponds to $\bm{U} = \bm{I}_2$, and the sign $-$ corresponds to $\bm{U} = \bm{P}$, where $\bm{P} \coloneqq \bm{1}_2 \bm{1}_2^\top - \bm{I}_2$ swaps the two objectives, since $\bm{P} \bm{1}_2 = \bm{1}_2$ and $\bm{P} (\bm{e}_1 - \bm{e}_2) = -(\bm{e}_1 - \bm{e}_2)$. Conversely, both $\bm{I}_2$ and $\bm{P}$ are orthogonal and fix $\bm{1}_2$, hence $\mathcal{G}_2 = \{ \bm{I}_2, \bm{P} \}$. If $\gamma_0 = \gamma^{*}$, then $c = 1$ and $T_1$ is the identity, so that $\bm{M} \in \{ \bm{M}^{*}, \bm{P} \bm{M}^{*} \}$, where $\bm{P} \bm{M}^{*}$ is obtained by swapping the rows of $\bm{M}^{*}$.

\paragraph{Case $d = 3$.}
Let $\bm{V} \in \mathbb{R}^{3 \times 2}$ be any matrix whose columns form an orthonormal basis of the orthogonal complement of $\bm{1}_3$, so that $\bm{V}^\top \bm{V} = \bm{I}_2$, $\bm{V}^\top \bm{1}_3 = \bm{0}_2$, and $\bm{V} \bm{V}^\top = \bm{\Pi}$, the latter being the orthogonal projection onto $\mathcal{V}$. Let $\bm{U} \in \mathcal{G}_3$ and $\bm{O} \coloneqq \bm{V}^\top \bm{U} \bm{V}$. Since $\bm{I}_3 = \bm{A}_3 + \bm{\Pi}$, $\bm{U} \bm{1}_3 = \bm{1}_3$, $\bm{1}_3^\top \bm{U} = \bm{1}_3^\top$, and $\bm{1}_3^\top \bm{\Pi} = \bm{0}_3^\top$, we have:
\begin{equation}
\nonumber
    \bm{U} = \left( \bm{A}_3 + \bm{\Pi} \right) \bm{U} \left( \bm{A}_3 + \bm{\Pi} \right) = \bm{A}_3 + \bm{\Pi} \bm{U} \bm{\Pi} = \bm{A}_3 + \bm{V} \bm{O} \bm{V}^\top.
\end{equation}
Moreover, by Equation~\eqref{eq:identifiability:commute}, $\bm{O}^\top \bm{O} = \bm{V}^\top \bm{U}^\top \bm{\Pi} \bm{U} \bm{V} = \bm{V}^\top \bm{U}^\top \bm{U} \bm{\Pi} \bm{V} = \bm{V}^\top \bm{V} = \bm{I}_2$, where we used $\bm{\Pi} \bm{V} = \bm{V}$. Conversely, for every $\bm{O} \in \mathbb{R}^{2 \times 2}$ with $\bm{O}^\top \bm{O} = \bm{I}_2$, the matrix $\bm{U} \coloneqq \bm{A}_3 + \bm{V} \bm{O} \bm{V}^\top$ satisfies $\bm{U} \bm{1}_3 = \bm{1}_3$ and $\bm{U}^\top \bm{U} = \bm{A}_3 + \bm{V} \bm{O}^\top \bm{O} \bm{V}^\top = \bm{A}_3 + \bm{\Pi} = \bm{I}_3$, since the cross terms vanish because $\bm{1}_3^\top \bm{V} = \bm{0}_2^\top$. This proves that $\mathcal{G}_3 = \{ \bm{A}_3 + \bm{V} \bm{O} \bm{V}^\top : \bm{O} \in \mathbb{R}^{2 \times 2}, \, \bm{O}^\top \bm{O} = \bm{I}_2 \}$.

Finally, assume $\bm{\Pi} \bm{M}^{*} \neq \bm{0}$. For every $\vartheta \in [0, 2\pi)$, let $\bm{O}_{\vartheta}$ be the counterclockwise rotation of angle $\vartheta$ in $\mathbb{R}^2$, and let $\bm{U}_{\vartheta} \coloneqq \bm{A}_3 + \bm{V} \bm{O}_{\vartheta} \bm{V}^\top \in \mathcal{G}_3$. Let $\vartheta, \vartheta' \in [0, 2\pi)$ be such that $T_c(\bm{U}_{\vartheta} \bm{M}^{*}) = T_c(\bm{U}_{\vartheta'} \bm{M}^{*})$. Multiplying both sides by $\bm{\Pi}$ on the left and using Equation~\eqref{eq:identifiability:commute}, we get $c \, \bm{U}_{\vartheta} \bm{\Pi} \bm{M}^{*} = c \, \bm{U}_{\vartheta'} \bm{\Pi} \bm{M}^{*}$. Since $\bm{1}_3^\top \bm{\Pi} = \bm{0}_3^\top$, this reads $\bm{V} \bm{O}_{\vartheta} \bm{Y} = \bm{V} \bm{O}_{\vartheta'} \bm{Y}$, where $\bm{Y} \coloneqq \bm{V}^\top \bm{\Pi} \bm{M}^{*} \in \mathbb{R}^{2 \times k}$ satisfies $\bm{V} \bm{Y} = \bm{\Pi} \bm{M}^{*} \neq \bm{0}$, and thus $\bm{Y} \neq \bm{0}$. Multiplying by $\bm{O}_{\vartheta'}^\top \bm{V}^\top$ on the left, we get $\bm{O}_{\vartheta - \vartheta'} \bm{Y} = \bm{Y}$, \ie the rotation of angle $\vartheta - \vartheta'$ fixes a non-zero column of $\bm{Y}$, which is possible only if $\vartheta = \vartheta'$. Hence, the map $\vartheta \mapsto T_c(\bm{U}_{\vartheta} \bm{M}^{*})$ is injective on $[0, 2\pi)$, and the set $\{ T_c(\bm{U} \bm{M}^{*}) : \bm{U} \in \mathcal{G}_3 \}$ is uncountable for every $c > 0$, including $c = 1$, \ie $\gamma_0 = \gamma^{*}$.
\end{proof}

\paretoscalarization*
\begin{proof}
We recall from the proof of Lemma~\ref{lem:identifiability} that $\bm{\Pi}$ is symmetric and that every $\bm{U} \in \mathcal{G}_d$ satisfies $\bm{U}^\top \bm{1}_d = \bm{1}_d$. Moreover, $\bm{U}^\top \in \mathcal{G}_d$, since $\bm{U} \bm{U}^\top = \bm{I}_d$ for every square orthogonal matrix and $\bm{U}^\top \bm{1}_d = \bm{1}_d$. Finally, $\bm{\Pi} \bm{w} = \bm{w} - \bm{1}_d (\bm{1}_d^\top \bm{w}) / d = \bm{w} - \bm{1}_d / d$, since $\bm{1}_d^\top \bm{w} = 1$. Then, by the definition of $T_c$, we have:
\begin{align}
    \bm{w}^\top T_c(\bm{U} \bm{M})
    & = \frac{1}{d} (\bm{w}^\top \bm{1}_d) \bm{1}_d^\top \bm{U} \bm{M} + c \, \bm{w}^\top \bm{\Pi} \bm{U} \bm{M}
    = \frac{1}{d} (\bm{U}^\top \bm{1}_d)^\top \bm{M} + c \, (\bm{\Pi} \bm{w})^\top \bm{U} \bm{M} \nonumber \\
    & = \left( \frac{1}{d} \bm{1}_d + c \, \bm{U}^\top \left( \bm{w} - \frac{1}{d} \bm{1}_d \right) \right)^\top \bm{M}
    = \widetilde{\bm{w}}^\top \bm{M}, \nonumber
\end{align}
where the second equality follows from $\bm{w}^\top \bm{1}_d = 1$ and the symmetry of $\bm{\Pi}$, and the third one from $\bm{U}^\top \bm{1}_d = \bm{1}_d$ and $\bm{\Pi} \bm{w} = \bm{w} - \bm{1}_d / d$. Moreover, $\bm{1}_d^\top \widetilde{\bm{w}} = 1 + c \, (\bm{U} \bm{1}_d)^\top (\bm{w} - \bm{1}_d / d) = 1 + c \, (\bm{1}_d^\top \bm{w} - 1) = 1$, and, since $\bm{U}^\top$ is orthogonal, $\| \widetilde{\bm{w}} - \bm{1}_d / d \|_2 = c \| \bm{U}^\top (\bm{w} - \bm{1}_d / d) \|_2 = c \| \bm{w} - \bm{1}_d / d \|_2$. However, $\widetilde{\bm{w}}$ may have negative elements, as in the proof of Proposition~\ref{prop:pareto_tight}.

To prove the second part of the statement, let $r \ge 0$ be such that $\max\{1, c\} \, r \le \rho_{\mathrm{in}}$, \ie $r \le \rho_{\mathrm{in}}$ and $c r \le \rho_{\mathrm{in}}$. By result ($v$) of Lemma~\ref{lem:orthant}, $\mathcal{W}_r$ and $\mathcal{W}_{c r}$ coincide with the sets of vectors $\bm{w} \in \mathbb{R}^d$ with $\bm{1}_d^\top \bm{w} = 1$ at distance at most $r$ and $c r$ from $\bm{1}_d / d$, respectively. Hence, by the first part of the statement, the map $\bm{w} \mapsto \widetilde{\bm{w}}$ sends $\mathcal{W}_r$ into $\mathcal{W}_{c r}$. Moreover, it is invertible, with inverse $\widetilde{\bm{w}} \mapsto \bm{1}_d / d + c^{-1} \bm{U} (\widetilde{\bm{w}} - \bm{1}_d / d)$, which, by the same computations with $\bm{U}^\top \in \mathcal{G}_d$ in place of $\bm{U}$ and $c^{-1}$ in place of $c$, sends $\mathcal{W}_{c r}$ into $\mathcal{W}_r$. Hence, $\bm{w} \mapsto \widetilde{\bm{w}}$ is a bijection between $\mathcal{W}_r$ and $\mathcal{W}_{c r}$. Since, for every $\bm{w} \in \mathcal{W}_r$, the functions $\pi \mapsto \bm{w}^\top T_c(\bm{U} \bm{M}) \bm{\psi}(\pi)$ and $\pi \mapsto \widetilde{\bm{w}}^\top \bm{M} \bm{\psi}(\pi)$ coincide, so do their sets of maximizers over $\Pi$, and taking the union over $\bm{w} \in \mathcal{W}_r$ yields $\Pi^{*}(T_c(\bm{U} \bm{M}); \mathcal{W}_r) = \Pi^{*}(\bm{M}; \mathcal{W}_{c r})$.
\end{proof}

The condition $\max\{1, c\} \, r \le \rho_{\mathrm{in}}$ ensures that $\widetilde{\bm{w}}$ is a scalarization weight in $\Delta_{d-1}$ whenever $\bm{w} \in \mathcal{W}_r$, and it cannot be dropped in general. Indeed, for $d = 3$ and $c \in (1/2, 1]$, Proposition~\ref{prop:pareto_tight} shows that $\Pi^{*}(T_c(\bm{U} \bm{M}^{*}); \Delta_2) = \Pi^{*}(T_c(\bm{U} \bm{M}^{*}); \mathcal{W}_{\sqrt{2/3}})$ may contain a policy that is not weakly Pareto optimal for $\bm{M}^{*}$, whereas every policy in $\Pi^{*}(\bm{M}^{*}; \mathcal{W}_{c \sqrt{2/3}})$ is weakly Pareto optimal for $\bm{M}^{*}$, by claim ($b$) in the proof of Theorem~\ref{thr:pareto} with $\widetilde{\eta} = 0$.

\paretoguarantees*
\begin{proof}
We first prove an auxiliary claim. Let $\widetilde{\bm{w}} \in \mathbb{R}^d$ be such that $\bm{1}_d^\top \widetilde{\bm{w}} = 1$, and let $\widehat{\pi} \in \argmax_{\pi \in \Pi} \widetilde{\bm{w}}^\top \bm{M}^{*} \bm{\psi}(\pi)$, so that $\widetilde{\bm{w}}^\top \bm{M}^{*} (\bm{\psi}(\pi') - \bm{\psi}(\widehat{\pi})) \le 0$ for every $\pi' \in \Pi$. We stress that $\widetilde{\bm{w}}$ is not required to belong to $\Delta_{d-1}$, \ie it is not necessarily a scalarization weight: it is an auxiliary vector of coefficients, which we will use to express the scalarized objective of $\widehat{\bm{M}}$ as a linear combination of the true objectives $\bm{M}^{*} \bm{\psi}(\pi)$. Then:
\begin{enumerate}[noitemsep, topsep=-1pt, label=(\alph*)]
    \item if all the elements of $\widetilde{\bm{w}}$ are strictly positive, then $\widehat{\pi}$ is Pareto optimal for $\bm{M}^{*}$;
    \item if $\widetilde{\bm{w}}$ has at most one negative element, and the absolute value of such an element is at most $\widetilde{\eta} \ge 0$, then $\widehat{\pi}$ is $\widetilde{\varepsilon}$-Pareto optimal for $\bm{M}^{*}$, with $\widetilde{\varepsilon} \coloneqq \widetilde{\eta} B / (1 + \widetilde{\eta})$.
\end{enumerate}
Claim ($a$) is the classical sufficient condition for Pareto optimality of the solutions of weighted-sum scalarizations with strictly positive weights \citep[][Theorem 3.1.2]{miettinen1999nonlinear}, applied to the minimization of the objectives $-\bm{M}^{*} \bm{\psi}(\pi)$ over $\pi \in \Pi$ with weights $\widetilde{\bm{w}}$, which in this case are strictly positive and sum to one, \ie $\widetilde{\bm{w}} \in \Delta_{d-1}$.
To prove claim ($b$), assume by contradiction that there exists $\pi' \in \Pi$ such that $\bm{g} \coloneqq \bm{M}^{*} (\bm{\psi}(\pi') - \bm{\psi}(\widehat{\pi})) > \widetilde{\varepsilon} \bm{1}_d$ element-wise, and observe that $g_i \le B$ for every $i \in \dsb{d}$ by definition of $B$. If $\widetilde{\bm{w}} \in \mathbb{R}^d_{\ge 0}$, then $\widetilde{\bm{w}}^\top \bm{g} \ge (\bm{1}_d^\top \widetilde{\bm{w}}) \min_{i \in \dsb{d}} g_i > \widetilde{\varepsilon} \ge 0$, which contradicts the optimality of $\widehat{\pi}$. Otherwise, let $j \in \dsb{d}$ be the unique index such that $\widetilde{w}_j < 0$, so that $\sum_{i \neq j} \widetilde{w}_i = 1 + |\widetilde{w}_j|$ and $|\widetilde{w}_j| \le \widetilde{\eta}$. Then:
\begin{equation}
\nonumber
    \widetilde{\bm{w}}^\top \bm{g}
    \ge (1 + |\widetilde{w}_j|) \min_{i \neq j} g_i - |\widetilde{w}_j| g_j
    > (1 + |\widetilde{w}_j|) \widetilde{\varepsilon} - |\widetilde{w}_j| B
    = \widetilde{\varepsilon} - |\widetilde{w}_j| (B - \widetilde{\varepsilon})
    \ge \widetilde{\varepsilon} - \widetilde{\eta} (B - \widetilde{\varepsilon})
    = 0,
\end{equation}
where the third inequality follows from $|\widetilde{w}_j| \le \widetilde{\eta}$ and $B - \widetilde{\varepsilon} = B / (1 + \widetilde{\eta}) \ge 0$, and the last equality from the definition of $\widetilde{\varepsilon}$. This contradicts the optimality of $\widehat{\pi}$, and proves the claim.

We now prove the claims of the statement. Let $r \ge 0$ and $\pi \in \Pi^{*}(\widehat{\bm{M}}; \mathcal{W}_r)$, so that $\pi \in \argmax_{\pi'' \in \Pi} \bm{w}^\top \widehat{\bm{M}} \bm{\psi}(\pi'')$ for some $\bm{w} \in \mathcal{W}_r$, and let $\widetilde{\bm{w}} \coloneqq \bm{1}_d / d + c \, \bm{U}^\top (\bm{w} - \bm{1}_d / d)$. By the first part of Proposition~\ref{prop:scalarization}, we have $\bm{w}^\top \widehat{\bm{M}} = \widetilde{\bm{w}}^\top \bm{M}^{*}$, so that $\pi \in \argmax_{\pi'' \in \Pi} \widetilde{\bm{w}}^\top \bm{M}^{*} \bm{\psi}(\pi'')$, \ie the auxiliary claim applies with $\widehat{\pi} = \pi$. Moreover, $\bm{1}_d^\top \widetilde{\bm{w}} = 1$ and $\| \bm{\Pi} \widetilde{\bm{w}} \|_2 = \| \widetilde{\bm{w}} - \bm{1}_d / d \|_2 = c \| \bm{w} - \bm{1}_d / d \|_2 \le c r$. We remark that $\widetilde{\bm{w}}$ need not belong to $\Delta_{d-1}$, and that whether it does depends on $r$, as we show below.

Optimality under the aggregate reward follows from Proposition~\ref{prop:scalarization} with $r = 0$, since $\mathcal{W}_0 = \{ \bm{1}_d / d \}$.

To prove Pareto optimality, let $r \in [0, \rho_{\mathrm{in}})$. Since $c \le 1$, we have $\| \bm{\Pi} \widetilde{\bm{w}} \|_2 \le c r < \rho_{\mathrm{in}} = \rho_{\mathrm{in}} \bm{1}_d^\top \widetilde{\bm{w}}$, and result ($iii$) of Lemma~\ref{lem:orthant} yields that all the elements of $\widetilde{\bm{w}}$ are strictly positive, and claim ($a$) yields that $\pi$ is Pareto optimal. Finally, if $d = 2$, by result ($v$) of Lemma~\ref{lem:orthant}, $\Delta_1 = \mathcal{W}_{\rho_{\mathrm{in}}}$, so that, by Proposition~\ref{prop:scalarization} with $r = \rho_{\mathrm{in}}$, which applies since $c \le 1$, we have $\Pi^{*}(\widehat{\bm{M}}; \Delta_1) = \Pi^{*}(\bm{M}^{*}; \mathcal{W}_{c \rho_{\mathrm{in}}}) \subseteq \Pi^{*}(\bm{M}^{*}; \mathcal{W}_{\rho_{\mathrm{in}}}) = \Pi^{*}(\bm{M}^{*}; \Delta_1)$, where the inclusion follows from $\mathcal{W}_{c \rho_{\mathrm{in}}} \subseteq \mathcal{W}_{\rho_{\mathrm{in}}}$, and holds with equality if $c = 1$.

To prove the last claim, let $d = 3$ and $r \ge \rho_{\mathrm{in}}$. Since $\bm{w} \in \Delta_2$, result ($ii$) of Lemma~\ref{lem:orthant} yields $\| \bm{w} - \bm{1}_3 / 3 \|_2 = \| \bm{\Pi} \bm{w} \|_2 \le \sqrt{2/3}$, and thus, since $c \le 1$, we have $\| \bm{\Pi} \widetilde{\bm{w}} \|_2 \le c \sqrt{2/3} \le \sqrt{2/3} \, \bm{1}_3^\top \widetilde{\bm{w}}$, and result ($iv$) of Lemma~\ref{lem:orthant} yields that $\widetilde{\bm{w}}$ has at most one negative element. Moreover, by result ($i$) of Lemma~\ref{lem:orthant}, $\widetilde{w}_i \ge 1/3 - \sqrt{2/3} \, c r$ for every $i \in \dsb{3}$, so that the absolute value of the negative element, if any, is at most $\eta$. Hence, claim ($b$) with $\widetilde{\eta} = \eta$ yields that $\pi$ is $(\eta B / (1 + \eta))$-Pareto optimal, which concludes the proof. 

\end{proof}

\section{Additional Theoretical Results}
\label{apx:additional}

In this appendix, we collect additional theoretical results that complement Section~\ref{sec:model_based}. First, we show that, for $d = 2$, a global optimum of the empirical NLL can be computed by solving a finite number of convex programs (Appendix~\ref{apx:global_2d}). Then, we show that, for $d = 3$, the guarantee of Theorem~\ref{thr:pareto} on the whole simplex is tight (Appendix~\ref{apx:pareto_tight}). Finally, we show how the finite-sample guarantees of Theorem~\ref{thr:sample_complexity_local} transfer to the Pareto front (Appendix~\ref{apx:finite_downstream}).

\subsection{Global Optimization for $d = 2$}
\label{apx:global_2d}

As discussed in Section~\ref{sec:model_based}, the empirical NLL of Equation~\eqref{eq:nll_theta} is, in general, non-convex, and Theorem~\ref{thr:sample_complexity_local} only provides guarantees for local minima, at the price of an additive term that does not vanish with $N$. We now show that, for $d = 2$, the non-convexity has an explicit combinatorial structure, which allows computing a global minimum over $\Theta_{\gamma_0}$ by solving a finite number of convex programs.

Throughout this section, we use the notation of the proof of Theorem~\ref{thr:sample_complexity_local}. In particular, $\bm{x}_i \coloneqq \bm{\phi}(\tau_i) - \bm{\phi}(\tau_i')$ for every $i \in \dsb{N}$, $\mathcal{I}_{\incomprel}$ is the set of incomparability-labeled samples, $\Theta_{\gamma_0}$ is identified with the convex set of the parameters $(\mathrm{vec}(\bm{M}), \alpha, \beta)$, and $\losshat(\cdot | \mathcal{D}) = \losshat_{\mathrm{cvx}} + \losshat_{\mathrm{ccv}}$, where $\losshat_{\mathrm{cvx}}$ is convex on $\Theta_{\gamma_0}$ and $\losshat_{\mathrm{ccv}}(\bm{\theta}) = - (\gamma_0 / N) \sum_{i \in \mathcal{I}_{\incomprel}} \| \bm{\Pi} \bm{M} \bm{x}_i \|_2$. Let $d = 2$ and let $\bm{z}(\bm{M}) \coloneqq \bm{M}^\top (\bm{e}_1 - \bm{e}_2) \in \mathbb{R}^k$, \ie the difference between the two rows of $\bm{M}$. Since $\bm{\Pi} = \frac{1}{2} (\bm{e}_1 - \bm{e}_2) (\bm{e}_1 - \bm{e}_2)^\top$, we have $\bm{\Pi} \bm{M} \bm{x} = \frac{1}{2} \bm{z}(\bm{M})^\top \bm{x} \, (\bm{e}_1 - \bm{e}_2)$, and thus $\| \bm{\Pi} \bm{M} \bm{x} \|_2 = | \bm{z}(\bm{M})^\top \bm{x} | / \sqrt{2}$ for every $\bm{x} \in \mathbb{R}^k$. Moreover, the samples with $\bm{x}_i = \bm{0}_k$ do not contribute to $\losshat_{\mathrm{ccv}}$. Hence, letting $\mathcal{I}_{\incomprel}^{+} \coloneqq \{ i \in \mathcal{I}_{\incomprel} : \bm{x}_i \neq \bm{0}_k \}$, we have:
\begin{equation}
\label{eq:global2d:concave}
    \losshat_{\mathrm{ccv}}(\bm{\theta}) = - \frac{\gamma_0}{\sqrt{2} N} \sum_{i \in \mathcal{I}_{\incomprel}^{+}} | \bm{z}(\bm{M})^\top \bm{x}_i |.
\end{equation}
For every sign vector $\bm{s} \in \{-1, 1\}^{\mathcal{I}_{\incomprel}^{+}}$, we define the function obtained by replacing each absolute value in Equation~\eqref{eq:global2d:concave} with a signed linear term:
\begin{equation}
\label{eq:global2d:ls}
    \losshat_{\bm{s}}(\bm{\theta}) \coloneqq \losshat_{\mathrm{cvx}}(\bm{\theta}) - \frac{\gamma_0}{\sqrt{2} N} \sum_{i \in \mathcal{I}_{\incomprel}^{+}} s_i \, \bm{z}(\bm{M})^\top \bm{x}_i.
\end{equation}
Finally, we consider the central arrangement of the hyperplanes $\{ \bm{z} \in \mathbb{R}^k : \bm{z}^\top \bm{x}_i = 0 \}$, with $i \in \mathcal{I}_{\incomprel}^{+}$, whose \emph{regions} are the connected components of the complement of their union in $\mathbb{R}^k$. Each region $\mathcal{C}$ is associated with the sign vector $\bm{s} \in \{-1, 1\}^{\mathcal{I}_{\incomprel}^{+}}$ such that $s_i \bm{z}^\top \bm{x}_i > 0$ for every $\bm{z} \in \mathcal{C}$ and $i \in \mathcal{I}_{\incomprel}^{+}$, and we denote as $\mathcal{S}$ the set of the sign vectors of all the regions.

\begin{prop}[Global Optimization for $d = 2$]
\label{prop:global_2d}
    Let $d = 2$, let $\gamma_0 \in \Gamma_2^{+}$, and let $\mathcal{D} = \{(\tau_i, \tau_i', \comp_i)\}_{i \in \dsb{N}}$ be a dataset with $m \coloneqq |\mathcal{I}_{\incomprel}^{+}| \ge 1$. Under Assumption~\ref{asm:features}, the following hold:
    \begin{enumerate}[label=(\roman*), noitemsep, topsep=-1pt]
        \item for every $\bm{s} \in \{-1, 1\}^{\mathcal{I}_{\incomprel}^{+}}$, $\losshat_{\bm{s}}$ is convex on $\Theta_{\gamma_0}$, and $\losshat(\bm{\theta} | \mathcal{D}) = \min_{\bm{s} \in \mathcal{S}} \losshat_{\bm{s}}(\bm{\theta})$ for every $\bm{\theta} \in \Theta_{\gamma_0}$;
        \item $\inf_{\bm{\theta} \in \Theta_{\gamma_0}} \losshat(\bm{\theta} | \mathcal{D}) = \min_{\bm{s} \in \mathcal{S}} \inf_{\bm{\theta} \in \Theta_{\gamma_0}} \losshat_{\bm{s}}(\bm{\theta})$, and, if $\bm{s}^{\star} \in \mathcal{S}$ attains the minimum on the right-hand side, every minimizer of $\losshat_{\bm{s}^{\star}}$ over $\Theta_{\gamma_0}$ is a global minimizer of $\losshat(\cdot | \mathcal{D})$ over $\Theta_{\gamma_0}$;
        \item $\mathcal{S} = - \mathcal{S}$ and $\inf_{\bm{\theta} \in \Theta_{\gamma_0}} \losshat_{-\bm{s}}(\bm{\theta}) = \inf_{\bm{\theta} \in \Theta_{\gamma_0}} \losshat_{\bm{s}}(\bm{\theta})$ for every $\bm{s} \in \mathcal{S}$, so that it suffices to solve $|\mathcal{S}| / 2$ convex programs, where:
        \begin{equation}
        \nonumber
            |\mathcal{S}| \le 2 \sum_{j = 0}^{k-1} \binom{m-1}{j} = \mathcal{O} \left( m^{k-1} \right).
        \end{equation}
    \end{enumerate}
\end{prop}

\begin{proof}
We first prove result ($i$). For every $\bm{s} \in \{-1, 1\}^{\mathcal{I}_{\incomprel}^{+}}$, $\losshat_{\bm{s}}$ is the sum of $\losshat_{\mathrm{cvx}}$, which is convex on $\Theta_{\gamma_0}$, as shown in the proof of Theorem~\ref{thr:sample_complexity_local}, and of a function that is linear in $\bm{\theta}$, since $\bm{z}(\bm{M})$ is linear in $\bm{M}$. Hence, $\losshat_{\bm{s}}$ is convex on $\Theta_{\gamma_0}$. Moreover, since $|a| \ge s a$ for every $a \in \mathbb{R}$ and $s \in \{-1, 1\}$, Equations~\eqref{eq:global2d:concave} and~\eqref{eq:global2d:ls} yield $\losshat_{\bm{s}}(\bm{\theta}) \ge \losshat(\bm{\theta} | \mathcal{D})$ for every $\bm{\theta} \in \Theta_{\gamma_0}$ and every sign vector $\bm{s}$. Conversely, let $\bm{\theta} \in \Theta_{\gamma_0}$, and let $\bm{z} \coloneqq \bm{z}(\bm{M})$. Since $\bm{x}_i \neq \bm{0}_k$ for every $i \in \mathcal{I}_{\incomprel}^{+}$, the union of the hyperplanes of the arrangement is a finite union of hyperplanes, and has empty interior. Thus, the union of the regions, which are finitely many, is dense in $\mathbb{R}^k$, and $\bm{z}$ belongs to the closure of some region $\mathcal{C}$, with sign vector $\bm{s} \in \mathcal{S}$. By continuity, $s_i \bm{z}^\top \bm{x}_i \ge 0$, \ie $s_i \bm{z}^\top \bm{x}_i = | \bm{z}^\top \bm{x}_i |$, for every $i \in \mathcal{I}_{\incomprel}^{+}$, and thus $\losshat_{\bm{s}}(\bm{\theta}) = \losshat(\bm{\theta} | \mathcal{D})$, which proves result ($i$).

Result ($ii$) follows from result ($i$) by exchanging the infimum over $\bm{\theta} \in \Theta_{\gamma_0}$ and the minimum over the finite set $\mathcal{S}$. Moreover, if $\bm{s}^{\star}$ attains the minimum, and $\widehat{\bm{\theta}}$ minimizes $\losshat_{\bm{s}^{\star}}$ over $\Theta_{\gamma_0}$, then, by result ($i$), $\losshat(\widehat{\bm{\theta}} | \mathcal{D}) \le \losshat_{\bm{s}^{\star}}(\widehat{\bm{\theta}}) = \inf_{\bm{\theta} \in \Theta_{\gamma_0}} \losshat(\bm{\theta} | \mathcal{D})$, \ie $\widehat{\bm{\theta}}$ is a global minimizer of $\losshat(\cdot | \mathcal{D})$ over $\Theta_{\gamma_0}$.

To prove result ($iii$), let $\bm{P} \coloneqq \bm{1}_2 \bm{1}_2^\top - \bm{I}_2$ be the matrix that swaps the two objectives, as in the proof of Theorem~\ref{thr:identifiability_main}, and, for every $\bm{\theta} = (\mathrm{vec}(\bm{M}), \alpha, \beta) \in \Theta_{\gamma_0}$, let $\bm{\theta}^{\bm{P}} \coloneqq (\mathrm{vec}(\bm{P} \bm{M}), \alpha, \beta)$. The map $\bm{\theta} \mapsto \bm{\theta}^{\bm{P}}$ is a bijection of $\Theta_{\gamma_0}$ onto itself, since it permutes the elements of $\bm{M}$, and it satisfies $\bm{z}(\bm{P} \bm{M}) = - \bm{z}(\bm{M})$. Moreover, $\losshat_{\mathrm{cvx}}(\bm{\theta}^{\bm{P}}) = \losshat_{\mathrm{cvx}}(\bm{\theta})$. Indeed, $\losshat_{\mathrm{cvx}}$ depends on $\bm{M}$ only through the scores of Definition~\ref{defi:mobt}, \ie through $\bm{1}_2^\top \bm{M} \bm{x}_i$ and $\| \bm{\Pi} \bm{M} \bm{x}_i \|_2$, with $i \in \dsb{N}$, which are invariant under $\bm{M} \mapsto \bm{P} \bm{M}$, since $\bm{1}_2^\top \bm{P} = \bm{1}_2^\top$, $\bm{\Pi} \bm{P} = \bm{P} \bm{\Pi}$, and $\bm{P}$ is orthogonal. Hence, $\losshat_{-\bm{s}}(\bm{\theta}^{\bm{P}}) = \losshat_{\bm{s}}(\bm{\theta})$ for every $\bm{\theta} \in \Theta_{\gamma_0}$, and the two convex programs have the same infimum. Moreover, if $\mathcal{C}$ is a region with sign vector $\bm{s}$, then $-\mathcal{C}$ is a region with sign vector $-\bm{s} \neq \bm{s}$, since $m \ge 1$. Thus, $\mathcal{S} = - \mathcal{S}$, and $\mathcal{S}$ is partitioned into $|\mathcal{S}| / 2$ pairs $\{ \bm{s}, -\bm{s} \}$, for each of which it suffices to solve one of the two convex programs. Finally, $m$ hyperplanes through the origin of $\mathbb{R}^k$ in general position divide it into $2 \sum_{j = 0}^{k-1} \binom{m-1}{j}$ regions \citep[][Theorem 2]{cover1965geometrical}, which upper bounds the number of regions of any central arrangement of $m$ hyperplanes in $\mathbb{R}^k$, and is of order $\mathcal{O}(m^{k-1})$ for a fixed $k$.
\end{proof}

Proposition~\ref{prop:global_2d} shows that, for $d = 2$, a global minimum of the empirical NLL over $\Theta_{\gamma_0}$ can be computed by solving at most $|\mathcal{S}| / 2$ convex programs, a number that is polynomial in $N$ for a fixed number of features $k$, although exponential in $k$. The set $\mathcal{S}$ can be enumerated incrementally, adding one hyperplane at a time and checking, for every sign vector realized by the hyperplanes considered so far and for both signs of the new one, the feasibility of the linear constraints $s_i \bm{x}_i^\top \bm{z} \ge 1$ on $\bm{z} \in \mathbb{R}^k$, which, by homogeneity, is equivalent to the feasibility of the corresponding strict inequalities. This requires solving $\mathcal{O}(m^{k})$ linear programs for a fixed $k$.
Importantly, if $\widehat{\bm{\theta}}$ is a global minimizer of $\losshat(\cdot | \mathcal{D})$ over $\Theta_{\gamma_0}$, the last term of Equation~\eqref{eq:samplecomplexitylocal:decomposition} is non-positive, since $\bm{\theta}^{*}_{\gamma_0} \in \Theta_{\gamma_0}$. Thus, the proof of Theorem~\ref{thr:sample_complexity_local} yields the bound of Equation~\eqref{eq:sample_complexity_local} \emph{without} the additive term proportional to $\mathcal{P}_{\bm{\theta}^{*}}(\incomprel)$, and so does the proof of Corollary~\ref{coroll:endtoend} in Appendix~\ref{apx:finite_downstream}.

When a local minimum is computed instead, the structure of the case $d = 2$ also allows refining the additive term of Theorem~\ref{thr:sample_complexity_local}, as shown in the following remark.

\begin{remark}[Local Minima for $d = 2$]
\label{rem:local_2d}
    Let $d = 2$, and let $\widehat{\bm{\theta}}_{\mathrm{loc}}$, $\bm{\theta}^{*}_{\gamma_0}$, and $\bm{v}_i$ be as in the proof of Theorem~\ref{thr:sample_complexity_local}. Let $\bm{z}^{*} \coloneqq \bm{z}(\bm{M}^{*})$ and $\widehat{\bm{z}}_{\mathrm{loc}} \coloneqq \bm{z}(\widehat{\bm{M}}_{\mathrm{loc}})$, and, for every $i \in \mathcal{I}_{\incomprel}$, let $a_i \coloneqq \bm{z}^{*\top} \bm{x}_i$ and $\widehat{a}_i \coloneqq \widehat{\bm{z}}_{\mathrm{loc}}^\top \bm{x}_i$, with $\mathrm{sign}(0) \coloneqq 0$. For $d = 2$, we have $\bm{v}_i = \mathrm{sign}(\widehat{a}_i) (\bm{e}_1 - \bm{e}_2) / \sqrt{2}$ and $\gamma_0 \bm{\Pi} \bm{M}^{*}_{\gamma_0} \bm{x}_i = \gamma^{*} \bm{\Pi} \bm{M}^{*} \bm{x}_i = \gamma^{*} a_i (\bm{e}_1 - \bm{e}_2) / 2$. Hence, each summand of the first inequality of Equation~\eqref{eq:samplecomplexitylocal:gap}, multiplied by $\gamma_0$, equals $\gamma^{*} ( |a_i| - \mathrm{sign}(\widehat{a}_i) a_i ) / \sqrt{2}$, which is zero if $\mathrm{sign}(\widehat{a}_i) = \mathrm{sign}(a_i)$, and at most $\sqrt{2} \gamma^{*} |a_i|$ otherwise. Thus:
    \begin{equation}
    \nonumber
        \losshat(\widehat{\bm{\theta}}_{\mathrm{loc}} | \mathcal{D}) - \losshat(\bm{\theta}^{*}_{\gamma_0} | \mathcal{D}) \le \frac{\sqrt{2} \gamma^{*}}{N} \!\!\! \sum_{\substack{i \in \mathcal{I}_{\incomprel} : \\ \mathrm{sign}(\widehat{a}_i) \neq \mathrm{sign}(a_i)}} \!\!\! | a_i |,
    \end{equation}
    \ie a local minimum only pays for the incomparabilities on which it misjudges which objective favors $\tau_i$ over $\tau_i'$. Moreover, the parametrization obtained from $\bm{\theta}^{*}_{\gamma_0}$ by replacing $\bm{M}^{*}_{\gamma_0}$ with $\bm{P} \bm{M}^{*}_{\gamma_0} = T_c(\bm{P} \bm{M}^{*})$ also belongs to $\Theta_{\gamma_0}$ and induces the same MOBT model (Theorem~\ref{thr:identifiability_main}). Thus, repeating the linearization argument of the proof of Theorem~\ref{thr:sample_complexity_local} with $\bm{\theta}^{*}_{\gamma_0}$ replaced by such a parametrization yields the same bound with $\bm{z}^{*}$ replaced by $-\bm{z}^{*}$.
\end{remark}

No analogous structure is available for $d = 3$. Indeed, in this case, $\| \bm{\Pi} \bm{M} \bm{x}_i \|_2$ is the Euclidean norm of a two-dimensional linear function of $\bm{M}$, which is not piecewise linear. Thus, the concave part $\losshat_{\mathrm{ccv}}$ can only be written as a minimum of linear functions over a continuum, since $- \| \bm{y} \|_2 = \min_{\| \bm{u} \|_2 \le 1} - \bm{u}^\top \bm{y}$, and the approach of Proposition~\ref{prop:global_2d} does not yield a finite number of convex programs. Whether a global minimum can be computed efficiently for $d = 3$ is left open.

\subsection{Tightness of the Scalarized Front Guarantees for $d = 3$}
\label{apx:pareto_tight}

By result ($v$) of Lemma~\ref{lem:orthant}, $\Delta_2 = \mathcal{W}_{\sqrt{2/3}}$. Thus, for $d = 3$, Theorem~\ref{thr:pareto} with $r = \sqrt{2/3}$, for which $\eta = \max\{0, (2c - 1)/3\}$, guarantees that every policy in $\Pi^{*}(\widehat{\bm{M}}; \Delta_2)$ is $\varepsilon$-Pareto optimal w.r.t.\ $\bm{M}^{*}$, with $\varepsilon = \max\{0, 2c - 1\} B / (2 (1 + c))$, which is non-decreasing in $c$ and equal to $B/4$ for $c = 1$. We now show that, for $d = 3$, the suboptimality $\max\{0, 2c - 1\} B / (2(1+c))$ guaranteed on the whole simplex $\Delta_2$ cannot be improved in general, even when the set of feature expectations is convex.

\begin{restatable}[Tightness for $d = 3$]{prop}{paretotightness}
\label{prop:pareto_tight}
    Let $d = 3$, $c \in (1/2, 1]$, $B > 0$, and $\varepsilon \in [0, (2c - 1) B / (2(1+c)))$. There exist an \mdpr with $k = 3$ features, for which the set $\{ \bm{\psi}(\pi) : \pi \in \Pi \}$ is convex, a matrix $\bm{M}^{*} \in \mathbb{R}^{3 \times 3}$ whose largest range of an objective is $B$, and $\bm{U} \in \mathcal{G}_3$ such that $\Pi^{*}(T_c(\bm{U} \bm{M}^{*}); \Delta_2)$ contains a policy that is not $\varepsilon$-Pareto optimal for $\bm{M}^{*}$.
\end{restatable}

\begin{proof}
Let $\varepsilon' \in (\varepsilon, (2c - 1) B / (2(1+c)))$, and observe that $\varepsilon' < B/4 < B$, since $c \le 1$. Consider the \mdpr with a single state $s$, two actions $a_1, a_2$, horizon $H = 1$, initial state distribution concentrated on $s$, and features $\bm{\varphi}(s, a_1) \coloneqq \bm{0}_3$ and $\bm{\varphi}(s, a_2) \coloneqq \bm{x} \coloneqq (\varepsilon', \varepsilon', B)^\top$. Every policy $\pi$ is identified by $t_\pi \coloneqq \pi_1(a_2 | s) \in [0, 1]$, and induces $\bm{\psi}(\pi) = t_\pi \bm{x}$, so that $\{ \bm{\psi}(\pi) : \pi \in \Pi \} = \{ t \bm{x} : t \in [0, 1] \}$, which is convex. Let $\pi_0$ and $\pi_1$ be the policies that deterministically play $a_1$ and $a_2$, respectively, so that $\bm{\psi}(\pi_0) = \bm{0}_3$ and $\bm{\psi}(\pi_1) = \bm{x}$. Let $\bm{M}^{*} \coloneqq \bm{I}_3$, so that the ranges of the three objectives over $\Pi$ are $\varepsilon'$, $\varepsilon'$, and $B$, and the largest one is $B$.

Let $\bm{U} \coloneqq \frac{2}{3} \bm{1}_3 \bm{1}_3^\top - \bm{I}_3$. Since $(\bm{1}_3 \bm{1}_3^\top)^2 = 3 \bm{1}_3 \bm{1}_3^\top$, we have $\bm{U}^\top \bm{U} = \bm{U}^2 = \frac{4}{3} \bm{1}_3 \bm{1}_3^\top - \frac{4}{3} \bm{1}_3 \bm{1}_3^\top + \bm{I}_3 = \bm{I}_3$, and $\bm{U} \bm{1}_3 = 2 \bm{1}_3 - \bm{1}_3 = \bm{1}_3$, hence $\bm{U} \in \mathcal{G}_3$. Moreover, $\bm{U} \bm{z} = - \bm{z}$ for every $\bm{z} \in \mathbb{R}^3$ with $\bm{1}_3^\top \bm{z} = 0$. Let $\bm{w} \coloneqq \bm{e}_3 \in \Delta_2$. By Proposition~\ref{prop:scalarization}, we have $\bm{w}^\top T_c(\bm{U} \bm{M}^{*}) \bm{\psi}(\pi) = \widetilde{\bm{w}}^\top \bm{\psi}(\pi)$ for every $\pi \in \Pi$, where, since $\bm{U}^\top = \bm{U}$ and $\bm{1}_3^\top (\bm{e}_3 - \bm{1}_3 / 3) = 0$:
\begin{equation}
\nonumber
    \widetilde{\bm{w}} = \frac{1}{3} \bm{1}_3 - c \left( \bm{e}_3 - \frac{1}{3} \bm{1}_3 \right) = \left( \frac{1 + c}{3}, \, \frac{1 + c}{3}, \, \frac{1 - 2c}{3} \right)^\top.
\end{equation}
We remark that $\widetilde{\bm{w}} \notin \Delta_2$, since its third element is negative for $c > 1/2$. Hence, for every $\pi \in \Pi$, we have $\widetilde{\bm{w}}^\top \bm{\psi}(\pi) = \frac{t_\pi}{3} \left( 2 (1 + c) \varepsilon' - (2c - 1) B \right)$, which is strictly negative whenever $t_\pi > 0$, by the choice of $\varepsilon'$. Therefore, $\argmax_{\pi \in \Pi} \bm{w}^\top T_c(\bm{U} \bm{M}^{*}) \bm{\psi}(\pi) = \{\pi_0\}$, and $\pi_0 \in \Pi^{*}(T_c(\bm{U} \bm{M}^{*}); \Delta_2)$. However, $\bm{M}^{*} (\bm{\psi}(\pi_1) - \bm{\psi}(\pi_0)) = (\varepsilon', \varepsilon', B)^\top > \varepsilon \bm{1}_3$ element-wise, since $\varepsilon' > \varepsilon$ and $B > \varepsilon$. Thus, $\pi_0$ is not $\varepsilon$-Pareto optimal for $\bm{M}^{*}$.
\end{proof}

\subsection{From KL Guarantees to Downstream Guarantees}
\label{apx:finite_downstream}

Theorem~\ref{thr:pareto} provides a \emph{population} statement, meaning that it holds when $\widehat{\bm{M}}$ can be defined as a transformation of $\bm{M}^{*}$ following Theorem~\ref{thr:identifiability_main}.
We now show how the finite-sample guarantees of Theorem~\ref{thr:sample_complexity_local} transfer to the Pareto front.
The resulting end-to-end guarantee, \ie Corollary~\ref{coroll:endtoend}, is a \emph{corollary} of the previous results. Indeed, it is obtained by chaining Theorem~\ref{thr:sample_complexity_local}, which bounds the KL divergence between the expert's and the learned comparison distributions, Proposition~\ref{prop:kl_to_params}, which converts this bound into bounds on the identifiable quantities of Theorem~\ref{thr:identifiability_main}, and Theorem~\ref{thr:pareto_approx}, which is the finite-sample counterpart of Theorem~\ref{thr:pareto}.

This requires quantitative counterparts of Assumption~\ref{asm:rich_design}, which we formalize in the following assumption.

\begin{asm}[Quantitative Comparison Design]
\label{asm:quant_design}
    Let $\bm{x} \coloneqq \bm{\phi}(\tau) - \bm{\phi}(\tau')$, with $(\tau, \tau') \sim \mathcal{Q}$. The constants $\lambda_2 \coloneqq \lambda_{\min} ( \mathbb{E}_{\mathcal{Q}} [ \bm{x} \bm{x}^\top ] )$, $\lambda_4 \coloneqq \min_{\bm{D} \in \mathbb{S}^k : \| \bm{D} \|_F = 1} \mathbb{E}_{\mathcal{Q}} [ (\bm{x}^\top \bm{D} \bm{x})^2 ]$, and $p_0 \coloneqq \mathcal{Q}(\bm{x} = \bm{0}_k)$ are strictly positive.
\end{asm}

The constants $\lambda_2$ and $\lambda_4$ are coverage conditions on $\mathcal{Q}$, analogous to concentrability in offline RLHF \citep{zhu2023principled, zhan2024provable}, and are quantitative counterparts of item ($i$) of Assumption~\ref{asm:rich_design}. Similarly, $p_0 > 0$ is a quantitative counterpart of item ($ii$), and it is satisfied by including a fraction of pairs with identical features, whose labels only carry information on $\alpha$ and $\beta$. Moreover, Assumption~\ref{asm:quant_design} implies Assumption~\ref{asm:rich_design}. Indeed, if $\bm{D} \in \mathbb{S}^k \setminus \{\bm{0}\}$ satisfied $\bm{x}^\top \bm{D} \bm{x} = 0$ for every $\bm{x} \in \mathcal{X}$, then $\mathbb{E}_{\mathcal{Q}} [ (\bm{x}^\top \bm{D} \bm{x})^2 ] = 0$, contradicting $\lambda_4 > 0$. Hence, the only $\bm{D} \in \mathbb{S}^k$ orthogonal, w.r.t.\ the Frobenius inner product, to every $\bm{x} \bm{x}^\top$ with $\bm{x} \in \mathcal{X}$ is $\bm{0}$, \ie $\mathrm{span} \{ \bm{x} \bm{x}^\top : \bm{x} \in \mathcal{X} \} = \mathbb{S}^k$, which is item ($i$). Moreover, $p_0 > 0$ implies $\bm{0}_k \in \mathcal{X}$, so that item ($ii$) holds with $\bm{x} = \bm{0}_k$ and $t = 0$.

Moreover, the bounds depend on the \emph{design-local curvature} $\omega_{\mathcal{Q}}$, \ie the smallest probability that the expert's own model assigns to any outcome on any pair that the design can present, which we formally define in Proposition~\ref{prop:kl_to_params}. Under Assumption~\ref{asm:bounded}, $\omega_{\mathcal{Q}} > 0$, and it is small whenever some outcome is unlikely on some pair that the design can present, e.g., when the expert's propensity towards incomparability or indifference is low.

We first convert a bound on the KL divergence into bounds on the identifiable quantities.

\begin{restatable}[From KL to Identifiable Quantities]{prop}{kltoparams}
\label{prop:kl_to_params}
    Let $d \in \{2, 3\}$, let $\bm{\theta}^{*} = (\mathrm{vec}(\bm{M}^{*}), \alpha^{*}, \beta^{*}, \gamma^{*}) \in \Theta$ and $\widehat{\bm{\theta}} = (\mathrm{vec}(\widehat{\bm{M}}), \widehat{\alpha}, \widehat{\beta}, \widehat{\gamma}) \in [-R, R]^{dk} \times (0, R] \times [-R, 0] \times \Gamma_d^{+}$, and let $\kappa \coloneqq \dkl{\mathcal{P}_{\bm{\theta}^{*}}}{\mathcal{P}_{\widehat{\bm{\theta}}}}$. Let $\bar{h} \coloneqq d R \Lambda \sqrt{k}/2$, so that $\varsigma = \log 4 + 2\bar{h}$, with $\varsigma$ as in Lemma~\ref{lem:loss_bounded}, and let:
    \begin{equation}
    \label{eq:kltoparams:chi}
        \chi \coloneqq \sqrt{\frac{2 \kappa}{c_{\mathcal{Q}}}},
        \qquad \text{where} \qquad
        c_{\mathcal{Q}} \coloneqq \max \left\{ \frac{\omega_{\mathcal{Q}}}{2 (1 + 4\bar{h})}, \; \frac{e^{-2\bar{h}}}{8} \right\},
        \qquad
        \omega_{\mathcal{Q}} \coloneqq \essinf_{(\tau, \tau') \sim \mathcal{Q}} \; \min_{\comp \in \Comps} f(\comp | \tau, \tau'; \bm{\theta}^{*}).
    \end{equation}
    Let $\bm{m}^{*} \coloneqq \bm{M}^{*\top} \bm{1}_d$ and $\bm{S}^{*} \coloneqq \gamma^{*2} \bm{M}^{*\top} \bm{\Pi} \bm{M}^{*}$, and analogously $\widehat{\bm{m}}$ and $\widehat{\bm{S}}$. Under Assumptions~\ref{asm:features}, \ref{asm:bounded}, and~\ref{asm:quant_design}, it holds that $| \widehat{\alpha} - \alpha^{*} | \le \chi$, $| \widehat{\beta} - \beta^{*} | \le \chi / \sqrt{p_0}$, and:
    \begin{equation}
    \nonumber
        \| \widehat{\bm{m}} - \bm{m}^{*} \|_2 \le \epsilon_m \coloneqq \frac{\chi}{\sqrt{\lambda_2}},
        \qquad
        \| \widehat{\bm{S}} - \bm{S}^{*} \|_F \le \epsilon_S \coloneqq 2 d R \Lambda \chi \sqrt{\frac{k}{(d-1) p_0 \lambda_4}}.
    \end{equation}
\end{restatable}

Two remarks on $\chi$ are in order. First, the two branches of $c_{\mathcal{Q}}$ yield:
\begin{equation}
\label{eq:kltoparams:branches}
    \chi \le 2 \sqrt{\frac{(1 + 4\bar{h}) \kappa}{\omega_{\mathcal{Q}}}}
    \qquad \text{and} \qquad
    \chi \le 4 \sqrt{e^{2\bar{h}} \kappa} = 2 \sqrt{e^{\varsigma} \kappa},
\end{equation}
the second of which is exactly the quantity obtained from the uniform curvature bound, \ie Proposition~\ref{prop:kl_to_params} is never weaker than the analysis based on the range of the scores alone. Second, $\omega_{\mathcal{Q}} > 0$ under Assumptions~\ref{asm:features} and~\ref{asm:bounded}, since all the scores of $\bm{\theta}^{*}$ lie in $[-\bar{h}, \bar{h}]$ by Lemma~\ref{lem:loss_bounded}, which also gives $\omega_{\mathcal{Q}} \ge e^{-2\bar{h}}/4 = e^{-\varsigma}$. Hence, the \emph{design-local curvature} $\omega_{\mathcal{Q}}$, \ie the smallest probability that the expert's own model assigns to any outcome on any pair that the design can present, replaces the worst-case factor $e^{\varsigma}$ at the price of the polynomial factor $1 + 4\bar{h}$.

\begin{proof}
For every $\tau, \tau' \in \Ts$, let $\bm{x} \coloneqq \bm{\phi}(\tau) - \bm{\phi}(\tau')$, let $\bm{h}^{*}(\bm{x}) \coloneqq (h_{\comp}(\bm{M}^{*} \bm{x}))_{\comp \in \Comps}$ be the scores of Definition~\ref{defi:mobt} computed with parameters $(\gamma^{*}, \alpha^{*}, \beta^{*})$, let $\widehat{\bm{h}}(\bm{x})$ be the analogous scores computed at $\widehat{\bm{\theta}}$, and let $\bm{\Delta}(\bm{x}) \coloneqq \widehat{\bm{h}}(\bm{x}) - \bm{h}^{*}(\bm{x})$. As shown in the proof of Lemma~\ref{lem:loss_bounded}, since both $\bm{\theta}^{*}$ and $\widehat{\bm{\theta}}$ belong to $[-R, R]^{dk} \times (0, R] \times [-R, 0] \times \Gamma_d^{+}$, all the scores lie in $[-\bar{h}, \bar{h}]$. Consequently, every element of $\bm{\Delta}(\bm{x})$ lies in $[-2\bar{h}, 2\bar{h}]$, so that its span, \ie $\mathrm{sp}(\bm{\Delta}(\bm{x})) \coloneqq \max_{\comp \in \Comps} \Delta_{\comp}(\bm{x}) - \min_{\comp \in \Comps} \Delta_{\comp}(\bm{x})$, satisfies $\mathrm{sp}(\bm{\Delta}(\bm{x})) \le 4\bar{h}$.

\paragraph{Curvature of the Softmax.}
Fix $\tau, \tau' \in \Ts$ and $\comp, \comp' \in \Comps$ with $\comp \neq \comp'$, what follows being trivial if $\comp = \comp'$. Let $\bm{a} \coloneqq \bm{h}^{*}(\bm{x})$, $\bm{b} \coloneqq \widehat{\bm{h}}(\bm{x})$, and $\bm{\Delta} \coloneqq \bm{\Delta}(\bm{x}) = \bm{b} - \bm{a}$, so that $\bm{a}, \bm{b} \in [-\bar{h}, \bar{h}]^{\Comps}$, as shown above. Moreover, let $\bm{p} \coloneqq f(\cdot | \tau, \tau'; \bm{\theta}^{*})$ and $\bm{q} \coloneqq f(\cdot | \tau, \tau'; \widehat{\bm{\theta}})$, so that, by Equation~\eqref{eq:mobt_model}, $p_{\comp''} = \exp(a_{\comp''}) / \sum_{\comp''' \in \Comps} \exp(a_{\comp'''})$ and $q_{\comp''} = \exp(b_{\comp''}) / \sum_{\comp''' \in \Comps} \exp(b_{\comp'''})$ for every $\comp'' \in \Comps$, and let $\omega \coloneqq \min_{\comp'' \in \Comps} p_{\comp''}$. We show that:
\begin{equation}
\label{eq:kltoparams:curvature}
    \dkl{\bm{p}}{\bm{q}} \ge \frac{1}{2} \max \left\{ \omega \, \zeta(\mathrm{sp}(\bm{\Delta})), \; \frac{e^{-2\bar{h}}}{8} \right\} \left( \Delta_{\comp} - \Delta_{\comp'} \right)^2,
\end{equation}
where $\zeta(z) \coloneqq (e^{-z} + z - 1) / z^2$ for every $z > 0$ and $\zeta(0) \coloneqq 1/2$, and that $\zeta(z) \ge 1/(2(1+z))$ for every $z \ge 0$.

Let $\mathrm{LSE}(\bm{v}) \coloneqq \log \sum_{\comp \in \Comps} \exp(v_{\comp})$ for every $\bm{v} \in \mathbb{R}^{\Comps}$, whose gradient $\nabla \mathrm{LSE}(\bm{v}) = \bm{s}(\bm{v})$ is the softmax of $\bm{v}$, and whose Hessian is $\nabla^2 \mathrm{LSE}(\bm{v}) = \mathrm{diag}(\bm{s}(\bm{v})) - \bm{s}(\bm{v}) \bm{s}(\bm{v})^\top$ \citep[][Section 3.1.5]{boyd2004convex}. By definition, $\bm{p} = \bm{s}(\bm{a})$ and $\bm{q} = \bm{s}(\bm{b})$, so that:
\begin{equation}
\label{eq:softmaxkl:bregman}
    \dkl{\bm{p}}{\bm{q}} = \sum_{\comp \in \Comps} p_{\comp} \left( a_{\comp} - \mathrm{LSE}(\bm{a}) - b_{\comp} + \mathrm{LSE}(\bm{b}) \right) = \mathrm{LSE}(\bm{b}) - \mathrm{LSE}(\bm{a}) - \bm{p}^\top \bm{\Delta} = \int_0^1 (1 - t) \, \bm{\Delta}^\top \nabla^2 \mathrm{LSE}(\bm{a} + t \bm{\Delta}) \bm{\Delta} \, \mathrm{d}t,
\end{equation}
where the second expression is the Bregman divergence induced by $\mathrm{LSE}$ \citep{banerjee2005clustering}, and the last equality is Taylor's formula with integral remainder. Let $t \in [0, 1]$, $\bm{v} \coloneqq \bm{a} + t \bm{\Delta} = (1 - t) \bm{a} + t \bm{b}$, and $\bm{s} \coloneqq \bm{s}(\bm{v})$. Letting $\bar{\Delta} \coloneqq \bm{s}^\top \bm{\Delta}$, and using $(u - v)^2 \le 2 u^2 + 2 v^2$ in the last step, we have:
\begin{equation}
\label{eq:softmaxkl:hessian}
\begin{aligned}
    \bm{\Delta}^\top \nabla^2 \mathrm{LSE}(\bm{v}) \bm{\Delta}
    &= \sum_{\comp'' \in \Comps} s_{\comp''} \Delta_{\comp''}^2 - \bar{\Delta}^2
    = \sum_{\comp'' \in \Comps} s_{\comp''} (\Delta_{\comp''} - \bar{\Delta})^2 \\
    &\ge \left( \min_{\comp'' \in \Comps} s_{\comp''} \right) \left( (\Delta_{\comp} - \bar{\Delta})^2 + (\Delta_{\comp'} - \bar{\Delta})^2 \right)
    \ge \frac{\min_{\comp'' \in \Comps} s_{\comp''}}{2} (\Delta_{\comp} - \Delta_{\comp'})^2.
\end{aligned}
\end{equation}
We lower bound $\min_{\comp''} s_{\comp''}$ in two ways, which yield the two branches of the maximum in Equation~\eqref{eq:kltoparams:curvature}.

\paragraph{Uniform Branch.}
Since $\bm{v} \in [-\bar{h}, \bar{h}]^{\Comps}$, being a convex combination of $\bm{a}$ and $\bm{b}$, and since $|\Comps| = 4$, we have $s_{\comp''} \ge e^{-\bar{h}} / (4 e^{\bar{h}}) = e^{-2\bar{h}}/4$ for every $\comp'' \in \Comps$. Combining this with Equations~\eqref{eq:softmaxkl:bregman} and~\eqref{eq:softmaxkl:hessian}, and since $\int_0^1 (1 - t) \, \mathrm{d}t = 1/2$, yields the second branch.

\paragraph{Local Branch.}
The softmax is invariant under common shifts of its argument, \ie $\bm{s}(\bm{v} + c \bm{1}) = \bm{s}(\bm{v})$ for every $c \in \mathbb{R}$. Hence, replacing $\bm{b}$ with $\bm{b} - c \bm{1}$ leaves $\bm{q}$, $\dkl{\bm{p}}{\bm{q}}$, the span $\mathrm{sp}(\bm{\Delta})$, and every difference $\Delta_{\comp} - \Delta_{\comp'}$ unchanged. Choosing $c \coloneqq (\max_{\comp} \Delta_{\comp} + \min_{\comp} \Delta_{\comp})/2$, we may thus assume \wwlog that $\| \bm{\Delta} \|_{\infty} = \mathrm{sp}(\bm{\Delta})/2$. Then, for every $\comp'' \in \Comps$, bounding the numerator and the denominator of the softmax separately gives:
\begin{equation}
\nonumber
    s_{\comp''}(\bm{v}) = \frac{\exp(a_{\comp''} + t \Delta_{\comp''})}{\sum_{\comp''' \in \Comps} \exp(a_{\comp'''} + t \Delta_{\comp'''})} \ge \frac{e^{-t \| \bm{\Delta} \|_{\infty}} \exp(a_{\comp''})}{e^{t \| \bm{\Delta} \|_{\infty}} \sum_{\comp''' \in \Comps} \exp(a_{\comp'''})} = e^{-2 t \| \bm{\Delta} \|_{\infty}} p_{\comp''} \ge e^{-t \mathrm{sp}(\bm{\Delta})} \omega.
\end{equation}
Integrating by parts, $\int_0^1 (1 - t) e^{-t z} \, \mathrm{d}t = 1/z - (1 - e^{-z})/z^2 = \zeta(z)$ for every $z > 0$, and the identity also holds for $z = 0$, since $\int_0^1 (1 - t) \, \mathrm{d}t = 1/2 = \zeta(0)$. Evaluating it at $z = \mathrm{sp}(\bm{\Delta})$ and combining this with Equations~\eqref{eq:softmaxkl:bregman} and~\eqref{eq:softmaxkl:hessian} yields the first branch.

\paragraph{Lower Bound on $\zeta$.}
Finally, let $g(z) \coloneqq 2 (1 + z) (e^{-z} + z - 1) - z^2$, so that, for every $z > 0$, $\zeta(z) \ge 1/(2(1+z))$ is equivalent to $g(z) \ge 0$. We have $g(0) = 0$ and, for every $z \ge 0$:
\begin{equation}
\nonumber
    g'(z) = 2 (e^{-z} + z - 1) + 2 (1 + z) (1 - e^{-z}) - 2 z = 4 z - 2 z e^{-z} - 2 z = 2 z \left( 1 - e^{-z} \right) \ge 0,
\end{equation}
\ie $g$ is non-decreasing on $[0, +\infty)$, and thus $\zeta(z) \ge 1/(2(1+z))$ for every $z \ge 0$, where the case $z = 0$ holds with equality by definition of $\zeta(0)$.

\paragraph{Bounds on the Identifiable Quantities.}
Since $\mathrm{sp}(\bm{\Delta}(\bm{x})) \le 4\bar{h}$, we have $\zeta(\mathrm{sp}(\bm{\Delta}(\bm{x}))) \ge 1/(2(1+4\bar{h}))$. Moreover, $\omega = \min_{\comp''} f(\comp'' | \tau, \tau'; \bm{\theta}^{*}) \ge \omega_{\mathcal{Q}}$ for ${\mathcal{Q}}$-almost every pair. Thus, the maximum in Equation~\eqref{eq:kltoparams:curvature} is at least $c_{\mathcal{Q}}$. Hence, by Equation~\eqref{eq:samplecomplexity:kl}, for every $\comp, \comp' \in \Comps$, we have:
\begin{equation}
\nonumber
    \kappa = \mathbb{E}_{(\tau, \tau') \sim \mathcal{Q}} \left[ \dkl{f(\cdot | \tau, \tau'; \bm{\theta}^{*})}{f(\cdot | \tau, \tau'; \widehat{\bm{\theta}})} \right] \ge \frac{c_{\mathcal{Q}}}{2} \mathbb{E}_{\mathcal{Q}} \left[ (\Delta_{\comp}(\bm{x}) - \Delta_{\comp'}(\bm{x}))^2 \right],
\end{equation}
\ie $\mathbb{E}_{\mathcal{Q}} [ (\Delta_{\comp}(\bm{x}) - \Delta_{\comp'}(\bm{x}))^2 ] \le 2 \kappa / c_{\mathcal{Q}} = \chi^2$.

We now compute the relevant differences. Let $e_1(\bm{x}) \coloneqq (\widehat{\bm{m}} - \bm{m}^{*})^\top \bm{x}$, $a^{*}(\bm{x}) \coloneqq \gamma^{*} \| \bm{\Pi} \bm{M}^{*} \bm{x} \|_2$, $\widehat{a}(\bm{x}) \coloneqq \widehat{\gamma} \| \bm{\Pi} \widehat{\bm{M}} \bm{x} \|_2$, and $e_3(\bm{x}) \coloneqq \widehat{a}(\bm{x}) - a^{*}(\bm{x}) + \widehat{\beta} - \beta^{*}$. By Definition~\ref{defi:mobt} and Equation~\eqref{eq:derivation:dispersion}, we have $\Delta_{\prelsinv}(\bm{x}) = e_1(\bm{x}) / 2$, $\Delta_{\prels}(\bm{x}) = - e_1(\bm{x}) / 2$, $\Delta_{\equivrel}(\bm{x}) = \widehat{\alpha} - \alpha^{*}$, and $\Delta_{\incomprel}(\bm{x}) = e_3(\bm{x})$. Thus, omitting the dependence on $\bm{x}$:
\begin{equation}
\nonumber
    \Delta_{\prelsinv} - \Delta_{\prels} = e_1,
    \qquad
    \Delta_{\equivrel} - \Delta_{\prels / \prelsinv} = (\widehat{\alpha} - \alpha^{*}) \pm \frac{e_1}{2},
    \qquad
    \Delta_{\incomprel} - \Delta_{\prels / \prelsinv} = e_3 \pm \frac{e_1}{2},
\end{equation}
where the sign $+$ corresponds to $\prels$ and the sign $-$ to $\prelsinv$. Recalling that $u^2 \le \frac{1}{2} (u + v)^2 + \frac{1}{2} (u - v)^2$ for every $u, v \in \mathbb{R}$, we obtain:
\begin{equation}
\label{eq:kltoparams:moments}
    \mathbb{E}_{\mathcal{Q}} [ e_1(\bm{x})^2 ] \le \chi^2,
    \qquad
    (\widehat{\alpha} - \alpha^{*})^2 \le \chi^2,
    \qquad
    \mathbb{E}_{\mathcal{Q}} [ e_3(\bm{x})^2 ] \le \chi^2.
\end{equation}

The bound on $\widehat{\alpha}$ is the second inequality of Equation~\eqref{eq:kltoparams:moments}. For the bound on $\widehat{\bm{m}}$, the first inequality of Equation~\eqref{eq:kltoparams:moments} yields $\lambda_2 \| \widehat{\bm{m}} - \bm{m}^{*} \|_2^2 \le (\widehat{\bm{m}} - \bm{m}^{*})^\top \mathbb{E}_{\mathcal{Q}} [\bm{x} \bm{x}^\top] (\widehat{\bm{m}} - \bm{m}^{*}) = \mathbb{E}_{\mathcal{Q}} [e_1(\bm{x})^2] \le \chi^2$. For the bound on $\widehat{\beta}$, we observe that $\widehat{a}(\bm{0}_k) = a^{*}(\bm{0}_k) = 0$, so that $e_3(\bm{x}) = \widehat{\beta} - \beta^{*}$ whenever $\bm{x} = \bm{0}_k$, which happens with probability $p_0$. Hence, $p_0 (\widehat{\beta} - \beta^{*})^2 \le \mathbb{E}_{\mathcal{Q}} [e_3(\bm{x})^2] \le \chi^2$.

For the bound on $\widehat{\bm{S}}$, we first observe that $\mathbb{E}_{\mathcal{Q}} [ (\widehat{a}(\bm{x}) - a^{*}(\bm{x}))^2 ] \le 2 \mathbb{E}_{\mathcal{Q}} [e_3(\bm{x})^2] + 2 (\widehat{\beta} - \beta^{*})^2 \le 2 \chi^2 + 2 \chi^2 / p_0 \le 4 \chi^2 / p_0$. Moreover, since $\bm{\Pi}^\top \bm{\Pi} = \bm{\Pi}$, we have $\widehat{a}(\bm{x})^2 - a^{*}(\bm{x})^2 = \bm{x}^\top (\widehat{\bm{S}} - \bm{S}^{*}) \bm{x}$. By Equations~\eqref{eq:lossbounded:delta} and~\eqref{eq:lossbounded:dispersion}, and since $\gamma^{*}, \widehat{\gamma} \le \sqrt{d} / (2\sqrt{d-1})$, we have $0 \le \widehat{a}(\bm{x}), a^{*}(\bm{x}) \le d R \Lambda \sqrt{k} / (2 \sqrt{d-1})$, and thus:
\begin{equation}
\nonumber
    \mathbb{E}_{\mathcal{Q}} \left[ \left( \bm{x}^\top (\widehat{\bm{S}} - \bm{S}^{*}) \bm{x} \right)^2 \right] = \mathbb{E}_{\mathcal{Q}} \left[ (\widehat{a}(\bm{x}) - a^{*}(\bm{x}))^2 (\widehat{a}(\bm{x}) + a^{*}(\bm{x}))^2 \right] \le \frac{d^2 R^2 \Lambda^2 k}{d-1} \cdot \frac{4 \chi^2}{p_0}.
\end{equation}
Since $\widehat{\bm{S}} - \bm{S}^{*} \in \mathbb{S}^k$, by definition of $\lambda_4$ and by homogeneity, the left-hand side is at least $\lambda_4 \| \widehat{\bm{S}} - \bm{S}^{*} \|_F^2$, which yields the bound on $\widehat{\bm{S}}$ and concludes the proof.
\end{proof}

We then show that bounds on the identifiable quantities translate into approximate Pareto optimality, provided that the conflict is non-degenerate.

\begin{restatable}[Approximate Pareto Front Guarantees]{thr}{paretoapprox}
\label{thr:pareto_approx}
    Under the assumptions of Proposition~\ref{prop:kl_to_params}, let $\gamma_0 \in \Gamma_d^{+}$ with $\gamma_0 \ge \gamma^{*}$, let $\widehat{\gamma} = \gamma_0$, and let $c \coloneqq \gamma^{*} / \gamma_0$. Assume that $\sigma_{\min} \coloneqq \sigma_{d-1}(\bm{\Pi} \bm{M}^{*}) > 0$, where $\sigma_i(\cdot)$ denotes the $i$-th largest singular value. For every $r \ge 0$, let:
    \begin{equation}
    \nonumber
        \xi(r) \coloneqq \Lambda \left( \frac{\epsilon_m}{d} + \frac{r \, \epsilon_S}{\sqrt{2 (\sqrt{2} - 1)} \, \gamma_0 \gamma^{*} \sigma_{\min}} \right).
    \end{equation}
    Then, for every $r \in [0, \rho_{\mathrm{in}}]$, every $\pi \in \Pi^{*}(\widehat{\bm{M}}; \mathcal{W}_r)$ is $\xi(r)$-Pareto optimal for $\bm{M}^{*}$. Moreover, if $d = 3$, for every $r \ge \rho_{\mathrm{in}}$, every $\pi \in \Pi^{*}(\widehat{\bm{M}}; \mathcal{W}_r)$ is $\varepsilon$-Pareto optimal for $\bm{M}^{*}$, with $\varepsilon \coloneqq (\eta B + \xi(r)) / (1 + \eta)$, where $\eta$ and $B$ are defined as in Theorem~\ref{thr:pareto}.
\end{restatable}

\begin{proof}
We first observe that, for every pair of policies $\pi, \pi'$, drawing $\tau \sim \pi$ and $\tau' \sim \pi'$ independently, Jensen's inequality and Assumption~\ref{asm:bounded} yield $\| \bm{\psi}(\pi') - \bm{\psi}(\pi) \|_2 = \| \mathbb{E}[\bm{\phi}(\tau') - \bm{\phi}(\tau)] \|_2 \le \Lambda$. The proof is structured in three steps: first, we align the learned and true conflict components with a matrix $\bm{U} \in \mathcal{G}_d$, second, we show that the scalarized objectives of $\widehat{\bm{M}}$ and $\bm{M}^{*}$ are close up to the change of weights of Proposition~\ref{prop:scalarization}, and third, we extend claim ($b$) of the proof of Theorem~\ref{thr:pareto} to approximate maximizers.

\paragraph{Alignment.}
Let $\bm{V} \in \mathbb{R}^{d \times (d-1)}$ be a matrix whose columns form an orthonormal basis of the orthogonal complement of $\bm{1}_d$, so that $\bm{V}^\top \bm{V} = \bm{I}_{d-1}$, $\bm{V}^\top \bm{1}_d = \bm{0}_{d-1}$, and $\bm{V} \bm{V}^\top = \bm{\Pi}$. Let $\widehat{\bm{E}} \coloneqq \gamma_0 \bm{\Pi} \widehat{\bm{M}}$, $\bm{E}^{*} \coloneqq \gamma^{*} \bm{\Pi} \bm{M}^{*}$, $\widehat{\bm{Y}} \coloneqq \bm{V}^\top \widehat{\bm{E}}$, and $\bm{Y}^{*} \coloneqq \bm{V}^\top \bm{E}^{*}$, so that $\widehat{\bm{E}} = \bm{V} \widehat{\bm{Y}}$ and $\bm{E}^{*} = \bm{V} \bm{Y}^{*}$. Then, $\widehat{\bm{Y}}^\top \widehat{\bm{Y}} = \widehat{\bm{E}}^\top \widehat{\bm{E}} = \widehat{\bm{S}}$ and $\bm{Y}^{*\top} \bm{Y}^{*} = \bm{S}^{*}$, and the singular values of $\bm{Y}^{*} \in \mathbb{R}^{(d-1) \times k}$ coincide with the $d-1$ largest singular values of $\bm{E}^{*}$. Hence, the smallest singular value of $\bm{Y}^{*}$ is $\sigma_{d-1}(\bm{Y}^{*}) = \gamma^{*} \sigma_{\min}$.
By \citet[][Lemma 5.4]{tu2016low}, applied to the matrices $\widehat{\bm{Y}}^\top, \bm{Y}^{*\top} \in \mathbb{R}^{k \times (d-1)}$, for which $(\widehat{\bm{Y}}^\top)(\widehat{\bm{Y}}^\top)^\top = \widehat{\bm{S}}$, $(\bm{Y}^{*\top})(\bm{Y}^{*\top})^\top = \bm{S}^{*}$, and $\sigma_{d-1}(\bm{Y}^{*\top}) = \gamma^{*} \sigma_{\min}$, there exists an orthogonal matrix $\bm{O}' \in \mathbb{R}^{(d-1) \times (d-1)}$ such that $\| \widehat{\bm{Y}}^\top - \bm{Y}^{*\top} \bm{O}' \|_F^2 \le \| \widehat{\bm{S}} - \bm{S}^{*} \|_F^2 / (2 (\sqrt{2} - 1) \gamma^{*2} \sigma_{\min}^2)$. Letting $\bm{O} \coloneqq \bm{O}'^\top$, which is orthogonal, and recalling that $\| \widehat{\bm{S}} - \bm{S}^{*} \|_F \le \epsilon_S$ by Proposition~\ref{prop:kl_to_params}, we obtain $\| \widehat{\bm{Y}} - \bm{O} \bm{Y}^{*} \|_F \le \epsilon_S / (\sqrt{2 (\sqrt{2} - 1)} \, \gamma^{*} \sigma_{\min})$. Let $\bm{U} \coloneqq \bm{A}_d + \bm{V} \bm{O} \bm{V}^\top$. Then, $\bm{U} \bm{1}_d = \bm{1}_d$ and $\bm{U}^\top \bm{U} = \bm{A}_d + \bm{V} \bm{O}^\top \bm{O} \bm{V}^\top = \bm{I}_d$, since the cross terms vanish, \ie $\bm{U} \in \mathcal{G}_d$. Moreover, since $\bm{1}_d^\top \bm{E}^{*} = \bm{0}_k^\top$, we have $\bm{U} \bm{E}^{*} = \bm{V} \bm{O} \bm{Y}^{*}$, and thus:
\begin{equation}
\label{eq:paretoapprox:alignment}
    \| \widehat{\bm{E}} - \bm{U} \bm{E}^{*} \|_{\mathrm{op}} \le \| \bm{V} (\widehat{\bm{Y}} - \bm{O} \bm{Y}^{*}) \|_F = \| \widehat{\bm{Y}} - \bm{O} \bm{Y}^{*} \|_F \le \frac{\epsilon_S}{\sqrt{2 (\sqrt{2} - 1)} \, \gamma^{*} \sigma_{\min}}.
\end{equation}

\paragraph{Approximate Scalarization.}
Let $r \ge 0$, $\bm{w} \in \mathcal{W}_r$, $\bm{z} \coloneqq \bm{w} - \bm{1}_d / d = \bm{\Pi} \bm{w}$, and $\widetilde{\bm{w}} \coloneqq \bm{1}_d / d + c \, \bm{U}^\top \bm{z}$, which satisfies $\bm{1}_d^\top \widetilde{\bm{w}} = 1$ and $\| \bm{\Pi} \widetilde{\bm{w}} \|_2 = c \| \bm{z} \|_2 \le c r$ by Proposition~\ref{prop:scalarization}, but need not belong to $\Delta_{d-1}$. Since $\bm{I}_d = \bm{1}_d \bm{1}_d^\top / d + \bm{\Pi}$, $\bm{1}_d^\top \bm{w} = \bm{1}_d^\top \widetilde{\bm{w}} = 1$, $\bm{\Pi} \widetilde{\bm{w}} = c \, \bm{U}^\top \bm{z}$, and $c / \gamma^{*} = 1 / \gamma_0$, we have:
\begin{equation}
\nonumber
    \bm{w}^\top \widehat{\bm{M}} = \frac{\widehat{\bm{m}}^\top}{d} + \bm{z}^\top \bm{\Pi} \widehat{\bm{M}} = \frac{\widehat{\bm{m}}^\top}{d} + \frac{\bm{z}^\top \widehat{\bm{E}}}{\gamma_0},
    \qquad
    \widetilde{\bm{w}}^\top \bm{M}^{*} = \frac{\bm{m}^{*\top}}{d} + c \, \bm{z}^\top \bm{U} \bm{\Pi} \bm{M}^{*} = \frac{\bm{m}^{*\top}}{d} + \frac{\bm{z}^\top \bm{U} \bm{E}^{*}}{\gamma_0}.
\end{equation}
Hence, for every $\pi, \pi' \in \Pi$, by Proposition~\ref{prop:kl_to_params}, Equation~\eqref{eq:paretoapprox:alignment}, $\| \bm{z} \|_2 \le r$, and $\| \bm{\psi}(\pi') - \bm{\psi}(\pi) \|_2 \le \Lambda$:
\begin{equation}
\nonumber
    \left| (\bm{w}^\top \widehat{\bm{M}} - \widetilde{\bm{w}}^\top \bm{M}^{*}) (\bm{\psi}(\pi') - \bm{\psi}(\pi)) \right| \le \frac{\| \widehat{\bm{m}} - \bm{m}^{*} \|_2 \Lambda}{d} + \frac{r \| \widehat{\bm{E}} - \bm{U} \bm{E}^{*} \|_{\mathrm{op}} \Lambda}{\gamma_0} \le \xi(r).
\end{equation}
Therefore, if $\widehat{\pi} \in \argmax_{\pi \in \Pi} \bm{w}^\top \widehat{\bm{M}} \bm{\psi}(\pi)$, then, for every $\pi' \in \Pi$, we have:
\begin{equation}
\label{eq:paretoapprox:approxmax}
    \widetilde{\bm{w}}^\top \bm{M}^{*} (\bm{\psi}(\pi') - \bm{\psi}(\widehat{\pi})) \le \bm{w}^\top \widehat{\bm{M}} (\bm{\psi}(\pi') - \bm{\psi}(\widehat{\pi})) + \xi(r) \le \xi(r).
\end{equation}

\paragraph{Approximate Pareto Optimality.}
We extend claim ($b$) of the proof of Theorem~\ref{thr:pareto}. Let $\widetilde{\bm{w}} \in \mathbb{R}^d$ with $\bm{1}_d^\top \widetilde{\bm{w}} = 1$ have at most one negative element, whose absolute value is at most $\widetilde{\eta} \ge 0$, let $\xi \ge 0$, and let $\widehat{\pi} \in \Pi$ satisfy $\widetilde{\bm{w}}^\top \bm{M}^{*} (\bm{\psi}(\pi') - \bm{\psi}(\widehat{\pi})) \le \xi$ for every $\pi' \in \Pi$. We show that $\widehat{\pi}$ is $\widetilde{\varepsilon}$-Pareto optimal for $\bm{M}^{*}$, with $\widetilde{\varepsilon} \coloneqq (\widetilde{\eta} B + \xi) / (1 + \widetilde{\eta})$. If $\widetilde{\varepsilon} \ge B$, the result is trivial, since every element of $\bm{M}^{*} (\bm{\psi}(\pi') - \bm{\psi}(\widehat{\pi}))$ is at most $B$ for every $\pi' \in \Pi$. Otherwise, $\widetilde{\varepsilon} < B$ implies $\xi < B$, and thus $\widetilde{\varepsilon} - \xi = \widetilde{\eta} (B - \xi) / (1 + \widetilde{\eta}) \ge 0$. Assume by contradiction that there exists $\pi' \in \Pi$ such that $\bm{g} \coloneqq \bm{M}^{*} (\bm{\psi}(\pi') - \bm{\psi}(\widehat{\pi})) > \widetilde{\varepsilon} \bm{1}_d$ element-wise. If $\widetilde{\bm{w}} \in \mathbb{R}^d_{\ge 0}$, then $\widetilde{\bm{w}}^\top \bm{g} \ge \min_{i \in \dsb{d}} g_i > \widetilde{\varepsilon} \ge \xi$. Otherwise, letting $j \in \dsb{d}$ be the unique index with $\widetilde{w}_j < 0$:
\begin{equation}
\nonumber
    \widetilde{\bm{w}}^\top \bm{g} \ge (1 + |\widetilde{w}_j|) \min_{i \neq j} g_i - |\widetilde{w}_j| g_j > (1 + |\widetilde{w}_j|) \widetilde{\varepsilon} - |\widetilde{w}_j| B = \widetilde{\varepsilon} - |\widetilde{w}_j| (B - \widetilde{\varepsilon}) \ge \widetilde{\varepsilon} - \widetilde{\eta} (B - \widetilde{\varepsilon}) = \xi.
\end{equation}
In both cases, this contradicts $\widetilde{\bm{w}}^\top \bm{g} \le \xi$, and proves the claim.

\paragraph{Conclusion.}
Let $\pi \in \Pi^{*}(\widehat{\bm{M}}; \mathcal{W}_r)$, \ie $\pi \in \argmax_{\pi'' \in \Pi} \bm{w}^\top \widehat{\bm{M}} \bm{\psi}(\pi'')$ for some $\bm{w} \in \mathcal{W}_r$, and let $\widetilde{\bm{w}}$ be as in the second step, so that Equation~\eqref{eq:paretoapprox:approxmax} holds with $\widehat{\pi} = \pi$. If $r \le \rho_{\mathrm{in}}$, then, since $c \le 1$, $\| \bm{\Pi} \widetilde{\bm{w}} \|_2 \le c r \le \rho_{\mathrm{in}}$, and result ($iii$) of Lemma~\ref{lem:orthant} yields $\widetilde{\bm{w}} \in \mathbb{R}^d_{\ge 0}$, \ie $\widetilde{\bm{w}} \in \Delta_{d-1}$. Thus, the claim with $\widetilde{\eta} = 0$ yields that $\pi$ is $\xi(r)$-Pareto optimal for $\bm{M}^{*}$. If $d = 3$ and $r \ge \rho_{\mathrm{in}}$, then, as in the proof of the last claim of Theorem~\ref{thr:pareto}, $\widetilde{\bm{w}}$ has at most one negative element, whose absolute value is at most $\eta$, and the claim with $\widetilde{\eta} = \eta$ concludes the proof.
\end{proof}

Some comments on the constants of Theorem~\ref{thr:pareto_approx} are in order. The condition $\sigma_{\min} > 0$ requires a \emph{non-degenerate} conflict, \eg for $d = 3$ the deviations of the objectives from their mean must span a two-dimensional space, and $1 / \sigma_{\min}$ measures how well-posed recovering the rotation $\bm{U}$ from $\widehat{\bm{S}}$ is. For $r = 0$, \ie for the uniform weights, the proof does not use $\sigma_{\min}$ at all, and $\xi(0) = \Lambda \epsilon_m / d$. Moreover, for $d = 3$, the term $\eta B / (1 + \eta)$ cannot be improved in general, since, for $r = \sqrt{2/3}$, it matches the population-level lower bound $(2c-1) B / (2(1+c))$ of Proposition~\ref{prop:pareto_tight}.

Finally, combining Theorem~\ref{thr:sample_complexity_local}, Proposition~\ref{prop:kl_to_params}, and Theorem~\ref{thr:pareto_approx}, we obtain the following end-to-end guarantee.

\begin{coroll}[End-to-End Pareto Front Guarantees]
\label{coroll:endtoend}
    Let $d \in \{2, 3\}$, let $\gamma_0 \in \Gamma_d^{+}$ with $\gamma_0 \ge \gamma^{*}$, and let $\widehat{\bm{M}}$ be the utility matrix of a local minimum $\widehat{\bm{\theta}}_{\mathrm{loc}}$ as in Theorem~\ref{thr:sample_complexity_local}. Assume $\sigma_{\min} > 0$.
    Under Assumptions~\ref{asm:features}, \ref{asm:bounded}, and~\ref{asm:quant_design}, for every $\delta \in (0,1)$, with probability at least $1 - \delta$, for every $r \in [0, \rho_{\mathrm{in}}]$, every $\pi \in \Pi^{*}(\widehat{\bm{M}}; \mathcal{W}_r)$ is $\varepsilon$-Pareto optimal for $\bm{M}^{*}$, with:
    \begin{equation}
    \nonumber
    \begin{aligned}
        \varepsilon = \widetilde{\mathcal{O}} \Bigg( \frac{C_r}{\sqrt{\omega_{\mathcal{Q}}}} \bigg( &R \Lambda d^{5/4} k^{3/4} \left( \frac{\log (4/\delta)}{N} \right)^{1/4} \\
        & + R \Lambda d^{3/4} \sqrt{k \gamma^{*} \mathcal{P}_{\bm{\theta}^{*}}(\incomprel)} \bigg) \Bigg),
    \end{aligned}
    \end{equation}
    where the constant $C_r \coloneqq \Lambda ( 1 / (d \sqrt{\lambda_2}) + r d R \Lambda \sqrt{k / ((d-1) p_0 \lambda_4)} / (\gamma_0 \gamma^{*} \sigma_{\min}) )$ collects the design and conditioning constants.
\end{coroll}

\begin{remark}[On the End-to-End Bound]
\label{rem:endtoend_not_optimized}
    We stress that Corollary~\ref{coroll:endtoend} is a direct consequence of Theorem~\ref{thr:sample_complexity_local}, Proposition~\ref{prop:kl_to_params}, and Theorem~\ref{thr:pareto_approx}, and that its proof simply composes the three results, each in the form that is most convenient to chain them. Thus, the proof is \emph{not} optimized to provide the tightest possible upper bound, and the purpose of Corollary~\ref{coroll:endtoend} is rather to show that the finite-sample guarantees transfer to the Pareto front, and to expose the structure of the resulting bound. Several steps of the proof are loose. For instance, the $N^{-1/4}$ rate stems from converting the bound of Theorem~\ref{thr:sample_complexity_local}, of order $N^{-1/2}$, into bounds on the parameters via the quadratic lower bound on the KL divergence of Proposition~\ref{prop:kl_to_params}, the alignment step of Theorem~\ref{thr:pareto_approx} bounds an operator norm with a Frobenius norm, and numerical constants are absorbed in the $\Theta(\cdot)$ and $\widetilde{\mathcal{O}}(\cdot)$ notations. A refined analysis, \eg one that directly targets the parameters rather than going through the KL divergence, may improve the dependence on $N$, $d$, $k$, and on the design constants, which we leave to future work.
\end{remark}

Moreover, three comments on the structure of the bound of Corollary~\ref{coroll:endtoend} are in order.
First, the rate is of order $N^{-1/4}$: the square root is the price of moving from a divergence to the parameters that induce it.
Second, the factor $1 / \sqrt{\omega_{\mathcal{Q}}}$ is the price of the curvature of the softmax link, which in BT-type models is usually paid in the worst case over the parameter space, i.e., exponentially in the range of the scores \citep{faury2020improved, zhu2023principled}. Evaluating it at the expert's parametrization instead makes it an \emph{instance-dependent} quantity: $1 / \omega_{\mathcal{Q}}$ never exceeds that worst-case factor, and is exponentially smaller whenever the expert's scores on the design span a range much smaller than the worst-case range $2\bar{h}$.
Third, and most importantly, the second term \emph{does not vanish} with $N$, as it follows directly from the local optimality guarantee.

\begin{proof}[Proof of Corollary~\ref{coroll:endtoend}]
By Theorem~\ref{thr:sample_complexity_local}, with probability at least $1 - \delta$, the quantity $\kappa \coloneqq \dkl{\mathcal{P}_{\bm{\theta}^{*}}}{\mathcal{P}_{\widehat{\bm{\theta}}_{\mathrm{loc}}}}$ is bounded by the right-hand side of Equation~\eqref{eq:sample_complexity_local} (see Equation~\eqref{eq:samplecomplexitylocal:explicit} for the explicit constants), \ie:
\begin{equation}
\label{eq:endtoend:kappa}
    \kappa \le \widetilde{\mathcal{O}} \left( R \Lambda d^{3/2} k \sqrt{\frac{\log(4/\delta)}{N}} + \gamma^{*} R \Lambda \sqrt{dk} \, \mathcal{P}_{\bm{\theta}^{*}}(\incomprel) \right).
\end{equation}
We work on this event for the remainder of the proof. By Proposition~\ref{prop:kl_to_params}, applied with $\widehat{\bm{\theta}} = \widehat{\bm{\theta}}_{\mathrm{loc}}$ and $\widehat{\gamma} = \gamma_0$, we have $\epsilon_m = \chi / \sqrt{\lambda_2}$ and $\epsilon_S = 2 d R \Lambda \chi \sqrt{k / ((d-1) p_0 \lambda_4)}$, with $\chi$ as in Equation~\eqref{eq:kltoparams:chi}. Substituting these into the definition of $\xi(r)$ of Theorem~\ref{thr:pareto_approx}, and collecting $\chi$, we obtain, for every $r \ge 0$:
\begin{equation}
\nonumber
    \xi(r) = \chi \, \Lambda \left( \frac{1}{d \sqrt{\lambda_2}} + \frac{2 \, r \, d R \Lambda}{\sqrt{2 (\sqrt{2} - 1)} \, \gamma_0 \gamma^{*} \sigma_{\min}} \sqrt{\frac{k}{(d-1) p_0 \lambda_4}} \right) = \Theta \left( \chi \, C_r \right),
\end{equation}
where $C_r$ is defined as in the statement, the numerical factor $2 / \sqrt{2 (\sqrt{2} - 1)}$ being absorbed in the $\Theta(\cdot)$ notation. Since $\sigma_{\min} > 0$ by assumption, and $r \le \rho_{\mathrm{in}}$, Theorem~\ref{thr:pareto_approx} applies, and every $\pi \in \Pi^{*}(\widehat{\bm{M}}; \mathcal{W}_r)$ is $\xi(r)$-Pareto optimal for $\bm{M}^{*}$. Finally, since $\sqrt{a + b} \le \sqrt{a} + \sqrt{b}$ for every $a, b \ge 0$, Equation~\eqref{eq:endtoend:kappa} yields:
\begin{equation}
    \chi \le 2 \sqrt{\frac{(1 + 4\bar{h}) \kappa}{\omega_{\mathcal{Q}}}} = \widetilde{\mathcal{O}} \left( \sqrt{\frac{1 + \bar{h}}{\omega_{\mathcal{Q}}}} \left( \left( R \Lambda d^{3/2} k \right)^{1/2} \left( \frac{\log(4/\delta)}{N} \right)^{1/4} + \left( R \Lambda \gamma^{*} \sqrt{dk} \, \mathcal{P}_{\bm{\theta}^{*}}(\incomprel) \right)^{1/2} \right) \right),
\end{equation}
where the first inequality is the first branch of Equation~\eqref{eq:kltoparams:branches}. Finally, since $\bar{h} = d R \Lambda \sqrt{k}/2 \ge 1$, we have $\sqrt{1+\bar{h}} \, ( R \Lambda d^{3/2} k )^{1/2} = \mathcal{O} ( R \Lambda d^{5/4} k^{3/4} )$ and $\sqrt{1+\bar{h}} \, ( R \Lambda \gamma^{*} \sqrt{dk} \, \mathcal{P}_{\bm{\theta}^{*}}(\incomprel) )^{1/2} = \mathcal{O} ( R \Lambda d^{3/4} \sqrt{k \gamma^{*} \mathcal{P}_{\bm{\theta}^{*}}(\incomprel)} )$, which, combined with $\xi(r) = \Theta(\chi C_r)$, concludes the proof.
\end{proof}

\section{Auxiliary Lemmas}
\label{apx:auxiliary}

In this appendix, we report the auxiliary lemmas employed in the proofs of Appendices~\ref{apx:proofs} and~\ref{apx:additional}.

\begin{lem}[Incomparability Rate]
\label{lem:incomp_rate}
Let $d \in \{2, 3\}$, let $\mathcal{Q} \in \Delta(\Ts^2)$, and let Assumptions~\ref{asm:features} and~\ref{asm:bounded} hold, so that $\bm{\theta}^{*} = (\mathrm{vec}(\bm{M}^{*}), \alpha^{*}, \beta^{*}, \gamma^{*}) \in \mathfrak{B}$, with $\mathfrak{B}$ defined as in Lemma~\ref{lem:loss_bounded}. Denote, for every $\tau, \tau' \in \Ts$, the utility difference as $\utildiff \coloneqq \bm{M}^{*} (\bm{\phi}(\tau) - \bm{\phi}(\tau'))$, and let:
\begin{equation}
    A \coloneqq \frac{d R \Lambda \sqrt{k}}{2 \sqrt{d-1}},
    \qquad
    K \coloneqq 2 + \exp(\alpha^{*}),
    \qquad \text{and} \qquad
    c_{\beta^{*}} \coloneqq \min \left\{ \frac{1}{4}, \; \frac{\exp(A + \beta^{*})}{K} \right\}.
\end{equation}
Then, it holds that:
\begin{equation}
\label{eq:incomprate:affine}
    \mathcal{P}_{\bm{\theta}^{*}}(\incomprel) \le f(\incomprel | \bm{0}_d) + c_{\beta^{*}} \gamma^{*} \sqrt{d} \; \mathbb{E}_{(\tau, \tau') \sim \mathcal{Q}} \left[ \std \left( \utildiff \right) \right],
\end{equation}
where $f(\incomprel | \bm{0}_d) = \exp(\beta^{*}) / (K + \exp(\beta^{*}))$ is the probability of declaring as incomparable two trajectories of equal utility. Moreover, it holds that:
\begin{equation}
\label{eq:incomprate:mult}
    \mathcal{P}_{\bm{\theta}^{*}}(\incomprel) \le \frac{\exp(A + \beta^{*})}{K}.
\end{equation}
\end{lem}

\begin{proof}
Fix $\tau, \tau' \in \Ts$, and let $\bm{\delta} \coloneqq \utildiff$, so that $f(\incomprel | \tau, \tau'; \bm{\theta}^{*}) = f(\incomprel | \bm{\delta})$ by Assumption~\ref{asm:features}. Let $s \coloneqq \bm{1}_d^\top \bm{\delta}$ and $a \coloneqq \gamma^{*} \sqrt{d} \std(\bm{\delta}) = \gamma^{*} \widetilde{h}_{\incomprel}(\bm{\delta})$, so that $h_{\incomprel}(\bm{\delta}) = a + \beta^{*}$ by Equation~\eqref{eq:mobt:incomparability}. By Equations~\eqref{eq:lossbounded:delta} and~\eqref{eq:lossbounded:dispersion}, and since $\gamma^{*} \le \sqrt{d} / (2\sqrt{d-1})$ by Equation~\eqref{eq:gammarange}, we have that:
\begin{equation}
\label{eq:incomprate:range}
    0 \le a \le \frac{\sqrt{d}}{2\sqrt{d-1}} R \Lambda \sqrt{dk} = A.
\end{equation}

By the AM-GM inequality, $\exp(s/2) + \exp(-s/2) \ge 2$, so that the denominator of Equation~\eqref{eq:mobt_model} is at least $K + \exp(a + \beta^{*})$, and thus:
\begin{equation}
\label{eq:incomprate:sigma}
    f(\incomprel | \bm{\delta}) \le \frac{\exp(a + \beta^{*})}{K + \exp(a + \beta^{*})} = g(a),
    \qquad \text{where } g(z) \coloneqq \sigma \left( z + \beta^{*} - \log K \right), \; \forall z \in \mathbb{R},
\end{equation}
and $\sigma(\cdot)$ is the sigmoid function. Since $\sigma' = \sigma (1 - \sigma)$, with $\sigma \in (0, 1)$, we have that $\sigma' \le \min \{ 1/4, \sigma \}$, the first bound following from $u (1 - u) \le 1/4$ for every $u \in \mathbb{R}$. Moreover, $\sigma(u) \le \exp(u)$ for every $u \in \mathbb{R}$, since $\sigma(u) = \exp(u) / (1 + \exp(u))$ and $1 + \exp(u) \ge 1$. Hence, by the monotonicity of $\sigma$ and Equation~\eqref{eq:incomprate:range}, for every $z \in [0, A]$, we have that:
\begin{equation}
\nonumber
    0 \le g'(z) \le \min \left\{ \frac{1}{4}, \; \sigma \left( z + \beta^{*} - \log K \right) \right\} \le \min \left\{ \frac{1}{4}, \; \frac{\exp(A + \beta^{*})}{K} \right\} = c_{\beta^{*}}.
\end{equation}
Therefore, by the mean value theorem applied to $g$ on $[0, a]$, it follows that $g(a) \le g(0) + c_{\beta^{*}} a$. By Definition~\ref{defi:mobt}, the scores at $\bm{\delta} = \bm{0}_d$ are $(0, 0, \alpha^{*}, \beta^{*})$, so that:
\begin{equation}
\nonumber
    g(0) = \frac{\exp(\beta^{*})}{K + \exp(\beta^{*})} = \frac{\exp(\beta^{*})}{2 + \exp(\alpha^{*}) + \exp(\beta^{*})} = f(\incomprel | \bm{0}_d).
\end{equation}
Combining with Equation~\eqref{eq:incomprate:sigma}, we obtain $f(\incomprel | \bm{\delta}) \le f(\incomprel | \bm{0}_d) + c_{\beta^{*}} \gamma^{*} \sqrt{d} \std(\bm{\delta})$. Since this bound is affine in $\std(\bm{\delta})$ and holds for every $\tau, \tau' \in \Ts$, taking the expectation under $\mathcal{Q}$ yields Equation~\eqref{eq:incomprate:affine}.

Finally, Equation~\eqref{eq:incomprate:mult} follows by bounding $\sigma(a + \beta^{*} - \log K) \le \exp(a + \beta^{*} - \log K) \le \exp(A + \beta^{*}) / K$ in Equation~\eqref{eq:incomprate:sigma}, and taking the expectation under $\mathcal{Q}$.
\end{proof}

\begin{lem}[Bounded MOBT Loss]
\label{lem:loss_bounded}
Let $d \in \{2, 3\}$, let Assumption~\ref{asm:features} and item ($ii$) of Assumption~\ref{asm:bounded} hold, with $\Lambda \ge 1$, and let $\mathfrak{B} \coloneqq [-R, R]^{dk} \times (0, R] \times [-R, 0] \times \Gamma_d^{+}$, for some $R \ge 1$. Then, for every $\bm{\theta} \in \mathfrak{B}$, $\tau, \tau' \in \Ts$, and $\comp \in \Comps$, it holds that:
\begin{equation}
\nonumber
    0 \le - \log f(\comp | \tau, \tau'; \bm{\theta}) \le \varsigma \coloneqq \log 4 + d R \Lambda \sqrt{k}.
\end{equation}
\end{lem}

\begin{proof}
Fix $\bm{\theta} = (\mathrm{vec}(\bm{M}), \alpha, \beta, \gamma) \in \mathfrak{B}$ and $\tau, \tau' \in \Ts$, and let $\bm{x} \coloneqq \bm{\phi}(\tau) - \bm{\phi}(\tau')$ and $\bm{\delta} \coloneqq \bm{M} \bm{x}$. By item ($ii$) of Assumption~\ref{asm:bounded}, we have that $\| \bm{x} \|_2 \le \Lambda$. Moreover, since $\bm{M}$ has $dk$ elements in $[-R, R]$, we have that $\| \bm{M} \|_F \le R \sqrt{dk}$, and thus:
\begin{equation}
\label{eq:lossbounded:delta}
    \| \bm{\delta} \|_2 \le \| \bm{M} \|_F \| \bm{x} \|_2 \le R \Lambda \sqrt{dk}.
\end{equation}
Following Equation~\eqref{eq:derivation:dispersion}, let us denote the dispersion of $\bm{\delta}$ as $\widetilde{h}_{\incomprel}(\bm{\delta}) \coloneqq \sqrt{d} \std(\bm{\delta}) = \| \bm{\delta} - \bar{\delta} \bm{1}_d \|_2$, where $\bar{\delta} = \bm{1}_d^\top \bm{\delta} / d$. Expanding the square as in Equation~\eqref{eq:gammarange:square}, we have that $\widetilde{h}_{\incomprel}(\bm{\delta})^2 = \| \bm{\delta} \|_2^2 - (\bm{1}_d^\top \bm{\delta})^2 / d$, from which it follows that, for every $\bm{\delta} \in \mathbb{R}^d$:
\begin{equation}
\label{eq:lossbounded:dispersion}
    0 \le \widetilde{h}_{\incomprel}(\bm{\delta}) \le \| \bm{\delta} \|_2.
\end{equation}
We now bound the four scores of Definition~\ref{defi:mobt}. By the Cauchy-Schwarz inequality and Equation~\eqref{eq:lossbounded:delta}, we have that:
\begin{equation}
\nonumber
    |h_{\prelsinv}(\bm{\delta})| = |h_{\prels}(\bm{\delta})| = \frac{|\bm{1}_d^\top \bm{\delta}|}{2} \le \frac{\sqrt{d} \| \bm{\delta} \|_2}{2} \le \frac{d R \Lambda \sqrt{k}}{2}.
\end{equation}
Moreover, $0 < h_{\equivrel}(\bm{\delta}) = \alpha \le R$, since $\alpha \in (0, R]$. Finally, by Equations~\eqref{eq:lossbounded:delta} and~\eqref{eq:lossbounded:dispersion}, and since $\beta \in [-R, 0]$ and $\gamma \in \Gamma_d^{+}$, so that $0 < \gamma \le \sqrt{d} / (2\sqrt{d-1})$ by Equation~\eqref{eq:gammarange}, we have that:
\begin{equation}
\nonumber
    -R \le \beta \le h_{\incomprel}(\bm{\delta}) = \gamma \widetilde{h}_{\incomprel}(\bm{\delta}) + \beta \le \frac{\sqrt{d}}{2\sqrt{d-1}} R \Lambda \sqrt{dk} = \frac{d R \Lambda \sqrt{k}}{2\sqrt{d-1}} \le \frac{d R \Lambda \sqrt{k}}{2},
\end{equation}
where the last inequality holds since $d \ge 2$. Combining these bounds, and observing that $R \le d R \Lambda \sqrt{k} / 2$ since $d \ge 2$, $\Lambda \ge 1$, and $k \ge 1$, all four scores lie in $[-dR\Lambda\sqrt{k}/2, dR\Lambda\sqrt{k}/2]$, and thus:
\begin{equation}
\nonumber
    \max_{\comp \in \Comps} h_{\comp}(\bm{\delta}) - \min_{\comp \in \Comps} h_{\comp}(\bm{\delta}) \le d R \Lambda \sqrt{k}.
\end{equation}
Now, by Equation~\eqref{eq:mobt_model}, for every $\comp \in \Comps$ we can write:
\begin{equation}
\nonumber
    - \log f(\comp | \tau, \tau'; \bm{\theta}) = \log \left( \sum_{\comp' \in \Comps} \exp \left( h_{\comp'}(\bm{\delta}) - h_{\comp}(\bm{\delta}) \right) \right).
\end{equation}
The sum inside the logarithm is at least $1$, since it comprises the term $\comp' = \comp$, which equals $1$, and three positive terms, yielding the lower bound. Moreover, each of its $|\Comps| = 4$ terms is at most $\exp(\max_{\comp' \in \Comps} h_{\comp'}(\bm{\delta}) - \min_{\comp' \in \Comps} h_{\comp'}(\bm{\delta}))$, which yields the upper bound and concludes the proof.
\end{proof}

\begin{lem}[Lipschitz MOBT Loss]
\label{lem:loss_lipschitz}
Under the assumptions of Lemma~\ref{lem:loss_bounded}, for every $\bm{\theta}, \bm{\theta}' \in \mathfrak{B}$, $\tau, \tau' \in \Ts$, and $\comp \in \Comps$, it holds that:
\begin{equation}
\nonumber
    \left| \log f(\comp | \tau, \tau'; \bm{\theta}) - \log f(\comp | \tau, \tau'; \bm{\theta}') \right| \le L \| \bm{\theta} - \bm{\theta}' \|_2,
    \qquad \text{with} \quad
    L \coloneqq 2 + \Lambda \left( \sqrt{d} + \sqrt{\frac{d}{d-1}} + 2R\sqrt{dk} \right).
\end{equation}
\end{lem}

\begin{proof}
Fix $\bm{\theta} = (\mathrm{vec}(\bm{M}), \alpha, \beta, \gamma) \in \mathfrak{B}$, $\bm{\theta}' = (\mathrm{vec}(\bm{M}'), \alpha', \beta', \gamma') \in \mathfrak{B}$, $\tau, \tau' \in \Ts$, and $\comp \in \Comps$. Let $\bm{x} \coloneqq \bm{\phi}(\tau) - \bm{\phi}(\tau')$, $\bm{\delta} \coloneqq \bm{M} \bm{x}$, and $\bm{\delta}' \coloneqq \bm{M}' \bm{x}$, and denote, for ease of notation, $h_{\comp} \coloneqq h_{\comp}(\bm{\delta})$ computed with parameters $(\gamma, \alpha, \beta)$, and $h'_{\comp} \coloneqq h_{\comp}(\bm{\delta}')$ computed with parameters $(\gamma', \alpha', \beta')$, for every $\comp \in \Comps$. We avoid differentiating the scores, since $h_{\incomprel}$ is not differentiable whenever $\bm{\delta} = \bar{\delta} \bm{1}_d$.

First, we reduce the problem to the scores. Let $\mathrm{LSE}(\bm{v}) \coloneqq \log \sum_{\comp' \in \Comps} \exp(v_{\comp'})$, for every $\bm{v} \in \mathbb{R}^{\Comps}$. Since $\mathrm{LSE}$ is non-decreasing in each argument, and $\mathrm{LSE}(\bm{v} + c \bm{1}) = \mathrm{LSE}(\bm{v}) + c$ for every $c \in \mathbb{R}$, it holds that $\mathrm{LSE}(\bm{v}) \le \mathrm{LSE}(\bm{v}' + \| \bm{v} - \bm{v}' \|_\infty \bm{1}) = \mathrm{LSE}(\bm{v}') + \| \bm{v} - \bm{v}' \|_\infty$ and, symmetrically, $|\mathrm{LSE}(\bm{v}) - \mathrm{LSE}(\bm{v}')| \le \| \bm{v} - \bm{v}' \|_\infty$. Since $- \log f(\comp | \tau, \tau'; \bm{\theta}) = \mathrm{LSE}(\bm{h}) - h_{\comp}$, where $\bm{h} \coloneqq (h_{\comp'})_{\comp' \in \Comps}$ and analogously for $\bm{h}'$, we get:
\begin{equation}
\label{eq:losslipschitz:reduction}
    \left| \log f(\comp | \tau, \tau'; \bm{\theta}) - \log f(\comp | \tau, \tau'; \bm{\theta}') \right| \le \left| \mathrm{LSE}(\bm{h}) - \mathrm{LSE}(\bm{h}') \right| + |h_{\comp} - h'_{\comp}| \le 2 \max_{\comp' \in \Comps} |h_{\comp'} - h'_{\comp'}|.
\end{equation}

Second, we bound the variation of each score. Observe that $\|\bm{\delta} - \bm{\delta}'\|_2 = \|(\bm{M} - \bm{M}') \bm{x}\|_2 \le \Lambda \| \bm{M} - \bm{M}' \|_F$. For direct and inverse preferences, by the Cauchy-Schwarz inequality, we have:
\begin{equation}
\nonumber
    |h_{\prelsinv} - h'_{\prelsinv}| = |h_{\prels} - h'_{\prels}| = \frac{|\bm{1}_d^\top (\bm{\delta} - \bm{\delta}')|}{2} \le \frac{\sqrt{d} \Lambda}{2} \| \bm{M} - \bm{M}' \|_F.
\end{equation}
For indifference, we have that $|h_{\equivrel} - h'_{\equivrel}| = |\alpha - \alpha'|$. For incomparability, we first observe that $\widetilde{h}_{\incomprel}(\bm{\delta}) = \| \bm{\Pi} \bm{\delta} \|_2$ by Equation~\eqref{eq:derivation:dispersion}, where the projection $\bm{\Pi} = \bm{\Pi}_d$ is linear. Thus, by the reverse triangle inequality and Equation~\eqref{eq:lossbounded:dispersion}, we have that $| \widetilde{h}_{\incomprel}(\bm{\delta}) - \widetilde{h}_{\incomprel}(\bm{\delta}') | \le \widetilde{h}_{\incomprel}(\bm{\delta} - \bm{\delta}') \le \| \bm{\delta} - \bm{\delta}' \|_2$. Moreover, $\widetilde{h}_{\incomprel}(\bm{\delta}') \le \| \bm{\delta}' \|_2 \le R \Lambda \sqrt{dk}$ by Equations~\eqref{eq:lossbounded:delta} and~\eqref{eq:lossbounded:dispersion}. Writing $\gamma \widetilde{h}_{\incomprel}(\bm{\delta}) - \gamma' \widetilde{h}_{\incomprel}(\bm{\delta}') = \gamma ( \widetilde{h}_{\incomprel}(\bm{\delta}) - \widetilde{h}_{\incomprel}(\bm{\delta}') ) + (\gamma - \gamma') \widetilde{h}_{\incomprel}(\bm{\delta}')$ and recalling that $\gamma \le \sqrt{d} / (2\sqrt{d-1})$ since $\gamma \in \Gamma_d^{+}$, we get:
\begin{equation}
\nonumber
    |h_{\incomprel} - h'_{\incomprel}| \le \frac{\sqrt{d} \Lambda}{2\sqrt{d-1}} \| \bm{M} - \bm{M}' \|_F + R \Lambda \sqrt{dk} \, |\gamma - \gamma'| + |\beta - \beta'|.
\end{equation}
Finally, since $\| \bm{M} - \bm{M}' \|_F$, $|\alpha - \alpha'|$, $|\beta - \beta'|$, and $|\gamma - \gamma'|$ are all upper bounded by $\| \bm{\theta} - \bm{\theta}' \|_2$, we obtain:
\begin{equation}
\nonumber
    \max_{\comp' \in \Comps} |h_{\comp'} - h'_{\comp'}| \le \max \left\{ \frac{\sqrt{d} \Lambda}{2}, \; 1, \; 1 + \frac{\sqrt{d} \Lambda}{2\sqrt{d-1}} + R \Lambda \sqrt{dk} \right\} \| \bm{\theta} - \bm{\theta}' \|_2 \le \left( 1 + \Lambda \left( \frac{\sqrt{d}}{2} + \frac{\sqrt{d}}{2\sqrt{d-1}} + R \sqrt{dk} \right) \right) \| \bm{\theta} - \bm{\theta}' \|_2.
\end{equation}
Plugging this bound into Equation~\eqref{eq:losslipschitz:reduction} concludes the proof.
\end{proof}

The following result provides the full characterization of observationally equivalent parametrizations, of which Theorem~\ref{thr:identifiability_main} is a direct consequence.

\begin{restatable}[Identifiability up to Invariance]{lem}{identifiability}
\label{lem:identifiability}
    Let $d \ge 1$, let $\mathcal{Q} \in \Delta(\Ts^2)$, and let $\mathcal{G}_d \coloneqq \{ \bm{U} \in \mathbb{R}^{d \times d} : \bm{U}^\top \bm{U} = \bm{I}_d, \, \bm{U} \bm{1}_d = \bm{1}_d \}$ be the group of orthogonal transformations that fix the standard diagonal. Under Assumptions~\ref{asm:features} and~\ref{asm:rich_design}, for every $\bm{\theta} = (\mathrm{vec}(\bm{M}), \alpha, \beta, \gamma)$ and $\bm{\theta}' = (\mathrm{vec}(\bm{M}'), \alpha', \beta', \gamma')$ in $\mathbb{R}^{dk} \times \mathbb{R} \times \mathbb{R} \times (0, +\infty)$, the following statements are equivalent:
    \begin{enumerate}[noitemsep, topsep=-1pt, label=(\roman*)]
    \item $\mathcal{P}_{\bm{\theta}} = \mathcal{P}_{\bm{\theta}'}$;
    \item $\bm{M}^\top \bm{1}_d = \bm{M}'^\top \bm{1}_d$, $\alpha = \alpha'$, $\beta = \beta'$, and $\gamma^2 \bm{M}^\top \bm{\Pi} \bm{M} = \gamma'^2 \bm{M}'^\top \bm{\Pi} \bm{M}'$;
    \item there exist $\bm{U} \in \mathcal{G}_d$ and $c > 0$ such that $\bm{M}' = T_c(\bm{U} \bm{M})$, $\gamma' = \gamma / c$, $\alpha' = \alpha$, and $\beta' = \beta$,
    \end{enumerate}
    where $T_c$ is defined as in Remark~\ref{rem:scale_invariance} and $\bm{\Pi} = \bm{\Pi}_d$ is defined as in Section~\ref{sec:intro}.
\end{restatable}

\begin{proof}
Throughout the proof, we denote $\bm{m} \coloneqq \bm{M}^\top \bm{1}_d$ and $\bm{m}' \coloneqq \bm{M}'^\top \bm{1}_d$. Moreover, for every $\bm{x} \in \mathbb{R}^k$ and $\comp \in \Comps$, we denote as $F_{\bm{\theta}}(\comp | \bm{x})$ the MOBT model of Definition~\ref{defi:mobt} with parameters $(\gamma, \alpha, \beta)$ evaluated at $\bm{\delta} = \bm{M} \bm{x}$, so that $f(\comp | \tau, \tau'; \bm{\theta}) = F_{\bm{\theta}}(\comp | \bm{\phi}(\tau) - \bm{\phi}(\tau'))$ for every $\tau, \tau' \in \Ts$, and analogously for $\bm{\theta}'$. By Equation~\eqref{eq:derivation:dispersion}, the incomparability score can be written as $h_{\incomprel}(\bm{M} \bm{x}) = \gamma \| \bm{\Pi} \bm{M} \bm{x} \|_2 + \beta$. We recall that $\bm{\Pi}$ is symmetric and satisfies $\bm{\Pi}^2 = \bm{\Pi}$, $\bm{1}_d^\top \bm{\Pi} = \bm{0}_d^\top$, and $\bm{I}_d = \bm{A}_d + \bm{\Pi}$. Finally, for every $\bm{U} \in \mathcal{G}_d$, we have that $\bm{U}^\top \bm{1}_d = \bm{U}^\top \bm{U} \bm{1}_d = \bm{1}_d$, and thus:
\begin{equation}
\label{eq:identifiability:commute}
    \bm{\Pi} \bm{U} = \bm{U} - \frac{1}{d} \bm{1}_d (\bm{U}^\top \bm{1}_d)^\top = \bm{U} - \frac{1}{d} \bm{1}_d \bm{1}_d^\top = \bm{U} - \frac{1}{d} (\bm{U} \bm{1}_d) \bm{1}_d^\top = \bm{U} \bm{\Pi}.
\end{equation}
We prove the result by showing that ($iii$) implies ($i$), ($i$) implies ($ii$), and ($ii$) implies ($iii$).

\paragraph{($iii$) Implies ($i$).}
Let $\bm{U} \in \mathcal{G}_d$ and $c > 0$ be such that $\bm{M}' = T_c(\bm{U} \bm{M})$, $\gamma' = \gamma / c$, $\alpha' = \alpha$, and $\beta' = \beta$. Since $\bm{1}_d^\top \bm{1}_d = d$ and $\bm{1}_d^\top \bm{\Pi} = \bm{0}_d^\top$, we have that $\bm{1}_d^\top \bm{M}' = \bm{1}_d^\top \bm{U} \bm{M} = (\bm{U}^\top \bm{1}_d)^\top \bm{M} = \bm{1}_d^\top \bm{M}$. Moreover, since $\bm{\Pi} \bm{1}_d = \bm{0}_d$ and $\bm{\Pi}^2 = \bm{\Pi}$, by Equation~\eqref{eq:identifiability:commute}, we have that $\bm{\Pi} \bm{M}' = c \, \bm{\Pi} \bm{U} \bm{M} = c \, \bm{U} \bm{\Pi} \bm{M}$. Therefore, for every $\bm{x} \in \mathbb{R}^k$, we have that $h_{\prelsinv}(\bm{M}' \bm{x}) = \bm{1}_d^\top \bm{M}' \bm{x} / 2 = h_{\prelsinv}(\bm{M} \bm{x})$, and analogously for $h_{\prels}$, the indifference scores are both equal to $\alpha' = \alpha$, and, since $\bm{U}$ is orthogonal:
\begin{equation}
\nonumber
    \gamma' \| \bm{\Pi} \bm{M}' \bm{x} \|_2 + \beta' = \frac{\gamma}{c} \, c \, \| \bm{U} \bm{\Pi} \bm{M} \bm{x} \|_2 + \beta = \gamma \| \bm{\Pi} \bm{M} \bm{x} \|_2 + \beta.
\end{equation}
Hence, $F_{\bm{\theta}'}(\cdot | \bm{x}) = F_{\bm{\theta}}(\cdot | \bm{x})$ for every $\bm{x} \in \mathbb{R}^k$, which implies $f(\cdot | \tau, \tau'; \bm{\theta}') = f(\cdot | \tau, \tau'; \bm{\theta})$ for every $\tau, \tau' \in \Ts$, and thus $\mathcal{P}_{\bm{\theta}'} = \mathcal{P}_{\bm{\theta}}$, since the two distributions share the same marginal $\mathcal{Q}$.

\paragraph{($i$) Implies ($ii$).}
Since $\mathcal{P}_{\bm{\theta}}$ and $\mathcal{P}_{\bm{\theta}'}$ share the same marginal $\mathcal{Q}$ over trajectory pairs, $\mathcal{P}_{\bm{\theta}} = \mathcal{P}_{\bm{\theta}'}$ implies that, for every $\comp \in \Comps$, $f(\comp | \tau, \tau'; \bm{\theta}) = f(\comp | \tau, \tau'; \bm{\theta}')$ for $\mathcal{Q}$-almost every $(\tau, \tau')$. Let $\mathcal{Q}_{\bm{x}}$ be the distribution of $\bm{\phi}(\tau) - \bm{\phi}(\tau')$ with $(\tau, \tau') \sim \mathcal{Q}$, whose support is $\mathcal{X}$. Then, $F_{\bm{\theta}}(\comp | \cdot) = F_{\bm{\theta}'}(\comp | \cdot)$ holds $\mathcal{Q}_{\bm{x}}$-almost everywhere. Since the scores and the softmax are continuous, the set $\{ \bm{x} \in \mathbb{R}^k : F_{\bm{\theta}}(\comp | \bm{x}) \neq F_{\bm{\theta}'}(\comp | \bm{x}) \}$ is open and $\mathcal{Q}_{\bm{x}}$-null, and it does not intersect $\mathcal{X}$, since every open neighborhood of a point of the support has positive measure. Hence, $F_{\bm{\theta}}(\cdot | \bm{x}) = F_{\bm{\theta}'}(\cdot | \bm{x})$ for every $\bm{x} \in \mathcal{X}$.

Since the scores are finite, all the probabilities are strictly positive, and the log-ratios of Equation~\eqref{eq:mobt_model} with respect to the inverse preference yield, for every $\bm{x} \in \mathcal{X}$:
\begin{equation}
\nonumber
    \log \frac{F_{\bm{\theta}}(\prelsinv | \bm{x})}{F_{\bm{\theta}}(\prels | \bm{x})} = \bm{m}^\top \bm{x},
    \qquad
    \log \frac{F_{\bm{\theta}}(\equivrel | \bm{x})}{F_{\bm{\theta}}(\prels | \bm{x})} = \alpha + \frac{\bm{m}^\top \bm{x}}{2},
    \qquad
    \log \frac{F_{\bm{\theta}}(\incomprel | \bm{x})}{F_{\bm{\theta}}(\prels | \bm{x})} = \gamma \| \bm{\Pi} \bm{M} \bm{x} \|_2 + \beta + \frac{\bm{m}^\top \bm{x}}{2},
\end{equation}
and analogously for $\bm{\theta}'$. Equating the log-ratios induced by $\bm{\theta}$ and $\bm{\theta}'$, we get that, for every $\bm{x} \in \mathcal{X}$:
\begin{equation}
\label{eq:identifiability:logratios}
    \bm{m}^\top \bm{x} = \bm{m}'^\top \bm{x},
    \qquad
    \alpha = \alpha',
    \qquad
    \gamma \| \bm{\Pi} \bm{M} \bm{x} \|_2 + \beta = \gamma' \| \bm{\Pi} \bm{M}' \bm{x} \|_2 + \beta',
\end{equation}
where the second and third equalities follow by plugging the first one, and where we used that $\mathcal{X} \neq \emptyset$, being the support of a probability distribution.

First, we show that $\mathrm{span}(\mathcal{X}) = \mathbb{R}^k$. By contradiction, let $\bm{y} \in \mathbb{R}^k \setminus \{\bm{0}_k\}$ be orthogonal to every $\bm{x} \in \mathcal{X}$. Then, $\bm{y} \bm{y}^\top \in \mathbb{S}^k$ satisfies $\mathrm{tr}(\bm{y} \bm{y}^\top \bm{x} \bm{x}^\top) = (\bm{y}^\top \bm{x})^2 = 0$ for every $\bm{x} \in \mathcal{X}$, \ie it is orthogonal, w.r.t.\ the Frobenius inner product, to $\mathrm{span} \{ \bm{x} \bm{x}^\top : \bm{x} \in \mathcal{X} \} = \mathbb{S}^k$ by item ($i$) of Assumption~\ref{asm:rich_design}, and thus to itself, which implies $\bm{y} \bm{y}^\top = \bm{0}$, a contradiction. Therefore, the first equality of Equation~\eqref{eq:identifiability:logratios} yields $\bm{m} = \bm{m}'$.

Second, we show that $\beta = \beta'$. Let $a(\bm{x}) \coloneqq \gamma \| \bm{\Pi} \bm{M} \bm{x} \|_2$ and $a'(\bm{x}) \coloneqq \gamma' \| \bm{\Pi} \bm{M}' \bm{x} \|_2$ for every $\bm{x} \in \mathbb{R}^k$, which are absolutely homogeneous of degree one, \ie $a(t \bm{x}) = |t| a(\bm{x})$ and $a'(t \bm{x}) = |t| a'(\bm{x})$ for every $t \in \mathbb{R}$, and let $b \coloneqq \beta' - \beta$, so that the third equality of Equation~\eqref{eq:identifiability:logratios} reads $a(\bm{x}) - a'(\bm{x}) = b$ for every $\bm{x} \in \mathcal{X}$. Let $\bm{x}_0 \in \mathcal{X}$ and $t \in \mathbb{R}$ with $|t| \neq 1$ be such that $t \bm{x}_0 \in \mathcal{X}$, which exist by item ($ii$) of Assumption~\ref{asm:rich_design}. Then, $b = a(t \bm{x}_0) - a'(t \bm{x}_0) = |t| (a(\bm{x}_0) - a'(\bm{x}_0)) = |t| b$, \ie $(1 - |t|) b = 0$, and thus $b = 0$, since $|t| \neq 1$. We remark that the requirement $|t| \neq 1$ cannot be relaxed to $t \neq 1$, since $a$ and $a'$ are even, so that the pair $\bm{x}_0, -\bm{x}_0$ provides no information on $b$.

Third, we show that $\gamma^2 \bm{M}^\top \bm{\Pi} \bm{M} = \gamma'^2 \bm{M}'^\top \bm{\Pi} \bm{M}'$. Since $b = 0$, we have that $a(\bm{x}) = a'(\bm{x})$ for every $\bm{x} \in \mathcal{X}$, and, since both are non-negative, $a(\bm{x})^2 = a'(\bm{x})^2$. Recalling that $\bm{\Pi}^\top \bm{\Pi} = \bm{\Pi}$, we have that $a(\bm{x})^2 = \gamma^2 \bm{x}^\top \bm{M}^\top \bm{\Pi} \bm{M} \bm{x} = \mathrm{tr}(\gamma^2 \bm{M}^\top \bm{\Pi} \bm{M} \bm{x} \bm{x}^\top)$, and analogously for $a'$. Thus, the symmetric matrix $\bm{D} \coloneqq \gamma^2 \bm{M}^\top \bm{\Pi} \bm{M} - \gamma'^2 \bm{M}'^\top \bm{\Pi} \bm{M}' \in \mathbb{S}^k$ satisfies $\mathrm{tr}(\bm{D} \bm{x} \bm{x}^\top) = 0$ for every $\bm{x} \in \mathcal{X}$. By item ($i$) of Assumption~\ref{asm:rich_design}, $\bm{D}$ is orthogonal to $\mathbb{S}^k$, and thus to itself, which implies $\bm{D} = \bm{0}$.

\paragraph{($ii$) Implies ($iii$).}
Let $\bm{E} \coloneqq \gamma \bm{\Pi} \bm{M}$ and $\bm{E}' \coloneqq \gamma' \bm{\Pi} \bm{M}'$. Since $\bm{1}_d^\top \bm{\Pi} = \bm{0}_d^\top$, the columns of both matrices belong to $\mathcal{V} \coloneqq \{ \bm{v} \in \mathbb{R}^d : \bm{1}_d^\top \bm{v} = 0 \}$, whose orthogonal complement is spanned by $\bm{1}_d$. Moreover, $\bm{E}^\top \bm{E} = \gamma^2 \bm{M}^\top \bm{\Pi} \bm{M} = \gamma'^2 \bm{M}'^\top \bm{\Pi} \bm{M}' = \bm{E}'^\top \bm{E}'$.
We now construct $\bm{U} \in \mathcal{G}_d$ such that $\bm{E}' = \bm{U} \bm{E}$. Let $\mathcal{R}, \mathcal{R}' \subseteq \mathcal{V}$ be the column spaces of $\bm{E}$ and $\bm{E}'$, respectively. For every $\bm{z} \in \mathbb{R}^k$, we have that $\| \bm{E} \bm{z} \|_2^2 = \bm{z}^\top \bm{E}^\top \bm{E} \bm{z} = \bm{z}^\top \bm{E}'^\top \bm{E}' \bm{z} = \| \bm{E}' \bm{z} \|_2^2$. Hence, the map $\bm{E} \bm{z} \mapsto \bm{E}' \bm{z}$ is well defined, since $\bm{E} \bm{z} = \bm{E} \bm{z}'$ implies $\| \bm{E}' (\bm{z} - \bm{z}') \|_2 = \| \bm{E} (\bm{z} - \bm{z}') \|_2 = 0$, and it is a linear isometry from $\mathcal{R}$ onto $\mathcal{R}'$. In particular, $\dim \mathcal{R} = \dim \mathcal{R}'$, so that $\mathcal{V} \cap \mathcal{R}^\perp$ and $\mathcal{V} \cap \mathcal{R}'^\perp$ have the same dimension. Let $\bm{U}$ be the linear map that acts as this isometry on $\mathcal{R}$, as any linear isometry from $\mathcal{V} \cap \mathcal{R}^\perp$ onto $\mathcal{V} \cap \mathcal{R}'^\perp$ on the former subspace, and as the identity on $\mathcal{V}^\perp = \mathrm{span}(\bm{1}_d)$. Since $\bm{U}$ maps each summand of the orthogonal decomposition $\mathbb{R}^d = \mathcal{R} \oplus (\mathcal{V} \cap \mathcal{R}^\perp) \oplus \mathcal{V}^\perp$ isometrically onto the corresponding summand of the orthogonal decomposition $\mathbb{R}^d = \mathcal{R}' \oplus (\mathcal{V} \cap \mathcal{R}'^\perp) \oplus \mathcal{V}^\perp$, it is orthogonal. Moreover, $\bm{U} \bm{E} \bm{z} = \bm{E}' \bm{z}$ for every $\bm{z} \in \mathbb{R}^k$, \ie $\bm{U} \bm{E} = \bm{E}'$, and $\bm{U} \bm{1}_d = \bm{1}_d$, \ie $\bm{U} \in \mathcal{G}_d$. Let $c \coloneqq \gamma / \gamma' > 0$, so that $\gamma' = \gamma / c$ and $\bm{\Pi} \bm{M}' = \bm{E}' / \gamma' = \bm{U} \bm{E} / \gamma' = c \, \bm{U} \bm{\Pi} \bm{M}$. Then, since $\bm{m} = \bm{m}'$, we have:
\begin{equation}
\nonumber
    \bm{M}' = \bm{A}_d \bm{M}' + \bm{\Pi} \bm{M}' = \frac{1}{d} \bm{1}_d \bm{m}^\top + c \, \bm{U} \bm{\Pi} \bm{M} = \frac{1}{d} \bm{1}_d (\bm{U}^\top \bm{1}_d)^\top \bm{M} + c \, \bm{\Pi} \bm{U} \bm{M} = T_c(\bm{U} \bm{M}),
\end{equation}
where the third equality follows from $\bm{U}^\top \bm{1}_d = \bm{1}_d$ and Equation~\eqref{eq:identifiability:commute}. Together with $\alpha = \alpha'$ and $\beta = \beta'$, this concludes the proof.
\end{proof}

\begin{lem}[Positive Orthant around the Standard Diagonal]
\label{lem:orthant}
Let $d \ge 2$ and $\rho_{\mathrm{in}} \coloneqq 1 / \sqrt{d(d-1)}$. For every $\bm{y} \in \mathbb{R}^d$, letting $s \coloneqq \bm{1}_d^\top \bm{y}$, the following statements hold:
\begin{enumerate}[noitemsep, topsep=-1pt, label=(\roman*)]
\item $y_i \ge s / d - \sqrt{(d-1)/d} \, \| \bm{\Pi} \bm{y} \|_2$ for every $i \in \dsb{d}$;
\item if $\bm{y} \in \mathbb{R}^d_{\ge 0}$, then $\| \bm{\Pi} \bm{y} \|_2 \le \sqrt{(d-1)/d} \, s$;
\item if $\| \bm{\Pi} \bm{y} \|_2 \le \rho_{\mathrm{in}} s$, then $\bm{y} \in \mathbb{R}^d_{\ge 0}$, and if $\| \bm{\Pi} \bm{y} \|_2 < \rho_{\mathrm{in}} s$, then all the elements of $\bm{y}$ are strictly positive;
\item if $d = 3$ and $\| \bm{\Pi} \bm{y} \|_2 \le \sqrt{2/3} \, s$, then $\bm{y}$ has at most one negative element;
\item $\Delta_{d-1} = \mathcal{W}_{\sqrt{(d-1)/d}}$, and, for every $r \in [0, \rho_{\mathrm{in}}]$, $\mathcal{W}_r = \{ \bm{w} \in \mathbb{R}^d : \bm{1}_d^\top \bm{w} = 1 \wedge \| \bm{w} - \bm{1}_d / d \|_2 \le r \}$; in particular, if $d = 2$, then $\Delta_1 = \mathcal{W}_{\rho_{\mathrm{in}}}$.
\end{enumerate}
\end{lem}

\begin{proof}
First, we observe that $\bm{\Pi} \bm{e}_i = \bm{e}_i - \bm{1}_d / d$ has $i$-th element $(d-1)/d$, other elements $-1/d$, and norm $\| \bm{\Pi} \bm{e}_i \|_2 = \sqrt{(d-1)^2/d^2 + (d-1)/d^2} = \sqrt{(d-1)/d}$. Let $\bm{z} \coloneqq \bm{\Pi} \bm{y}$, so that $\bm{y} = (s/d) \bm{1}_d + \bm{z}$ with $\bm{1}_d^\top \bm{z} = 0$.

To prove result ($i$), we observe that, since $\bm{\Pi}$ is symmetric and $\bm{\Pi} \bm{z} = \bm{z}$, by the Cauchy-Schwarz inequality, $z_i = \bm{e}_i^\top \bm{\Pi} \bm{z} = (\bm{\Pi} \bm{e}_i)^\top \bm{z} \ge - \| \bm{\Pi} \bm{e}_i \|_2 \| \bm{z} \|_2 = - \sqrt{(d-1)/d} \, \| \bm{z} \|_2$, and thus $y_i = s/d + z_i \ge s/d - \sqrt{(d-1)/d} \, \| \bm{z} \|_2$.

To prove result ($ii$), we observe that, since $(s/d) \bm{1}_d$ and $\bm{z}$ are orthogonal, $\| \bm{y} \|_2^2 = s^2/d + \| \bm{z} \|_2^2$. If $\bm{y} \in \mathbb{R}^d_{\ge 0}$, then $\| \bm{y} \|_2 \le \| \bm{y} \|_1 = s$, and thus $\| \bm{z} \|_2^2 \le s^2 - s^2/d = (d-1) s^2 / d$, with $s \ge 0$.

To prove result ($iii$), we observe that $\| \bm{z} \|_2 \le \rho_{\mathrm{in}} s$ implies $s \ge 0$, and, by result ($i$), $y_i \ge s/d - \sqrt{(d-1)/d} \, \rho_{\mathrm{in}} s = s/d - s/d = 0$ for every $i \in \dsb{d}$. If the inequality is strict, then $y_i > 0$ for every $i \in \dsb{d}$ by the same argument.

To prove result ($iv$), let $d = 3$. If $s \le 0$, then $\| \bm{z} \|_2 \le \sqrt{2/3} \, s \le 0$ implies $\bm{z} = \bm{0}_3$ and $s = 0$, \ie $\bm{y} = \bm{0}_3$, which has no negative elements. If $s > 0$, assume by contradiction that two elements of $\bm{y}$ are negative, \wwlog $y_1, y_2 < 0$. Then, $y_3 = s - y_1 - y_2 > s$, and thus $z_1 = y_1 - s/3 < -s/3$, $z_2 < -s/3$, and $z_3 = y_3 - s/3 > 2s/3$. Therefore, $\| \bm{z} \|_2^2 > s^2/9 + s^2/9 + 4 s^2/9 = 2 s^2/3$, which contradicts $\| \bm{z} \|_2 \le \sqrt{2/3} \, s$.

To prove result ($v$), we observe that, for every $\bm{w} \in \mathbb{R}^d$ with $\bm{1}_d^\top \bm{w} = 1$, we have $\bm{\Pi} \bm{w} = \bm{w} - \bm{1}_d / d$, and thus $\| \bm{w} - \bm{1}_d / d \|_2 = \| \bm{\Pi} \bm{w} \|_2$. By definition, $\mathcal{W}_{\sqrt{(d-1)/d}} \subseteq \Delta_{d-1}$. Conversely, every $\bm{w} \in \Delta_{d-1}$ belongs to $\mathbb{R}^d_{\ge 0}$ and satisfies $\bm{1}_d^\top \bm{w} = 1$, so that result ($ii$) yields $\| \bm{w} - \bm{1}_d / d \|_2 \le \sqrt{(d-1)/d}$, \ie $\bm{w} \in \mathcal{W}_{\sqrt{(d-1)/d}}$. Hence, $\Delta_{d-1} = \mathcal{W}_{\sqrt{(d-1)/d}}$. Moreover, let $r \in [0, \rho_{\mathrm{in}}]$. By definition, every $\bm{w} \in \mathcal{W}_r$ satisfies $\bm{1}_d^\top \bm{w} = 1$ and $\| \bm{w} - \bm{1}_d / d \|_2 \le r$. Conversely, every $\bm{w} \in \mathbb{R}^d$ with $\bm{1}_d^\top \bm{w} = 1$ and $\| \bm{w} - \bm{1}_d / d \|_2 \le r$ satisfies $\| \bm{\Pi} \bm{w} \|_2 \le \rho_{\mathrm{in}} = \rho_{\mathrm{in}} \bm{1}_d^\top \bm{w}$, so that result ($iii$) yields $\bm{w} \in \mathbb{R}^d_{\ge 0}$, \ie $\bm{w} \in \Delta_{d-1}$, and thus $\bm{w} \in \mathcal{W}_r$. Finally, if $d = 2$, then $\rho_{\mathrm{in}} = \sqrt{(d-1)/d} = 1/\sqrt{2}$, and thus $\Delta_1 = \mathcal{W}_{\sqrt{(d-1)/d}} = \mathcal{W}_{\rho_{\mathrm{in}}}$.
\end{proof}

\section{Illustrative Example}
\label{apx:example}

In this appendix, we provide an illustrative example, eliciting feedback from a synthetic expert in a GridWorld environment, with the goal of highlighting the findings of this work.

We consider a custom $5 \times 5$ gridworld environment, with up to three objectives, graphically represented in Figure~\ref{fig:gw_rewards}: (a) reach the bottom-right \emph{target} cell, (b) reach the top row, and (c) avoid being adjacent to the obstacles.
Clearly, in such an environment, incomparability arises in several occasions, e.g., since rewards are collected at every step, time spent at the target cannot be spent in the top row, and vice versa.
If $d=2$ case, only objectives (a) and (b) are considered, and as such obstacles are removed as they are irrelevant. For completeness, we represent the two objectives in Figure~\ref{fig:gw_rewards_2d}: (a) reach the bottom-right \emph{target} cell and (b) reach the top row.

We consider a dataset of 5000 triples, with trajectories generated according to a fixed stochastic exploration policy, and comparisons generated by a synthetic expert under the MOBT model and Assumption~\ref{asm:features}, with one-hot features of the visited cell, i.e., $k = 25$.
We consider $\gamma_0 = \gamma^{*} = 1/2$, so that $c=1$, allowing us to isolate the effect of the rotation of the learned reward.
We recover $\widehat{\bm{M}}$ after optimizing the NLL using the ADAM optimizer \citep{kingma2015adam}. Since ADAM is not guaranteed to return a local minimum, this example illustrates the identifiability properties of Section~\ref{sec:identifiability}, rather than verifying the bounds of Theorems~\ref{thr:sample_complexity} and~\ref{thr:sample_complexity_local}.

First, considering the $d=2$ case, we report the recovered reward model in Figure~\ref{fig:gw_rewards_2d}. In this case, since $\mathcal{G}_2$ contains only the identity and the swap of the two objectives, Theorem~\ref{thr:identifiability_main} implies that, with $\gamma_0 = \gamma^{*}$, the reward is identified up to per-objective translations and a relabeling of the objectives, the latter being resolved by the simple alignment of the rows described in Appendix~\ref{apx:implementation}. The \quotes{reach the target} objective is recovered almost exactly, with a positive reward only in the location of the target and an almost constant reward elsewhere (Figure~\ref{fig:gw_rewards_2d:target_est}). The \quotes{reach top row} objective presents the same row-wise structure as the true one, with a reward that decreases from the top to the bottom row, up to a translation (Figure~\ref{fig:gw_rewards_2d:top_row_est}). Notably, $d=2$ with $c=1$ is the case in which Theorem~\ref{thr:pareto} guarantees that, for every parametrization observationally equivalent to the expert's, scalarizing the learned reward over the whole simplex recovers the scalarized front of the true reward, i.e., $\Pi^{*}(\widehat{\bm{M}}; \Delta_1) = \Pi^{*}(\bm{M}^{*}; \Delta_1)$.

\ins{\begin{figure*}[t]
    \centering
    \begin{subfigure}[t]{0.24\textwidth}
        \centering
        \resizebox{!}{3cm}{
            \includegraphics[]{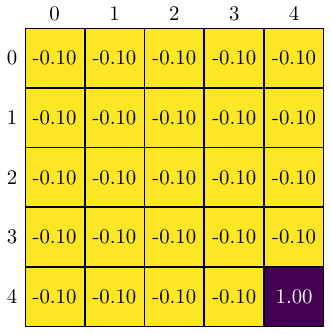}
        }
        \caption{\quotes{Reach the target}.}
        \label{fig:gw_rewards_2d:target}
    \end{subfigure}
    \hfill
    \begin{subfigure}[t]{0.24\textwidth}
        \centering
        \resizebox{!}{3cm}{
            \includegraphics[]{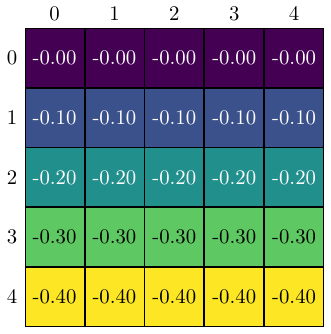}
        }
        \caption{\quotes{Reach top row}.}
        \label{fig:gw_rewards_2d:top_row}
    \end{subfigure}
    \hfill
    \begin{subfigure}[t]{0.24\textwidth}
        \centering
        \resizebox{!}{3cm}{
            \includegraphics[]{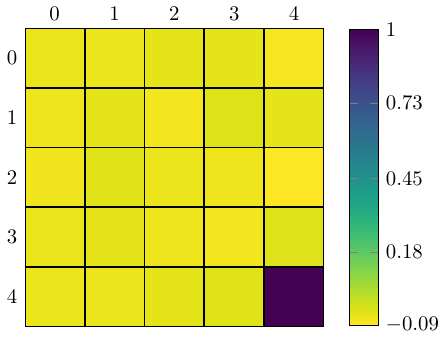}
        }
        \caption{\quotes{Reach the target}.}
        \label{fig:gw_rewards_2d:target_est}
    \end{subfigure}
    \hfill
    \begin{subfigure}[t]{0.24\textwidth}
        \centering
        \resizebox{!}{3cm}{
            \includegraphics[]{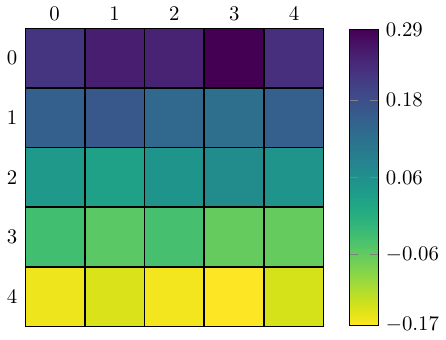}
        }
        \caption{\quotes{Reach top row}.}
        \label{fig:gw_rewards_2d:top_row_est}
    \end{subfigure}
    \caption{GridWorld true (a,b) and estimate (c,d; mean over 10 runs, 5000 samples) reward matrices for $d=2$.}
    \label{fig:gw_rewards_2d}
\end{figure*}
}

Then, we report the recovered reward model in the three-dimensional case in Figure~\ref{fig:gw_rewards_est}. Some comments are in order. First, from Theorem~\ref{thr:identifiability_main} and the discussion of Assumption~\ref{asm:rich_design} it follows that we can identify a reward up to rotations and per-objective translations, as we are in the $d=3$ case. Indeed, we observe effects consistent with such a rotation in the reconstructed matrices, as part of the obstacle penalization \emph{leaks} into Figures~\ref{fig:gw_rewards_est:target} and~\ref{fig:gw_rewards_est:top_row}, and Figure~\ref{fig:gw_rewards_est:top_row} presents a positive reward in the location of the target. Notably, differently from the $d=3$ case, the $d=2$ case presents no such \quotes{leakage}. Second, we observe that the objective-wise ranges of the reconstructed rewards are somewhat similar to those of the true rewards, up to translation. This is not unexpected as, similarly to the BT model, the probability of feedback directly depends on the utility differences, and thus on their scale.

\ins{\begin{figure*}[t]
    \centering
    \begin{subfigure}[t]{0.3\textwidth}
        \centering
        \resizebox{!}{3cm}{
            \includegraphics[]{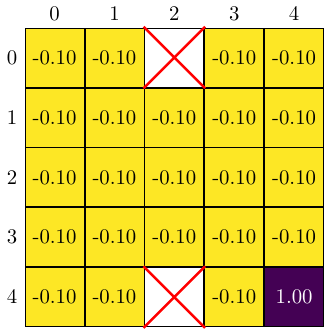}
        }
        \caption{\quotes{Reach the target} objective.}
        \label{fig:gw_rewards:target}
    \end{subfigure}
    \hfill
    \begin{subfigure}[t]{0.3\textwidth}
        \centering
        \resizebox{!}{3cm}{
            \includegraphics[]{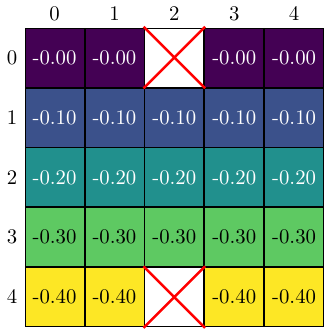}
        }
        \caption{\quotes{Reach top row} objective.}
        \label{fig:gw_rewards:top_row}
    \end{subfigure}
    \hfill
    \begin{subfigure}[t]{0.3\textwidth}
        \centering
        \resizebox{!}{3cm}{
            \includegraphics[]{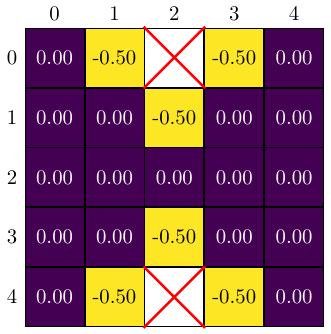}
        }
        \caption{\quotes{Avoid obstacles} objective.}
        \label{fig:gw_rewards:obstacles}
    \end{subfigure}
    \caption{GridWorld true reward matrices for $d=3$.}
    \label{fig:gw_rewards}
\end{figure*}}

\ins{\begin{figure*}[t]
    \centering
    \hspace{-0.3cm}
    \begin{subfigure}[t]{0.3\textwidth}
        \centering
        \resizebox{!}{3cm}{
            \includegraphics[]{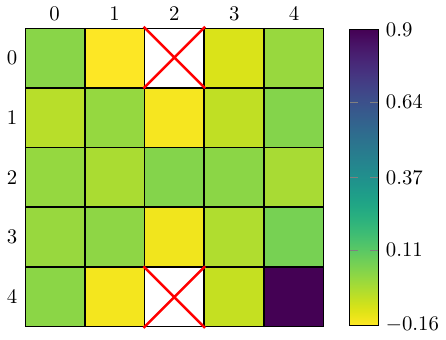}
        }
        \caption{\quotes{Reach the target} objective.}
        \label{fig:gw_rewards_est:target}
    \end{subfigure}
    \hfill
    \begin{subfigure}[t]{0.3\textwidth}
        \centering
        \resizebox{!}{3cm}{
            \includegraphics[]{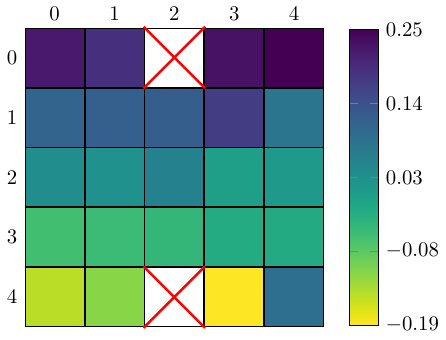}
        }
        \caption{\quotes{Reach top row} objective.}
        \label{fig:gw_rewards_est:top_row}
    \end{subfigure}
    \hfill
    \begin{subfigure}[t]{0.3\textwidth}
        \centering
        \resizebox{!}{3cm}{
            \includegraphics[]{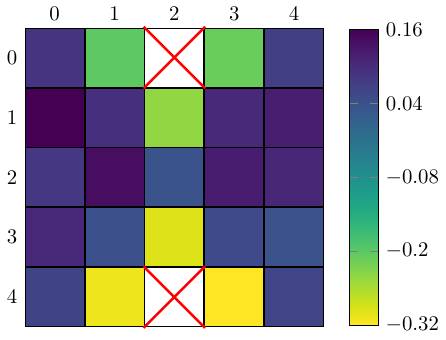}
        }
        \caption{\quotes{Avoid obstacles} objective.}
        \label{fig:gw_rewards_est:obstacles}
    \end{subfigure}
    \caption{Estimated GridWorld reward matrices for $d=3$ (mean over 10 runs, 5000 samples).}
    \label{fig:gw_rewards_est}
\end{figure*}}

Finally, this illustrative example shows that the MOBT model presents properties highlighted by the theory. A thorough empirical evaluation considering multiple environments, the effects of different feature representations, and misspecified or irrational experts is left as a future work.

\section{Implementation Details}
\label{apx:implementation}

In this appendix, we report the implementation details of the illustrative example of Appendix~\ref{apx:example}. The code is provided in the supplementary material.

\textbf{Environment.}~~We consider a $5 \times 5$ grid, in which the agent starts in the bottom-left cell and, at each step, can move in one of the four cardinal directions or stay in place, with deterministic transitions. Two obstacles are placed in the middle cells of the top and bottom rows, and the actions that would move the agent outside the grid or into an obstacle leave it in place. The episode does not terminate when an objective is reached. At each step, the agent receives a three-dimensional reward that depends on the cell it reaches (Figure~\ref{fig:gw_rewards}): ($i$) the \quotes{reach the target} objective yields $1$ in the bottom-right cell and $-0.1$ elsewhere, ($ii$) the \quotes{reach top row} objective yields $-0.1$ times the index of the row, counted from the top starting from $0$, and ($iii$) the \quotes{avoid obstacles} objective yields $-0.5$ in the cells adjacent to an obstacle and $0$ elsewhere. In the $d=2$ case, the obstacles and the \quotes{avoid obstacles} objective are removed, so that $\bm{M}^{*} \in \mathbb{R}^{2 \times 25}$, and all other settings are unchanged (Figure~\ref{fig:gw_rewards_2d}).

\textbf{Features and Utility.}~~The feature map $\bm{\varphi}$ is the one-hot encoding of the cell reached at each step, so that $\bm{\phi}(\tau) \in \mathbb{R}^{25}$ counts the visits to each cell along a trajectory of $H = 20$ steps, and the utility of the expert is $\bm{u}(\tau) = \bm{M}^{*} \bm{\phi}(\tau)$, where the rows of $\bm{M}^{*} \in \mathbb{R}^{3 \times 25}$ are the per-cell rewards of the three objectives, satisfying Assumption~\ref{asm:features}. As discussed after Assumption~\ref{asm:rich_design}, since $\bm{1}_{25}^\top \bm{\phi}(\tau) = H$ for every trajectory, each objective is identified only up to a constant per-step shift. Moreover, the obstacle cells are never visited, so that the corresponding columns of $\bm{M}$ do not affect the comparisons.

\textbf{Data Collection.}~~The trajectories are collected by a single exploration policy, trained once and shared across runs, which ignores the rewards of the environment. It is trained for $10^6$ steps, in episodes of $25$ steps, via tabular Q-learning with learning rate $0.1$ and discount factor $0.95$, on the intrinsic reward $1/\sqrt{n}$, where $n$ is the number of visits to the reached cell at the same step of the episode, and $-1$ for the actions that leave the agent in place. The policy selects its actions via a softmax over the Q-values. In each run, we collect $2000$ trajectories of $H = 20$ steps, in which, at each step, the agent takes a random action with probability $0.3$ to further guarantee randomness. Then, we sample $5000$ pairs of trajectories uniformly at random, and query the synthetic expert once for each pair, obtaining $5000$ comparison triples.

\textbf{Synthetic Expert.}~~The synthetic expert provides comparisons sampled from the MOBT model with $\alpha^{*} = 1$ and $\beta^{*} = 0$.

\textbf{Estimation.}~~In each run, we hold out $20\%$ of the comparisons as a test set, and use $20\%$ of the remaining ones as a validation set for early stopping. We estimate $\bm{M}$ by minimizing the empirical NLL with ADAM \citep{kingma2015adam}, with learning rate $10^{-3}$, batch size $512$, gradient clipping at norm $1$, and at most $5000$ epochs, stopping when the validation NLL does not improve for $100$ consecutive epochs. Parameter $\gamma$ is fixed to that of the expert, i.e., $\gamma_0 = \gamma^{*}$, allowing us to isolate the effect of the rotation of the rewards. Matrix $\bm{M}$ is initialized with the default initialization of PyTorch \citep{paszke2019pytorch}, and no constraint is enforced on its elements.

\textbf{Aggregation.}~~We executed $10$ seeded runs, with seeds $0, \ldots, 9$, each with newly collected trajectories and comparisons. Since the estimated matrices are identified only up to the transformations of Theorem~\ref{thr:identifiability_main}, in each run, before averaging over the runs (Figures~\ref{fig:gw_rewards_2d} and~\ref{fig:gw_rewards_est}), we reorder the rows of $\widehat{\bm{M}}$ to best match those of $\bm{M}^{*}$, minimizing the sum of the relative Euclidean distances between matched rows via the Hungarian algorithm \citep{kuhn1955hungarian}. Notably, this resolves only the relabeling of the objectives, and not the rotations of Theorem~\ref{thr:identifiability_main}.

\textbf{Computing Infrastructure.}~~All runs are performed on CPU, on a laptop with an Apple M4 processor and 16GB of RAM. The implementation uses PyTorch \citep{paszke2019pytorch}, NumPy \citep{harris2020array}, and SciPy \citep{virtanen2020scipy}, all released under the BSD-3-Clause license, and Gymnasium \citep{towers2025gymnasium}, released under the MIT license.

\end{document}